\documentclass[sigconf]{acmart}
\setcopyright{none}
\renewcommand\footnotetextcopyrightpermission[1]{}
\acmConference{ArXiv}{2020}{}

\usepackage{microtype}
\usepackage{thmtools}
\usepackage{thm-restate}

\usepackage{booktabs}
\usepackage{bm}
\usepackage{amsthm}
\usepackage{amsmath}
\usepackage{algorithmic}
\usepackage[vlined, ruled, linesnumbered]{algorithm2e}
\usepackage{graphicx}
\usepackage{subfigure}
\usepackage{pifont}
\usepackage{pdfpages}
\usepackage{hhline}
\usepackage{url}
\usepackage{hhline}
\graphicspath{ {images/} }
\usepackage{amsthm}
\usepackage{amsmath}
\usepackage{bbm}
\usepackage{caption} 
\usepackage{diagbox}

\newcommand{\mycomment}[1]{{\footnotesize\color{blue}\ttfamily//#1}}

\newcommand{\specialcell}[2][c]{  \begin{tabular}[#1]{@{}c@{}}#2\end{tabular}}

 \SetKwProg{Fn}{Function}{:}{}
 \SetKwProg{MFn}{Function}{:}{}

 \newcommand{\CD}{\mathcal{D}}
\newcommand{\CA}{\mathcal{A}}
\newcommand{\CE}{\mathcal{E}}

\newcommand{\CX}{\mathcal{X}}
\newcommand{\CN}{\mathcal{N}}
\newcommand{\CF}{\mathcal{F}}

\newcommand{\CU}{\mathcal{U}}
\newcommand{\CL}{\mathcal{L}}
\newcommand{\CI}{\mathcal{I}}
\newcommand{\CO}{\mathcal{O}}

\newcommand{\CQ}{\mathcal{Q}}
\newcommand{\MBP}{\mathbb{P}}
\newcommand{\MBE}{\mathbb{E}}
\newcommand{\MBR}{\mathbb{R}}
\newcommand{\RHmu}{\mathring{\hat{\mu}}}
\newcommand{\indicator}[1]{\mathbbm{1}_{\{#1\}}}

\newcommand{\NIPS}{\color{black}}

\newtheorem{theorem}{\hspace{-0.13in}\bf Theorem}
\newtheorem{corollary}{\hspace{-0in}\bf Corollary}
\newtheorem{strategy}{\hspace{-0.12in}\bf Strategy}
\newtheorem{example}{\hspace{-0.0in}\bf Example}
\newtheorem{assumption}{\hspace{-0.0in}\bf Assumption}
\newtheorem{property}{\hspace{-0.0in}\bf Property}
\newtheorem{definition}{\hspace{-0.0in}\bf Definition}
\renewcommand{\proofname}{\hspace{-0.12in}\textnormal{{\textbf{Proof}}}}

\newcounter{bandit}
\makeatletter
\newenvironment{bandit}[1][htb]{  \let\c@algocf\c@bandit
  \SetAlgorithmName{BanditOracle}{bandit}{List of Classes}
    \begin{algorithm}[#1]  }{\end{algorithm}
}
\makeatother

\newcounter{evaluator}
\makeatletter
\newenvironment{evaluator}[1][htb]{  \let\c@algocf\c@evaluator
  \SetAlgorithmName{OfflineEvaluator}{evaluator}{List of Classes}
    \begin{algorithm}[#1]  }{\end{algorithm}
}
\makeatother

\newcommand\independent{\protect\mathpalette{\protect\independenT}{\perp}}
\def\independenT#1#2{\mathrel{\rlap{$#1#2$}\mkern2mu{#1#2}}}

\begin{document}
\title{Combining Offline Causal Inference and Online Bandit Learning for Data
  Driven Decision}

 \author{Li Ye}
\affiliation{The Chinese University of Hong Kong}
\author{Yishi Lin}
\affiliation{Tencent}
\author{Hong Xie}
\affiliation{College of Computer Science, Chongqing University}
\author{John C.S. Lui}
\affiliation{The Chinese University of Hong Kong}

\begin{abstract}
A fundamental question for companies with large amount of logged data
is: {\em How to 
    use such logged data together with incoming streaming data
  to make good decisions?}
Many companies currently make decisions via online A/B tests, 
but wrong decisions during testing hurt users' experiences 
and cause irreversible damage.  
A typical alternative is offline causal inference, 
which analyzes logged data alone to make decisions.   
However, these decisions are not adaptive to the new incoming data, 
and so a wrong decision will continuously hurt users' experiences.
To overcome the aforementioned limitations, 
we propose a framework to unify offline causal inference algorithms 
(e.g., weighting, matching)
  and online learning algorithms (e.g., UCB, LinUCB).
We propose novel algorithms and derive 
bounds on the decision accuracy via the notion of ``regret''.
We derive the first upper regret bound for forest-based online bandit algorithms.
Experiments on two real datasets  
show that our algorithms outperform other algorithms that use only logged data or
  online feedbacks, or algorithms that do not use the data properly.
\end{abstract}
 
\maketitle

\section{Introduction} 

How to make good decisions is a key challenge in many web applications,
i.e., an Internet company such as Facebook that sells in-feeds advertisements (or ``ads'' for short) needs to decide whether
to place an ad below videos or below images, as illustrated in
Fig.~\ref{fig:video_adv}.   

\begin{figure}[htb]
    \centering
      \includegraphics[width=0.49\textwidth]{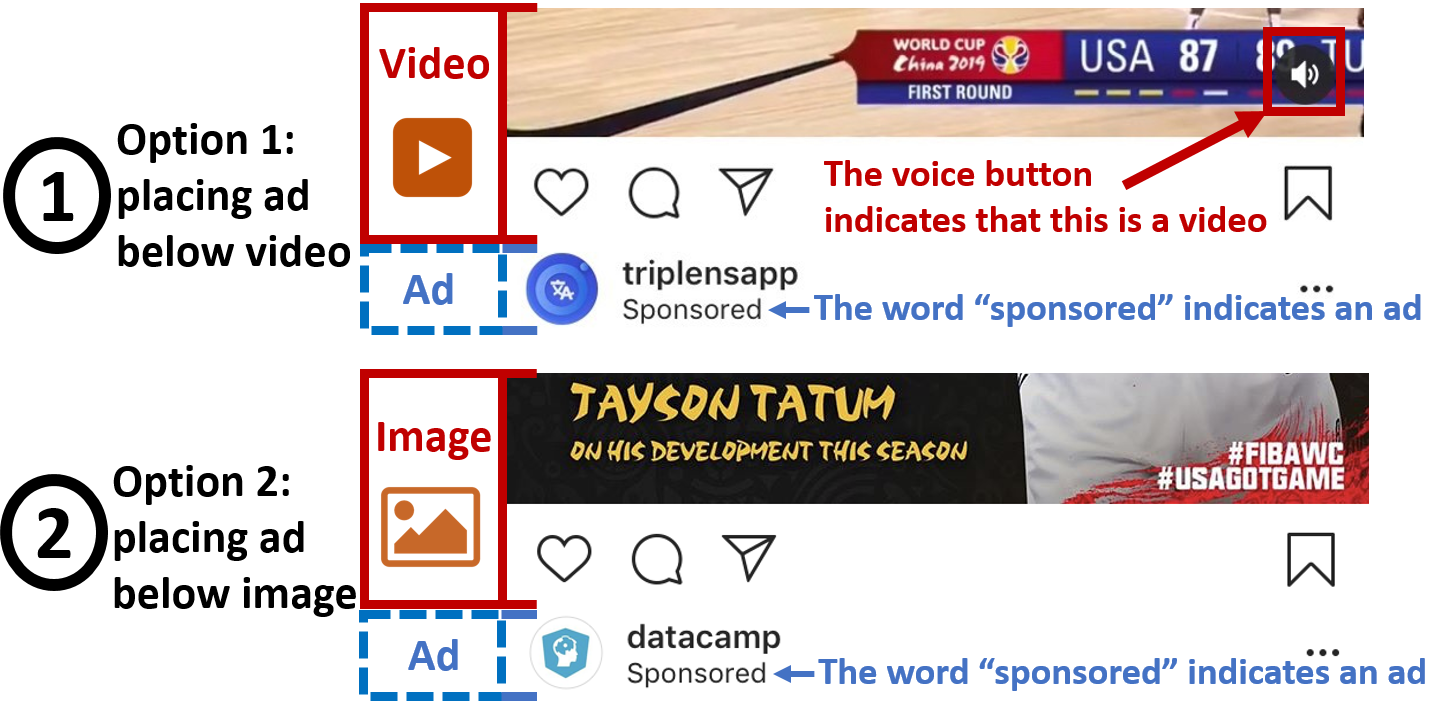} 
 \caption{In-feeds ad placement of Instagram}
 \label{fig:video_adv}
\end{figure}

\begin{comment}
\begin{figure}
  \includegraphics[width=1.04\textwidth]{causal_example2}
  \caption{ {\bf Four possible strategies for decisions in Example~\ref{example:Introduction:DecisionProb}. }
    (1) {\bf\em Empirical average} chooses the
    action with the highest average click rate in Table 3.
    (2) {\bf\em Causal inference} adjusts the fraction of users in each type (like
    videos or not) to be the
    fraction 50\% in the online setting.
    (3) {\bf\em A/B testing} randomly distribute the first 4,000 users to test the two
    actions, and then apply the action with higher click rate to the remaining
    6,000 users.
    (4) {\bf\em Our algorithm} uses the results of offline causal inference to
    {\em warm-start} the online learning algorithm (e.g. UCB) which makes adaptive
    online decisions.
  }
  \label{fig:four_strategies}
\end{figure}
\end{comment}

It is common that Internet companies have archived lots of logged data
which may assist decision making.
For example, Internet companies which sell in-feeds
advertisements have logs of advertisements' placement, as well as users' feedbacks
to these ads as illustrated in Table~\ref{tab:tabular_data}.  
The question is: {\em how to use these logs to make a better decision?} To
motivate this problem, consider Example~\ref{example:Introduction:DecisionProb}.    

\begin{table}[htb]
  \caption{Logged data of a company that sells in-feeds ads}
  \vspace{-0.1in}
  \centering
\makebox[0pt][l]{\raisebox{-2.5ex}{\hspace{-1.35in}{$\overbrace{\hspace{0.135\textwidth}}^{\text{\bf\normalsize action}}$
}}}
\makebox[0pt][l]{\raisebox{-2.5ex}{\hspace{-0.42in}{$\overbrace{\hspace{0.22\textwidth}}^{\text{\bf\normalsize contexts}}$
}}}
\makebox[0pt][l]{\raisebox{-2.5ex}{\hspace{1.12in}{$\overbrace{\hspace{0.065\textwidth}}^{\text{\bf\normalsize outcome}}$
}}}
  \begin{tabular}[t]{|c|c|c|c|c|c|} \hline
    {\small\bf ID} & \specialcell{\small\bf Ad below video?} & \specialcell{\small\bf User likes videos?} & {\small\bf Age} & ... & {\small\bf Click?} \\ \hline
    1 & no  & no & 30 & ... & no (0) \\\hline
    2 & yes  & yes & 20 & ... & yes (1) \\\hline
    {$\vdots$} & $\vdots$ & $\vdots$ & $\vdots$ & $\vdots$ & $\vdots$ \\\hline
  \end{tabular} 
  \label{tab:tabular_data} 
  \end{table}

\begin{example} 
10,000 new users will arrive to see the advertisement.  
The Internet company needs to decide whether to place the advertisement (ad) below a video or
below an image.  
The company wishes more clicks from these 10,000 new users.
Users are of two types --- users who ``like'' or users who ``dislike'' videos.   
For simplicity, assume 50\% of these new user likes (or dislikes) videos.  
The ``true click rates'' for each types of user, which are unknown to the company, are summarized in
Table~\ref{tab:settings_2types}. 
Furthermore, the company has a logged statistics of the past 400 users, half of whom like (or dislike) videos,
as shown in Table~\ref{tab:Introduction:SlickRateSum}.

\begin{table}[htb]
  \centering 
  \caption{True click rates of each type of user.  
  Action 2 ({\em ad below image}, with ``*'') is better for both types of users.}
  \vspace{-0.1in}
\begin{tabular}{|c||c|c|}
    \hline
\diagbox{\bf Action \#}{\bf User type} & \specialcell{\bf Like videos} & \specialcell{\bf Dislike videos} \\\hhline{|=#=|=|}
{\bf 1. Ad below video} & 11\% & 1\%  \\\hline
{\bf 2. Ad below image$^*$} & 14\%  & 4\%  \\\hline
  \end{tabular}
  \label{tab:settings_2types} 
\end{table}
\begin{table}[htb] 
 \centering 
\caption{Average click rate in logs of 400 users. 
In the logged data, users who like videos were more likely to see ads below videos, as they subscribed to more videos.}
  \vspace{-0.1in}
\begin{tabular}{|c||c|c|}
    \hline
\diagbox{\bf Action \#}{\bf User type} & \specialcell{\bf Like videos\\\bf (200 users)} & \specialcell{\bf Dislike videos\\\bf (200 users)} \\\hhline{|=#=|=|}
{\bf 1. Ad below video} & 10\% of 150 ads & 2\% of 50 ads \\\hline
{\bf 2. Ad below image$^*$} & 12\% of 50 ads & 4\% of 150 ads \\\hline
  \end{tabular}
  \label{tab:Introduction:SlickRateSum}
\end{table}

 \label{example:Introduction:DecisionProb}
\end{example}

{\NIPS
\noindent One may consider the following three strategies to make decisions.

  \noindent{\bf Empirical Average.} {\em The company chooses the action with the
  highest average click rate in the logged data to serve 10,000 incoming users.}
  For logs in Table~\ref{tab:Introduction:SlickRateSum}, the average
  click rate for ``ad below video'' is $(10\%{\times}
  150{+}2\%{\times}50)/(150{+}50){=}8\%$. Similarly, the average click rate is
  $6\%$ for ``ad below image''. Thus, the company chooses to place ``ad below
  video'' for the 10,000 incoming users. But it is the
  wrong action implied by the true click rates in Table~\ref{tab:settings_2types}. This method fails because it ignores users' {\em preferences to videos}.
        
  \noindent{\bf Offline causal inference.} {\em First, the company computes the average click
    rates w.r.t. each user type (as in
    Table~\ref{tab:Introduction:SlickRateSum}). Second, for each action, it computes
    the weighted average of such type-specific click rates where the weight is the
    fraction of users in each type. 
}
  For logs in Table~\ref{tab:Introduction:SlickRateSum}, 
the weighted average click rate for action 1 (ad below video) is
$10\% {\times} ({200}/{400}) + 2\% {\times} ({200}/{400}) {=} 6\%.$
Similarly, the weighted average click rate for action 2 is
$(12\%{+}4\%)/2{=}8\%$. Thus, the company chooses action 2 
based on the logged data in Table~\ref{tab:Introduction:SlickRateSum}. However, the
causal inference strategy has a risk of not finding the right action as the
logged data are only finite samples from the population. For example, in another sample statistics where the number of clicks for users
who dislike videos and see ad below video (the upper right cell in Table~\ref{tab:Introduction:SlickRateSum})
increases from $1$ (i.e. $2\%{\times} 50$) to $4$, the ``offline causal
inference'' strategy will then choose the inferior action of ``placing ad below video''.

\noindent{\bf Online A/B testing.}
{\em Each of the first 4,000 incoming users is randomly assigned to group A or B with equal probability.  
Users in group A see ads below videos (action 1), while users in group
B see ads below images (action 2). Then, the company selects the action with a higher average testing click rate for the remaining 6,000 users.}
In this A/B test, 2,000 testing users in group A suffer from the inferior
action.

The above three strategies have their own limitations. Taking the
``empirical average'' leads to a wrong decision by ignoring the important factor
of users' preferences. ``Offline
causal inference'' only uses the logged data and has a risk to make the wrong
decision due to the
incompleteness of the logged samples. ``A/B testing'' only uses the online data and pays a
{\em high cost} of testing the inferior actions. 
In this paper, we propose a novel
strategy which can use both the logged data and the online feedbacks.

\noindent{\bf Causal inference + online learning (our method).}
{\em 
The company applies offline causal inference to ``judiciously'' use the logged data to improve the
efficiency of an online learning algorithm. 
For example, UCB is used~\cite{auer2002finite} as the online learning
algorithm in Table~\ref{tab:example_revenue}.
}
}

\begin{table}[htb] 
  \centering
  \caption{The expected revenue(\$) of the four strategies over 10,000 users. 
  Suppose each click yields a revenue of \$1. The
    optimal expected revenue is \$900 (where the optimal action is to ``place videos
    below an image'').
    A strategy's {\em ``regret''} is the difference between the optimal revenue
    and its revenue.
  }
  \vspace{-0.1in}
\setlength{\tabcolsep}{3pt}
  \begin{tabular}{|c|c|c|c|c|}\hline
   {\bf Strategy} & \specialcell{\bf Empirical \\\bf average} &  \specialcell{\bf Causal \\\bf inference} & \specialcell{\bf A/B \\\bf testing} & \specialcell{\bf Our \\\bf method} \\ \hline
    \specialcell{\bf Expected Revenue} & 674.4  & 847.7 & 839.9 & {\bf 894.4}  \\\hline
    \specialcell{\bf Expected Regret}  & 225.6  & 52.3 & 60.1 & {\bf 5.6} \\\hline
  \end{tabular}
  \label{tab:example_revenue} 
\end{table}

Table~\ref{tab:example_revenue} shows that our algorithm achieves the highest
revenue for Example~\ref{example:Introduction:DecisionProb}. 
{\NIPS The key is to choose the appropriate data from the logged data to improve our decision making.}
Our contributions are:  

\noindent{\bf $\bullet$ A unified framework {with novel algorithms}.} 
We formulate a general online decision making problem, 
which utilizes logged data to improve both (1) {\em context-independent decisions}, and (2) {\em contextual decisions}.
{\NIPS Our framework unifies offline causal
  inference and online bandit algorithms. Our framework is generic enough to combine different causal inference
methods like matching and weighting~\cite{austin2011introduction}, and bandit
algorithms like UCB~\cite{auer2002finite} and LinUCB~\cite{li2010contextual}.  
This unification inspires us to extend the offline regression-forest to
an {\em``$\epsilon$-decreasing multi-action forest''} online learning algorithm.

\noindent{\bf $\bullet$ {Theoretical} regret bounds.}  
We derive regret upper bounds for algorithms in our framework.
            We show how the logged data can reduce the regret of online decisions.
    Moreover, we derive an asymptotic regret bound for the {\em ``$\epsilon$-decreasing
  multi-action forest''} algorithm. To the best of our knowledge, this is {the
    first regret analysis for a forest-based online bandit algorithm}.

\noindent{\bf $\bullet$ Extensive empirical evaluations.}
Experiments on synthetic data and real web datasets from Yahoo show 
that our algorithms that use both logged data and
  online feedbacks can make the right decision with the highest accuracy.
On the Yahoo's dataset, we reduce the regret by 21.1\%
compared to LinUCB of \cite{li2010contextual}.
Moreover, we show our algorithms outperforms the heuristics that uses supervised
learning algorithm to learn from offline data for decision making.

\section{Model \& Problem Formulation}

Our approach for the new online decision problem uses  
the logged data to improve online decision accuracy 
(more details in Section \ref{sec:framework}).  
Note that the observed logged data may have ``selection bias'' on
the actions, 
while in the online environment actions are chosen by the decision maker.  
This is why we need to find a formal approach 
to ``connect'' the logged data and the online data for correct usage.  

In this section, we first present the logged data model.  
Then we model the online environment.  
Finally, we present the online decision problem 
which aims to utilize both the logged data and online feedbacks to minimize the regret. 
 
\subsection{Model of Logged Data} 

We consider a tabular logged dataset (e.g., Table \ref{tab:tabular_data}), 
which was collected before the running of online decision algorithms.
The logged dataset has $I \in \mathbb{N}_+$ items, 
denoted by 
$
\mathcal{L} \triangleq \{ (a_{i}, \bm{x}_{i}, y_{i}) | i \in [-I] \}, 
$
where 
$(a_{i}, \bm{x}_{i}, y_{i})$ denotes the $i^{th}$ recorded data item
and 
$[-I] 
\triangleq \{-I, -I{+}1, \ldots, -1\}
$.
{
Here, we use {\em negative} indices to 
indicate that the logged data were collected in the past.  }  
 The action for data item $i$ is denoted as $a_i \in [K] \triangleq \{1, \ldots, K\}$, where $K \in \mathbb{N}_+$.  
The actions in the logged data can be generated according to the users'
natural behaviors or by the company's interventions. For example, option 1 and 2
in Figure~\ref{fig:video_adv} are actions.
The $y_i \in \mathcal{Y} \subseteq \mathbb{R}$ denotes the outcome (or reward).  
The $\bm{x}_i \triangleq (x_{i,1}, \ldots, x_{i,d}) \in \mathcal{X}$ 
denotes the contexts (or features) of data item $i$,
where $d \in \mathbb{N}_+$ and $\mathcal{X} \subseteq \mathbb{R}^d$.
The contexts are also known as {\em ``observed confounders''}~\cite{austin2011introduction}.   
We use $\bm{u}_i \triangleq (u_{i,1}, \ldots, u_{i, \ell}) \in \mathcal{U}$, 
where $\ell \in \mathbb{N}_+$ and $\mathcal{U} \subseteq \mathbb{R}^{\ell}$, 
to model the unobserved confounders.  
The $\bm{u}_i$ captures latent or hidden contexts, e.g., a user's monthly income.  

Now we introduce the generating process of the logged data.
For the $i^{th}$ user with context $\bm{x}_i$, let $A_i$ be the random
variable for the action of the $i^{th}$ user.
To capture the randomness of the outcome, 
{\color{black} let the random variable $Y_i (k)$ denote the outcome for the
  $i^{th}$ user if we had changed the action of the $i^{th}$ user to $k$.  
When $k\ne a_i$, $Y_i(k)$ is also called a {\em ``potential outcome''} in the causal
model~\cite{rubin2005causal} and it is not recorded in the logged data.}   
We have the following two assumptions, which are common for causal inference~\cite{rubin2005causal}.
 
\begin{assumption}[Stable unit for logged data]
  The potential outcome of a data item is independent of
the actions of other data items, i.e. 
$
\MBP[Y_i(k){=}y | A_i{=}a_i, A_j{=}a_j] = \MBP[Y_i(k){=}y|A_i{=}a_i]
$,
$\forall i {\in} [-I], j {\neq} i$.
  \label{asum:StableUnit:offline}
\end{assumption}

\begin{assumption}[Ignorability]
The potential outcomes of a data item $i$ are independent of the action $a_i$ given the
context $\bm{x}_i$ (so that we can ignore $\bm{u}_i$'s impacts), i.e.
$
  [ Y_i(1), \ldots, Y_i(K) ] \independent A_i | \bm{x}_i, \forall{i\in [-I]}. 
$
 \label{assumption:ignorability}
\end{assumption}

\noindent
Assumption~\ref{assumption:ignorability} holds in Example~\ref{example:Introduction:DecisionProb}
since the decision maker observes users' {\em preferences to videos} which
determine the users' types. In
Table~\ref{tab:settings_2types}, each type of users have a fixed click rates for
the actions,
which are independent of action.

\subsection{Model of Online Decision Environment}
 
Consider a discrete time system $t \in [T]$, 
where $T \in \mathbb{N}_+$ and $[T] \triangleq \{1,\ldots,T\}$.  
In time slot $t$, one new user arrives, 
and she is associated with the context $\bm{x}_t\in \CX$ 
and unobserved confounders $\bm{u}_t\in \CU$.
Then, the decision maker chooses an action $a_t\in [K]$, and observes the
outcome (or reward) $y_t$ corresponding to this chosen action.

Consider that the confounders 
$(\bm{x}_t, \bm{u}_t)$ 
are independent and identically generated by a cumulative distribution
  function 
$
F_{\bm{X}, \bm{U}} (\bm{x}, \bm{u})
\triangleq 
\mathbb{P} [\bm{X} \leq \bm{x}, \bm{U} \leq \bm{u} ], 
  $
where $\bm{X} \in \mathcal{X}$ and $\bm{U} \in \mathcal{U}$ 
denote two random variables.   
The distribution $F_{\bm{X}, \bm{U}} (\bm{x}, \bm{u})$ characterizes  
the joint distribution of the confounders over the whole user population.  
If we marginalize over $\bm{u}$, then the observed confounders $\bm{x}_t$ 
are independently identically generated from the marginal distribution  
$F_{ \bm{X} } (\bm{x}) \triangleq \mathbb{P}[\bm{X} \leq \bm{x}]$.   
Let the random variable $Y_t(k)$ denote
the outcome of taking action $k$ in time slot $t$.  
 
\begin{assumption}[Stable unit for online model]
  The outcome $Y_t(k)$ in time $t$ is 
independent of the actions in other time slots, i.e.
 \begin{align}
\MBP[Y_t(k){=}y | A_t{=}a_t, A_s{=}a_s] = \MBP[Y_t(k){=}y|A_t{=}a_t],
\forall t {\in} [T], s {\neq} t.
  \label{eq:assum_stable_online}
\end{align}
\label{asum:StableUnit:online}
\vspace{-0.2in}
\end{assumption}
  
In the online setting, before the decision maker chooses the
action, the distributions of the ``potential outcomes'' [$Y_t(1), \cdots, Y_t(K)$] are determined
given the confounders $(\bm{x}_t,\bm{u}_t)$. Moreover, as the unobserved confounders
$\bm{u}_t$ are i.i.d. in different time slots, the potential outcomes are
independent of how we select the action, given the user's context $\bm{x}_t$. Formally, we have the following property.

\begin{property}
The potential outcomes in time slot $t$ satisfies 
\begin{align}
 [ Y_t(1), \ldots, Y_t(K) ] \independent A_t | \bm{x}_t, \forall{t\in [T]}. 
  \label{eq:property_indep}
\end{align} 
   \label{property:ignorability}
\vspace{-0.2in}
\end{property}

One can see that Assumption \ref{asum:StableUnit:offline} and
\ref{assumption:ignorability} for the logged data correspond to Assumption
\ref{asum:StableUnit:online} and Property \ref{property:ignorability} for the
online decision model. This way, we can {\em``connect''} the logged data with the online
decision environment. 
Figure \ref{Fig:Model:LogOnline} summarizes our models of logged data and the
online feedbacks.  

\begin{figure}[htb] 
  \centering
  \includegraphics[width=0.48\textwidth]{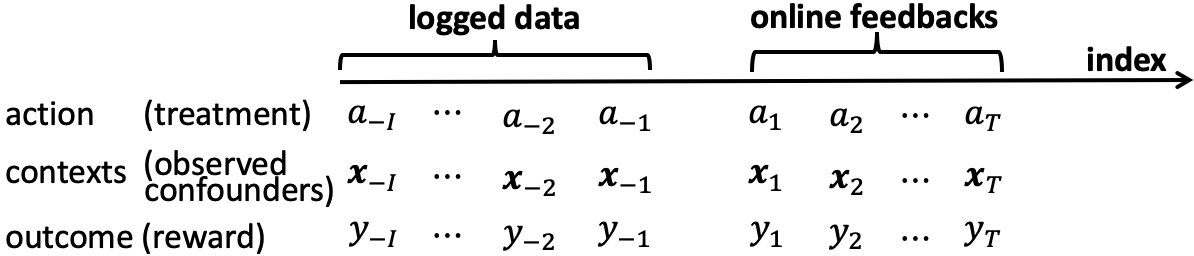} 
  \caption{\bf Summary of logged data and online feedbacks} 
  \label{Fig:Model:LogOnline} 
  \vspace{-0.2in}
\end{figure}

\subsection{ Online Decision Problems }  
\label{sec:DecisionProblem}

The decision maker selects an action in each time slot. 
We consider two
kinds of online decision problems depending on whether users with different contexts
can be treated differently or not.

\noindent{\bf $\bullet$ Context-independent decision problem. }
Consider the setting where a company makes a context-independent decision for all users.  
In causal inference, this setting corresponds
to the estimation of ``average treatment effect''~\cite{rubin2005causal}. In online learning, this
setting corresponds to the ``stochastic multi-armed bandit'' problem~\cite{kuleshov2014algorithms}.  
In time slot $t$, the decision maker can use the logged data $\mathcal{L}$
and the feedback history  
$
\mathcal{F}_t 
{\triangleq} \{(a_1, \bm{x}_1, y_1), \cdots, (a_{t-1}, \bm{x}_{t-1}, y_{t-1})\}.   
$ 
Let $\mathcal{E}$ denote an {\em``offline evaluator''} (e.g., an offline causal
inference algorithm), 
which synthesizes feedbacks from the logged data $\mathcal{L}$.
Let $\mathcal{O}$ denote an online context-independent bandit learning algorithm.   
We defer the details of $\mathcal{E}$ and $\mathcal{O}$ to Section \ref{sec:context_independent}.  
Let $\mathcal{A}_{\mathcal{O} +  \mathcal{E}} (\cdot, \cdot)$ denote an algorithm 
that combines $\mathcal{O}$ and $\mathcal{E}$ to make online context-independent decisions, 
i.e., 
$
a_t {=} \mathcal{A}_{\mathcal{O} +  \mathcal{E}} (\mathcal{L}, \mathcal{F}_t ).  
$  
The decision accuracy is quantified the following pseudo-regret: 
\begin{align}
R(T, \mathcal{A}_{\mathcal{O} +  \mathcal{E}}) 
\triangleq
\sum_{t=1}^T \left(  \mathbb{E}[y_t|a^\ast] - \mathbb{E}[y_t| 
a_t {=} \mathcal{A}_{\mathcal{O} +  \mathcal{E}} (\mathcal{L}, \mathcal{F}_t )] \right),
  \vspace{-0.1in}
  \label{eq:regret_population}
\end{align}  
where $a^\ast \triangleq \arg\max\nolimits_{a\in[K]} \mathbb{E}[y_t | a_t{=}a]$ denotes the
optimal action.
 
\noindent
{\bf $\bullet$ Context-dependent decision problem. }
Consider that a company can make different decisions for users coming with different contexts.   
Let $\mathcal{O}_c$ denote an online contextual bandit learning algorithm.   
Let $\mathcal{A}_{\mathcal{O}_c +  \mathcal{E}} (\cdot, \cdot, \cdot)$ denote an algorithm, 
that combines $\mathcal{O}_c$ and $\mathcal{E}$ to make online contextual decisions, 
i.e., 
$a_t {=} \mathcal{A}_{\mathcal{O}_c +  \mathcal{E}}(\mathcal{L}, \mathcal{F}_t, \bm{x}_t)$.   
Given $\bm{x}_t$, the {\em unknown} optimal action is 
$
a_t^\ast{\triangleq}\max_{a\in[K]}\mathbb{E}[y_t|a,\bm{x}_t].
$
The decision accuracy is quantified the following pseudo-regret: 
\[
R(T, \mathcal{A}_{\mathcal{O}_c +  \mathcal{E}}) 
{\triangleq}\!\! 
\sum_{t=1}^T  
\!\!
\left( 
\mathbb{E}[y_t|a_t^\ast,\bm{x}_t] 
{-}
\mathbb{E}[y_t|a_t {=}\mathcal{A}_{\mathcal{O}_c +  \mathcal{E}} (\mathcal{L}, \mathcal{F}_t, \bm{x}_t),\bm{x}_t]  
\right).  
\] 
This paper aims to develop a generic framework to combine 
different bandit learning algorithms $\CO$, $\mathcal{O}_c$, and offline evaluator $\mathcal{E}$ 
to make decisions with provable theoretical guarantee on the regret.     

In the following sections, we explore the following questions: 
\textit{
(1) How to combine offline evaluator $\mathcal{E}$ with online bandit learning $\mathcal{O}$ or $\mathcal{O}_c$? 
(2) How to prove bounds on the decision maker's regrets?  (3) What are the advantages of our
methods on real decision problems?
}

\section{General Algorithmic Framework}
\label{sec:framework}

We first develop a general algorithmic framework 
to combine offline evaluators ($\mathcal{E}$) with online bandit learning algorithms ($\mathcal{O}$ and $\mathcal{O}_c$).  
Then, we present regret bounds for the proposed framework.

\subsection{Algorithmic Framework}

  \begin{figure}[htb] 
  \centering
  \includegraphics[width=0.4\textwidth]{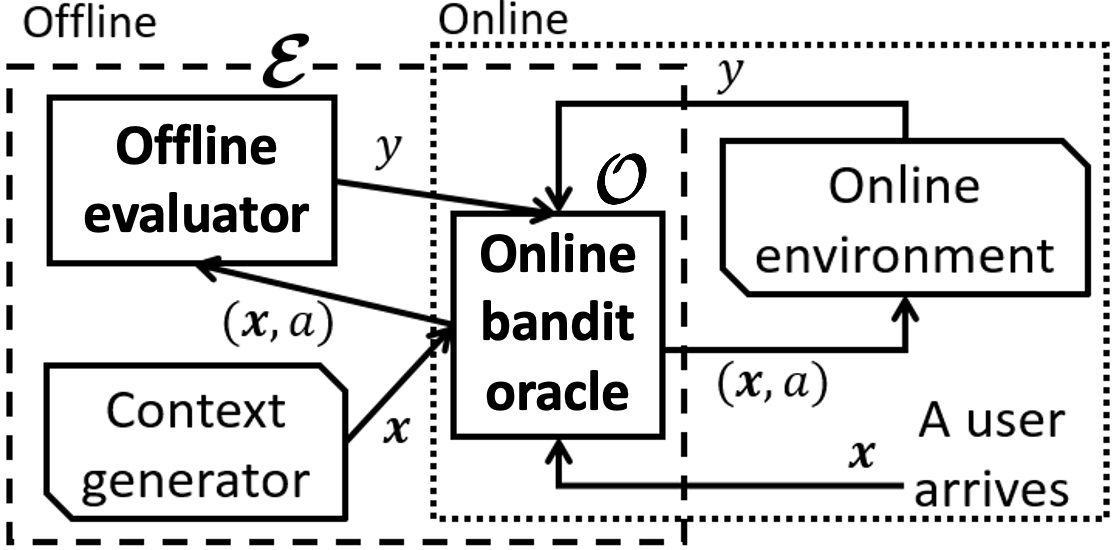}
  \caption{Illustration of algorithmic framework.
{\it \color{black} Online bandit oracle}
  has two functions: 
function {\it play}$(\bm{x})$ returns an action $a$ given a
 context $\bm{x}$;
function {\it update}$(\bm{x}, a, y)$ updates the oracle
 with the feedback $y$ w.r.t. action $a$, under the context $\bm{x}$.
{\it Offline evaluator}  
has one function {\it get\_outcome}$(\bm{x},a)$ that searches the logged data 
and returns a ``synthetic outcome'' $y$ given the pair $(\bm{x},a)$, where
the return value $y=\text{NULL}$
if the offline evaluator is not able to synthesize a feedback.  
  } 
  \label{fig:framework}
\end{figure}

{\NIPS
  The key idea of our framework is to select ``appropriate'' data from the log to
  improve online learning. This is achieved via the idea of ``virtual play''.
}
Figure~\ref{fig:framework} illustrates the workflow of our framework.  
The {\em ``BanditOracle''} $\CO$ denotes an online learning algorithm. 
The {\em``OfflineEvaluator''} $\CE$ denotes an algorithm that 
synthesizes feedbacks from the log.
Algorithm~\ref{alg:framework} shows how to coordinate these two components to 
make sequential decisions in $T$ rounds.  
Each round has an offline phase and an online phase.  
In the offline phase (Line 4-11), 
we first generate a context according to the CDF $F_{\bm{X}}(\cdot)$\footnote{In practice, the CDF is usually {\em unknown} but can be estimated with
  convergence guarantee (\cite{jiang2017uniform}). We will discuss using
  empirical context distribution in Section~\ref{sec:experiments}.}.
Then, we get an action
from the {\em BanditOracle}. 
The {\em OfflineEvaluator} returns a synthetic
feedback to update the {\em BanditOracle}.  
We repeat such procedure until the {\em OfflineEvaluator} 
cannot synthesize a feedback.  
When this happens, we turn to the online phase (Line 12-14), 
where the same {\em BanditOracle} chooses the action, 
and updates itself with online feedbacks.

 \begin{algorithm}[htb]
  \caption{\bf General Algorithmic Framework}\label{alg:framework}
    Initialize the \textit{OfflineEvaluator} with logged data $\CL$\\
  Initialize the \textit{BanditOracle}\\
    \For{$t=1$ to $T$}{
    \While{True}{
     $\bm{x} {\gets} \textit{context\_generator()}$ \mycomment{from CDF $F_{\bm{X}}(\cdot)$}\\
      $a \gets \textit{BanditOracle.}\textbf{play}(\bm{x})$
      \mycomment{virtual play}\\
      $y \gets \textit{OfflineEvaluator.}\textbf{get\_outcome}(\bm{x}, a)$\\
   \If{$y \ne$ {\em NULL} }{
        \textit{BanditOracle.}\textbf{update}$(\bm{x}, a, y)$
}\Else(\mycomment{offline evaluator cannot synthesize a feedback}){
        \textbf{break} 
      }
    }
       $a_t \gets \textit{BanditOracle.}\textbf{play}(\bm{x}_t)$
    \mycomment{online play} \\
    $y_t\gets$ the outcome from the online environment \\
    \textit{BanditOracle.}\textbf{update}$(\bm{x}_t, a_t, y_t)$
  } 
\end{algorithm}

\noindent{\bf Unifying causal inference and online bandit learning.}
Both online bandit algorithms and causal inference
algorithms are special cases of our framework.
First, if there are no logged data, then the offline evaluator cannot synthesize
feedbacks and always returns {``NULL''}. We use $\CE_\emptyset$ to denote such
offline evaluator that always returns ``NULL''.
Then, our framework always calls the online bandit oracle, and
it reduces to an online bandit algorithm.
Second, we consider a specific A/B test online learning oracle described in
BanditOracle~\ref{alg:abtest_oracle}, and we let $T{=}1$.  
Then, after the offline phase, the estimated outcome $\bar{y}_a$ can be used to estimate the causal effect. 
In this case, our framework reduces to a causal inference algorithm.
 
\setcounter{bandit}{-1}
\begin{bandit}
    \caption{\bf A/B Testing}\label{alg:abtest_oracle}
  {\bf Member variables:} the average outcome $\bar{y}_a$ of each action 
 $a{\in}[K]$, and the number of times $n_a$ that action $a$ was played.\\
  \SetKwFunction{FPlay}{{\bf play}}
  \MFn{\FPlay{$\bm{x}$}}{
    \Return $a$ with probability $1/K$ for each $a\in[K]$
  }
  \SetKwFunction{FUpdate}{{\bf update}}
  \MFn{\FUpdate{$\bm{x}, a, y$}}{
    $\bar{y}_a \gets (n_a\bar{y}_a+y) / (n_a+1)$,~~~~  $n_a \gets n_a + 1$ 
  }
\end{bandit}

\subsection{Regret Analysis Framework} 

We decompose the regret of Algorithm~\ref{alg:framework} as  
\textit{``online regret = total regret - regret of virtual plays''}.  
The intuition is that among all the decisions of the online bandit oracle, 
there are ``virtual plays'' whose feedbacks are simulated 
from {\em the logged data}, and ``online
plays'' whose feedbacks are from {\em the real online environment}.
The online bandit oracle cannot distinguish the ``virtual plays'' from ``online plays''.    
Thus we can apply the theories of the online bandit
oracles~(e.g. \cite{auer2002finite}\cite{li2010contextual}\cite{agrawal2012analysis}) to bound the {\em total regret}.
By subtracting the {\em regret
of virtual plays}, we get the bound for {\em online regret}.   
 
\begin{restatable}[General upper bound]{theorem}{generalUpperBound}
 \label{thm:general_upper_bound}
Suppose there exist $g(T)$ and $g_c(T)$, such that 
$R(T, \mathcal{A}_{\CO + \CE_{\emptyset}}){\le} g(T)$, 
and $R(T, \mathcal{A}_{\CO_c + \CE_\emptyset}){\le} g_c(T), \forall T$. 
Denote the returns of the offline evaluator till time $T$ as 
$\{\tilde{y}_j\}_{j=1}^N$ w.r.t. input $\{(\tilde{\bm{x}}_j{,} \tilde{a}_j)\}_{j=1}^N$. 
If $\mathcal{E}$ satisfies 
$\MBE[\CE\text{.\it{get\_outcome}}(\bm{x},a)]{=}\MBE[y|a]$, 
then  
 \begin{align}
   \label{eq:general_upper_bound}
  R(T, \CA_{\CO + \CE}) {\le} g(T{+}N) {-} \sum\nolimits_{j=1}^N\left( \max_{a^\prime\in[K]}\MBE[y|a^\prime] {-} \MBE[y|a=\tilde{a}_j] \right).
  \end{align}
If $\mathcal{E}$ satisfies
$\MBE[\CE\text{.\it{get\_outcome}}(\bm{x},a)]{=}\MBE[y|a,\bm{x}]$ contextually, 
then 
 \begin{align*}
R(T, \CA_{\CO_c {+} \CE}) {\le} g_c(T{+}N) {-}
\hspace{-0.03in} \sum\nolimits_{j{=}1}^N
\hspace{-0.00in} \left(  \hspace{-0.00in}\max_{a^\prime\in[K]}\MBE[y|a^\prime,\tilde{\bm{x}}_j] {-} \MBE[y|a{=}\tilde{a}_j\hspace{-0.01in}{,} \tilde{\bm{x}}_j] \hspace{-0.00in}\right) \hspace{-0.00in}. 
\end{align*} 
\end{restatable}

\noindent 
\textit{Due to page limit, all proofs are presented in the supplementary materials~\cite{supplement}. }
In Inequality (\ref{eq:general_upper_bound}), $g(T{+}N)$ is the upper bound of {\em
 total regret}, and $\sum\nolimits_{j=1}^N\left( \max_{a^\prime\in[K]}\MBE[y|a^\prime] {-}
  \MBE[y|a=\tilde{a}_j] \right)$ is the {\em regret of virtual
  plays}.  
The condition $\MBE[\CE\text{.\it{get\_outcome}}(\bm{x},a)]{=}\MBE[y|a]$ 
(or $\MBE[\CE\text{.\it{get\_outcome}}(\bm{x},a)]{=}\MBE[y|a,\bm{x}]$ ) 
implies that the offline evaluator $\CE$ returns unbiased context-independent (or contextual) outcomes.
 Using similar regret decomposition, we also derive a regret lower
bound with logged data in our supplementary material~\cite{supplement}.

\section{Case Study I: Context-independent Decision}
\label{sec:context_independent}

To demonstrate the versatility of our algorithmic framework for context-independent decisions, 
we start with a case of using UCB and exact matching in our framework.  
Then we extend the offline evaluator  
from exact matching to propensity score matching, and weighting method like inverse propensity score weighting.   
Finally, we study the case when Assumptions~\ref{asum:StableUnit:offline} and \ref{assumption:ignorability} do not hold.

\subsection{Warm-up: UCB +  Exact Matching  }

To illustrate Algorithm~\ref{alg:framework}, 
let us start with an instance that uses UCB~\cite{auer2002finite}
(BanditOracle~\ref{alg:ucb}) as the online bandit oracle
and the ``exact
matching'' causal inference algorithm~\cite{stuart2010matching} (OfflineEvaluator~\ref{alg:exact_match}) as
the offline evaluator.   
We denote this instance of Algorithm~\ref{alg:framework} as $\CA_{\text{UCB+EM}}$. 
In each round, BanditOracle~\ref{alg:ucb} selects an action with the maximum 
upper confidence bound defined as 
$\bar{y}_a{+}\beta\sqrt{{2\ln(n)}/{n_a}}$, 
where 
$\bar{y}_a$ is the average outcome, 
$\beta$ is a constant, and $n_a$ is the
number of times that an action $a$ was played. 
OfflineEvaluator~\ref{alg:exact_match} searches for 
a data item in log $\CL$ with the exact same context $\bm{x}$ and action
$a$, and returns the outcome $y$ of that data item.  
If it cannot find a matched data item for an action $a$,
it stops the matching process for the action $a$.  
The stop of matching is to ensure
that the synthetic feedbacks simulate the online feedbacks correctly.  

\begin{bandit}
    \caption{\bf {UCB} \cite{auer2002finite} }\label{alg:ucb}
    {\bf Variables:} the average outcome $\bar{y}_a$ of each action
$a{\in}[K]$, number of times $n_a$ action $a$ was played.\\
  \SetKwFunction{FPlay}{{\bf play}}
  \Fn{\FPlay{$\bm{x}$}}{
    \vspace{-0.05in}
    \textbf{return} $\arg\max\limits_{a\in[K]} \bar{y}_a{+}\beta \sqrt{\frac{2\ln(\sum_{a\in[K]}n_a)}{n_a}}$
  }
  \SetKwFunction{FUpdate}{{\bf update}}
  \Fn{\FUpdate{$\bm{x}, a, y$}}{     $\bar{y}_a \gets {(n_a\bar{y}_a{+}y)} / {(n_a{+}1)}$,~~~~  $n_a \gets n_a {+} 1$ 
  }
\end{bandit}

\begin{evaluator}
    \caption{\bf Exact Matching (EM) \cite{stuart2010matching}}\label{alg:exact_match}
  \SetKwFunction{FExactMatching}{{\bf get\_outcome}}
{\bf Member variables}: $S_a{\in}\{False,True\}$ indicates whether we stop matching
for action $a$, initially $S_a{\gets} False, \forall a{\in}[K]$.
  
  \MFn{\FExactMatching{$\bm{x}$, $a$}}{
    \If{$S_a=False$}{
      $\CI(\bm{x}, a)\gets \{i~|~\bm{x}_i=\bm{x}, a_i=a\}$\\
      \If{$\CI(\bm{x}, a) \ne \emptyset$}{
        $i\gets$ a random sample from $\CI(\bm{x},a)$\\
        $\CL\gets \CL \backslash \{(a_i,\bm{x}_i,y_i)\}$\\
        \Return $y_i$
      }
                }
          $S_a\gets True$
\mycomment{If we can't find a sample for the action $a$, i.e. $\CI(\bm{x},a){=}\emptyset$, stop matching for $a$}\\
      \Return NULL
                
      }
\end{evaluator}

Applying Theorem \ref{thm:general_upper_bound}, 
we present the regret upper bound of $\CA_{\text{UCB+EM}}$ 
in the following theorem.  

\begin{restatable}[UCB+Exact matching]{theorem}{EMUCB} 
Suppose there are $C \in \mathbb{N}_+$ possible categories 
of users' features denoted by $\bm{x}^1, \ldots, \bm{x}^C$.  
Denote ${\MBP}[\bm{x}^c]$ as the probability for an online user to have context $\bm{x}^c$.
Recall $a^\ast{=} 
\arg\max_{\tilde{a}\in[K]}\MBE[y|\tilde{a}]$ and denote $\Delta_a\triangleq \MBE[y|a^\ast]-\MBE[y|a]$.
Let $N(\bm{x}^c,a){\triangleq} \sum_{i\in[-I]}
\indicator{\bm{x}_i{=}\bm{x}^c,a_i{=}a}$ be the number of samples with context
$\bm{x}^c$ and action $a$.
Suppose the reward $y\in [0,1]$. Then,
\begin{align*}
  {R}(T, & \CA_{\text{UCB+EM}}) \le 
\sum\nolimits_{a\ne a^\ast} 
 \Delta_a
\left( 
1{+}\frac{\pi^2}{3}  \right. \\
  &\left. + \sum\nolimits_{c\in[C]}   
  \max\left\{0{,} 8\frac{\ln(T{+}A)}{\Delta_a^2}  
  {\MBP}[\bm{x}^c] {-} \hspace{-0.03in}
  \min_{\tilde{c}\in[C]} \hspace{-0.03in} \frac{ N(\bm{x}^{\tilde{c}}\hspace{-0.03in}{,}a)\MBP[\bm{x}^c]}{\MBP[\bm{x}^{\tilde{c}}]}\right\}  \right),
\end{align*}
where $A$ is derived as: 
  \[
    A  = N{-} \hspace{-0.05in}
    \sum_{a\ne a^\ast} \hspace{-0.03in}\sum_{c\in[C]} \hspace{-0.03in}
    \max \hspace{-0.03in} \left\{0,
      N(\bm{x}^{{c}}{,}a)   {-} (8\frac{\ln(T{+}N)}{\Delta_a^2}{+}1{+}\frac{\pi^2}{3}){\MBP}[\bm{x}^c]\right\}.
  \] 
  \label{mthm:exact_matching}
\end{restatable}

\noindent
Theorem~\ref{mthm:exact_matching} states how logged data reduces the regret.
  When there is no logged data, i.e., $N(\bm{x}_c,a)=0$ for $\forall \bm{x}_c, a$, the regret bound
  $O(\log(T))$ is the same as that of UCB. 
If the number of logged data $N(\bm{x}^c,a)$ is greater
than a threshold $\MBP[\bm{x}^c]8{\ln(T+A)}/{\Delta_a^2} $ for each context
$\bm{x}^c$ and action $a$, then the regret is smaller than a constant
$\left( 1+\frac{\pi^2}{3} \right) \sum_{a\ne a^\ast} \Delta_a$.  
Note that when we give all the data items the same dummy context $\bm{x}_0$, our $\CA_{\text{UCB+EM}}$ reduces to the ``Historical UCB''
(HUCB) algorithm in~\cite{shivaswamy2012multi}, as HUCB
ignores the context and only matches the actions.

One limitation of the exact matching evaluator is that 
when $\bm{x}$ is continuous or has a high dimension, 
it will be difficult to find a sample
in log-data with exactly the same context $\bm{x}$.   
To address this limitation, we consider the 
propensity score matching method~\cite{stuart2010matching}.

\subsection{ UCB + Propensity Score Matching }

We replace the offline evaluator, i.e., exact matching, of $\CA_{\text{UCB+EM}}$ 
with the {\em propensity score matching} stated in OfflineEvaluator \ref{alg:ps_match}.  
This replacement results in $\CA_{\text{PSM+UCB}}$.   
The propensity score $p_i(a){\in} [0,1]$ for action $a$ is the probability of observing the
action $a$ given the context $\bm{x}_i$, i.e. $p_i(a){=}\MBP[A_i{=}a|\bm{x_i}]$.  
For the context-independent case, Assumption~\ref{assumption:ignorability}
implies that one can ignore other contexts given the {\em propensity
  scores}~(\cite{rosenbaum1983central}), i.e. 
$
[Y_i(1), \cdots, Y_i(K)] {\independent} A_i | \left( p_i(1),\cdots, p_i(K)\right)$.
Since $\sum_{a=1}^K p(a)=1$, we use a vector $\bm{p}\triangleq (p(1),\cdots, p(K-1))$ to
represent the propensity scores on all actions.
For any incoming context-action pair $(\bm{x},a)$, OfflineEvaluator~\ref{alg:ps_match}
first finds a logged sample $i$ with a similar propensity score vector
$\bm{p}_i$ and the same action $a_i=a$, and returns the outcome $y_i$ of that logged sample (Line 5-9).
We use the
stratification strategy~\cite{austin2011introduction} to find samples with
similar propensity scores.
{\NIPS
  Note that every time we find a matched sample, we delete it in
  Line 8. Thus the matching process will terminate as we have finite samples.
  Since we can get a
  random element and delete it in $O(1)$ time via a HashMap, the total time complexity of calling $\CE_{\text{PSM}}$ is
  $O(I)$ where $I$ is the number of logged samples.
}

\begin{evaluator}
  \setcounter{AlgoLine}{0}
   \caption{\bf {\small Propensity Score Matching} (PSM) \cite{stuart2010matching}}\label{alg:ps_match}
 
{\bf Variables}: 
initially $S_a\gets False$, $\forall a{\in}[K]$.
The pivot set $\CQ{\subset} [0,1]$ with a finite number of elements.\\
 \SetKwFunction{FGetOutcome}{{\bf get\_outcome}}
 \Fn{\FGetOutcome{$\bm{x}, a$}}{
    \If{$S_a=False$}{
     $\bm{p}\gets (\MBP[A=1|\bm{x}], \cdots, \MBP[A=K-1|\bm{x}])$
     \mycomment{here, $\bm{p}\in[0,1]^{K-1} = (p(1), \cdots, p(K-1))$ is a vector}\\
      $\CI(\bm{p}, a){\gets} \{i~|~\textbf{stratify}(\bm{p}_i){=}\textbf{stratify}(\bm{p}), a_i{=}a\}$\\
    \If{$\CI(\bm{p}, a) \ne \emptyset$}{
         $i\gets$ a random sample from $\CI(\bm{p},a)$\\
         $\CL{\gets} \CL \backslash \{(\bm{x}_i,a_i,y_i)\}$ \mycomment{delete item}\\
         \Return $y_i$
      }
               $S_a\gets True$ \mycomment{stop matching for $a$}
           }
           \Return NULL
              }
\SetKwProg{MFn}{Function}{: \mycomment{this is used by $\CE_{\text{PSM}}$}}{}
\SetKwFunction{FStratify}{{\bf stratify}}
\MFn{\FStratify{$\bm{p}$}}{
  \textbf{return} $\arg\min_{\bm{q}{\in} \CQ}
  \hspace{-0.03in}||\bm{p}{-}\bm{q}||_2$ \mycomment{round to the nearest pivot}
}
\end{evaluator}

Applying Theorem \ref{thm:general_upper_bound}, 
we present the regret upper bound of $\CA_{\text{UCB+PSM}}$ 
in the following theorem.  

\begin{restatable}[UCB+Propensity score matching]{theorem}{PSMUCB}
Suppose the propensity scores are in a finite set $\bm{p}_i {\in} \CQ{\triangleq}
\{\bm{q}_1,\ldots, \bm{q}_Q\}{\subseteq} [0,1]^{K-1}$, for $\forall i{\in} [-I]$.
Let $N(\bm{q},a)$ be the number of data items whose $\bm{p}_i{=}\bm{q}$ and action $a_i{=}a$, and $N{\triangleq}
\sum_{c\in[Q],a\in[K]} N(\bm{q},a)$.
Denote ${\MBP}[\bm{q}_c]$ as the probability for an online user to have propensity score $\bm{q}_c$.
Suppose the reward $y{\in}[0,1]$.  
Then, 
\begin{align}
  \label{eq:ps_ucb_bound}
  {R}(& T, \CA_{\text{UCB+PSM}}) 
\le
\sum\nolimits_{a{\ne}a^\ast} \hspace{-0.03in}
\Delta_a
\left( 
          1{+}\frac{\pi^2}{3} {+} \hspace{-0.07in}
       \right. \nonumber\\
   & \left. \sum\nolimits_{c\in[Q]} \max\left\{0,  8\frac{\ln(T{+}A)}{\Delta_a^2}  
    {\MBP}[\bm{q}_c] - \hspace{-0.05in}
    \min_{\tilde{c}\in[Q]} \hspace{-0.05in}\frac{N(\bm{q}_{\tilde{c}}{,}a) \MBP[\bm{q}_c]}{\MBP[\bm{q}_{\tilde{c}}]}
    \right\} \right), 
\end{align}
where $A$ is derived as: 
\[
    A{=}N-\hspace{-0.05in}
    \sum\limits_{a\ne a^\ast}\hspace{-0.03in} \sum\limits_{c\in[Q]} 
  \hspace{-0.05in}\max \hspace{-0.03in}\left\{\!0, \hspace{-0.02in}
      \min\limits_{\tilde{c}\in[Q]} \hspace{-0.05in}\frac{N(\bm{q}_{\tilde{c}}{,}a) \MBP[\bm{q}_c]}{\MBP[\bm{q}_{\tilde{c}}]}
      {-} (\! \frac{8\ln(T{+}N)}{\Delta_a^2}{+}1{+}\frac{\pi^2}{3} \!){\MBP}[\bm{q}_c] \!\right\}.
\]
\vspace{-0.2in}
 \label{thm:ps_matching} 
\end{restatable}

{\color{black}
\noindent
  Theorem \ref{thm:ps_matching} is similar to Theorem~\ref{mthm:exact_matching}
  where we replace the context vector $\bm{x}^c$ with the propensity score
  vector $\bm{q}_c$.
    If the number of logged data $N(\bm{q}_c,a)$ is greater
than ${\MBP}[\bm{q}_c]8{\ln(T{+}A)}/{\Delta_a^2} $ for 
$\forall c{\in}[Q]$ and $a{\in}[K]$, then the regret is smaller than a constant
$(1{+}\pi^2/3)\sum_{a{\ne}a^\ast} \Delta_a$.
When we only have two actions, the propensity score vector $\bm{p}$ only has one
dimension, and the {\em propensity score matching} do not have the problem of
{\em exact matching} from the high-dimensional context $\bm{x}$. But when the
number of actions $K>2$, it is still difficult to find matched
propensity score vector $\{p(1), \cdots, p(K{-}1)\}$. The following weighting algorithm can deal with more
than two actions.
}

\subsection{UCB + Inverse Propensity Score Weighting}

To further demonstrate the versatility of our framework, 
we show how to use weighting
methods~\cite{swaminathan2015counterfactual}\cite{kallus2018balanced} in causal
inference.   
As shown in Line 4 in OfflineEvaluator~\ref{alg:ipsw_match}, we use the inverse
of the propensity score $1/p_i(a_i)$ as the weight. Here, we only need the
propensity score for the chosen action $a_i$. 
We replace the offline evaluator
with the IPS weighting OfflineEvaluator~\ref{alg:ipsw_match} to get $\CA_{\text{UCB+IPSW}}$.  

OfflineEvaluator \ref{alg:ipsw_match} first estimates the
outcome $\bar{y}_a$ as the weighted average of logged outcomes.
The intuition of IPS weighting is as follows: if an action is applied to users in group A more often
than users in other groups, then each sample for group A should have smaller 
weight so the total weights of each group is proportional to
its population. In fact, the IPS weighting estimator is unbiased via {\em importance sampling}\cite{rubin2005causal}.
Then, we calculate the {\em effective sample size} (a.k.a. ESS) $N_a$ of logged plays on the action $a$
according to~\cite{hoeffding1994probability}. 
After such initialization, the offline evaluator returns $\bar{y}_a$ w.r.t.
action $a$ for $\lfloor {N_a} \rfloor$ times, and return NULL afterwards.  

\begin{evaluator}
  \caption{\bf IPS Weighting (IPSW) \cite{swaminathan2015counterfactual}}\label{alg:ipsw_match}
   {\bf Member variables:} $\bar{y}_a,N_a (a{\in}[K])$ initialized in {\small\textbf{\_\_init\_\_($\CL$)}}\\
 \SetKwFunction{FInit}{{\bf \_\_init\_\_}}
 \MFn{\FInit{$\CL$}}{
   \For{$a\in [K]$}{
     $\bar{y}_a {\gets} \frac{\sum_{i\in[-I], a_i=a} y_i/p_i(a_i) }{\sum_{i\in[-I],
         a_i=a} 1/p_i(a_i)}$, 
     $N_a {\gets} \frac{ (\sum_{i\in[-I], a_i=a} 1/p_i(a_i))^2 }{ \sum_{i\in[-I], a_i=a} (1/p_i(a_i))^2 }$
   }
 }
 \SetKwFunction{FIPSW}{{\bf get\_outcome}}
 \MFn{\FIPSW{$\bm{x},a$}}{
   \If{$N_a\ge 1$}{
     $N_a \gets N_a-1$\\
     \Return $\bar{y}_a$
   }
        \Return NULL
    }
\end{evaluator}

\begin{restatable}[UCB + IPS weighting]{theorem}{IPSWUCB}
  Suppose the reward $y\in [0,1]$,
  and the propensity score is bounded $p_i {\ge}
\bar{s} {>}0$ $\forall i\in[I]$, then 
\begin{align*}
R(T{,}\CA_{\text{UCB+IPSW}}) {\le} 
& 
  \sum\nolimits_{a\ne a^\ast}  
  \Delta_a  
  \left(    1 +  \pi^2 / 3 \right. +  
\\
& 
 \left. \max
  \left\{
    0,  8 \Delta_a^{-2} \ln(T+\sum\nolimits_{a=1}^K \lceil {N}_a \rceil) - 
    \lfloor{{N}_a} \rfloor
  \right\}  \right),
\end{align*}
where 
${N}_a {=}  {
    \left( \sum_{i\in [-I]} p_i(a_i)^{-1} \indicator{a_i{=}a} \right)^2
  }/{
     \sum_{i\in [-I]} \left( p_i(a_i)^{-1} \indicator{a_i{=}a}
      \right)^2 
  }
$.
 \label{mthm:ipsw} 
 \end{restatable}

\noindent
Theorem \ref{mthm:ipsw} quantifies the impact of the logged data 
on the regret of the algorithm $\CA_{\text{UCB+IPSW}}$.  
Recall that ${N}_a$ is the
{\em effective sample size} of feedbacks for action $a$.
When there is no logged data, i.e. ${N}_a=0$, the regret bound reduces
to the $O(\log{T})$ bound of UCB.
A larger
${N}_a$ indicates a lower regret bound.
Notice that the number ${N}_a$ depends on the distribution of
logged data items' propensity scores. 
In particular, when all the propensity
scores are a constant $\tilde{p}$, i.e. $p_i(a_i){=}\tilde{p}$ for $\forall i$, the effective
sample size is the actual number of samples with action $a$, i.e. 
$N_a{=}\sum_{i\in[-I]}\indicator{a_i=a}$.
When the
propensity scores $\{p_i(a_i)\}_{i\in[-I]}$ have a more skewed distribution, the number ${N}_a$ will be smaller, leading to a larger
regret bound.  

Note that our framework is not
limited to the above instances.
 One can replace the online bandit oracle with $\epsilon$-greedy~\cite{kuleshov2014algorithms}, EXP3~\cite{auer2002nonstochastic}
or Thompson sampling~\cite{agrawal2012analysis}.
One can also replace the offline evaluator with balanced
weighting~\cite{kallus2018balanced} or supervised learning~\cite{zhang2019warm}.  
In Section~\ref{sec:experiments}, we will discuss more algorithms in the experiments.

\subsection{Relaxation of Assumptions on Logged Data}
The above theorems require the logged  data to satisfy the stable-unit Assumption~\ref{asum:StableUnit:offline} and ignorability Assumption~\ref{assumption:ignorability}.  
To see the impact of removing the Assumption~\ref{assumption:ignorability},
consider Example~\ref{example:Introduction:DecisionProb}. Let's say the logs do not
record {\em users' preferences to video}. In this case, our
{\em causal inference} strategy will calculate
the {\em empirical average}. Then, it
will select the wrong action of placing ad below videos.
The following theorem gives the regret upper bound when the assumptions on the
logged data do not hold.
\begin{restatable}[Removing assumptions on logged data]{theorem}{noIgnorability}
Suppose Assumptions~\ref{asum:StableUnit:offline} and \ref{assumption:ignorability} were removed.
Suppose the offline evaluator $\mathcal{E}$ returns $\{y_j\}_{j=1}^N$ w.r.t. $\{(\bm{x}_j{,}
  a_j)\}_{j=1}^N$. The bias of the average outcome
 for action $a$ is denoted as 
\[
\delta_a {\triangleq} ( {\sum\nolimits_{j=1}^N  
 \indicator{a_{j}=a} y_{j}}) / ( {\sum\nolimits_{j=1}^N \indicator{a_{j}=a}}) {-} \MBE[y|a]. 
\]
Suppose the reward $y$ is bounded in $[0,1]$.
Denote the number of samples for
action $a$ as $N_a{\triangleq} {\sum_{j=1}^N \indicator{a_{j}=a}}$.
Then, 
\begin{align*}
  R(T,\CA_{\mathcal{O} + \mathcal{E}}) 
\le
&
\sum\nolimits_{a\ne a^\ast} 
\Delta_a 
\left( 16 \Delta_a^{-2} \ln(N_a{+}T) 
\right.
\\
&
\left. 
- 2N_a ( 1- \Delta_a^{-1} \max\{0,\delta_a {-} \delta_{a^\ast}\} ) {+} (1 + \pi^2 / 3 ) \right)
\end{align*}
\vspace{-0.15in}
\label{thm:no_ignorability}
\end{restatable}

\noindent 
Theorem \ref{thm:no_ignorability} states the relationship between the bias of
the offline evaluator (i.e. $\delta_a$) and the algorithm's regret.   
When Assumptions~\ref{asum:StableUnit:offline} and \ref{assumption:ignorability} hold, the bias $\delta_a{=}0$.
In this case, the bound in
Theorem~\ref{thm:no_ignorability} is similar to the previous bounds in Theorem~\ref{thm:ps_matching} except that we
raise the constant from $8$ to $16$. When $\delta_a{-}\delta_{a^\ast}>0$,
i.e., the offline evaluator has a greater bias for an inferior
action than the bias of the optimal action, 
the regret upper bound becomes
larger compared to the case when the offline evaluator is unbiased. In
Theorem~\ref{thm:no_ignorability}, we also have a sufficient
condition for ``the logged data to reduce the regret upper bound'', i.e. $1{-}{\max\{0,\delta_a {-} \delta_{a^\ast}\} }/{\Delta_a}{>}0$, or, $\delta_a {-} \delta_{a^\ast}{<} \Delta_a$ for $\forall a{\ne} a^\ast$. 
The physical meaning is that when the estimated reward of the
optimal action is greater than that of other actions, the logged
data help to identify the optimal action and reduce the regret.

\section{Case Study II: Contextual Decision}
\label{sec:contextual_algorithm}

We first consider the case that the mean of the outcome is  
parametrized by a linear function.  
Then, we generalize it to non-parametric functions, 
where we design a forest-based online bandit
algorithm and prove its regret upper bound.   
To the best of our knowledge, it is the first regret upper bound 
for forest-based online bandit algorithms.  

\subsection{Linear Regression + LinUCB}

We consider that the mean of outcome follows a linear function: 
\begin{align}
  y_t = \bm{\theta}^T \phi(\bm{x}_t, a_t) + \epsilon_t && \forall t\in[T],
  \label{eq:linear_form} 
\end{align}
where $\phi(\bm{x},a)\in \mathbb{R}^d$ is an
$d$-dimensional {\em known} feature vector.  
The $\bm{\theta}$ is an $d$-dimensional {\em unknown} parameter to be learned,
and $\epsilon_t$ is a stochastic noise with $\mathbb{E}[\epsilon_t]{=}0$. 
We consider the case that Algorithm~\ref{alg:framework} 
uses ``LinUCB'' (outlined in BanditOracle~\ref{alg:linUCB}) as the online bandit oracle and ``linear regression'' (outlined in OfflineEvaluator~\ref{alg:linear_match}) as the offline evaluator.  
We denote this instance of Algorithm~\ref{alg:framework} as 
$\CA_{\text{LinUCB+LR}}$.   
BanditOracle~\ref{alg:linUCB} uses the LinUCB 
(Linear Upper Confidence Bound \cite{li2010contextual}) 
to make contextual online decisions.    
It estimates the unknown parameter
$\hat{\bm{\theta}}$ based on the feedbacks. 
The $\hat{y}_a{\triangleq} \hat{\bm{\theta}}^T
\phi(\bm{x},a) {+} \beta_t\sqrt{\phi(\bm{x},a)^T \bm{V}^{-1} \phi(\bm{x},a)}$ is
the upper confidence bound of reward, where $\{\beta_t\}_{t=1}^T$ are
parameters. 
The oracle always plays the action
with the largest upper confidence bound.

\begin{bandit}
  \caption{\bf LinUCB \cite{li2010contextual}}\label{alg:linUCB}
      {\bf Member variables:} a matrix $\bm{V}$ (initially $\bm{V}$ is a $d\times d$ matrix),
  a $d$-dimensional vector $\bm{b}$ (initially $\bm{b}{=}\bm{0}$ is zero),
  initial time $t{=}1$ 
  
  \MFn{\FPlay{$\bm{x}$}}{
    $\hat{\bm{\theta}} \gets \bm{V}^{-1}\bm{b}$\\
    \For{$a\in [K]$}{
      $\hat{y}_a \gets \hat{\bm{\theta}}^T \phi(\bm{\bm{x}, a}) +
      \beta_t \sqrt{\phi(\bm{x},a)^T \bm{V}^{-1} \phi(\bm{x},a)}$
    }
    \Return $\arg\max_{a\in[K]} \hat{y}_a$
  }
  \MFn{\FUpdate{$\bm{x}, a, y$}}{
    $\bm{V}\gets \bm{V} + \phi(\bm{x},a)\phi(\bm{x},a)^T$, \hspace{0.1in}
    $\bm{b} \gets \bm{b} + y\bm{x}$, \hspace{0.1in}
    $t\gets t+1$
  }
\end{bandit} 

\noindent
OfflineEvaluator~\ref{alg:linear_match} uses linear regression 
to synthesize feedbacks from the logged data.
From the logged data, it estimates the parameter $\hat{\bm{V}}$ (Line 3), 
and the parameter $\hat{\bm{\theta}}$ (Line 4).
It returns the estimated outcome $\phi(\bm{x},a)^T
\hat{\bm{\theta}}$ according to a linear model.
It stops returning outcomes when the logged data cannot
provide a tighter confidence bound than that of the online bandit oracle (Line 6
- 9).
 
\begin{evaluator}
  \caption{\bf Linear Regression (LR)}\label{alg:linear_match}  
  {\bf Member variables:} $\bm{V}, \hat{\bm{V}}$ are $d\times d$ matrices,
  where $\bm{V}$ ($\hat{\bm{V}}$) is for the online (offline) confidence bounds.
  $\hat{\bm{\theta}}$ is the estimated parameters. 
 The $\bm{V}$ is shared with LinUCB oracle.

  \SetKwFunction{InitLin}{{\bf \_\_init\_\_}}
  \MFn{\InitLin{$\CL$}}{
    $\hat{\bm{V}}\gets \bm{I}_d + \sum_{i\in [-I]} \phi(\bm{x}_i,a_i) \cdot
    \phi(\bm{x}_i,a_i)^T$
    \mycomment{ $\bm{I}_d$ is the $d{\times} d$ identity matrix}
    \\
    $\bm{b}\gets \sum_{i\in [-I]} y_i \cdot \phi(\bm{x}_i,a_i)$,
    $\hat{\bm{\theta}}\gets \hat{\bm{V}}^{-1} \bm{b}$
  }
  \SetKwFunction{FLinear}{{\bf get\_outcome}}
  \MFn{\FLinear{$\bm{x},a$}}{
    \If{$||\phi(\bm{x},a)||_{\bm{V}+\phi(\bm{x}_i,a_i) \cdot
    \phi(\bm{x}_i,a_i)^T} > ||\phi(\bm{x},a)||_{\hat{\bm{V}}}$}{
      $\bm{V}\gets \bm{V}+\phi(\bm{x}_i,a_i) \cdot \phi(\bm{x}_i,a_i)^T$\\
      \Return $\phi(\bm{x},a) \cdot \hat{\bm{\theta}}$
    }
          \Return NULL
      }
\end{evaluator}

Suppose for any context $\bm{x}_t$, the difference of
expected rewards between the best and the {\em ``second best''} actions is at least $\Delta_{\min}$.  
This is the settings of section 5.2 in the paper \cite{abbasi2011improved}.  
In the following theorem, we derive a regret upper bound for $\CA_{\text{LinUCB+LR}}$. 

\begin{restatable}[LinUCB+Linear regression]{theorem}{linearProblemDependent}
Suppose the rewards satisfy the linear model in Equation (\ref{eq:linear_form}).
Suppose offline evaluator returns a sequence $\{y_i\}_{i=1}^N$ w.r.t. $\{(\bm{x}_i,a_i)\}_{i=1}^N$. Let
$\bm{V}_N {\triangleq} \sum_{i\in[N]} \bm{x}_i\bm{x}_i^T$,
$L{\triangleq}\max_{t{\le}T}\{||\bm{x}_t||_2\}$. 
Moreover, the random noise is $1$-sub-Gaussian, i.e. $\MBE[e^{\alpha \epsilon_t}]\le \exp(\alpha^2/2)$, $\forall \alpha\in \MBR$.
Then 
\begin{align*}
{R}(T,\CA_{\text{LinUCB+LR}}) \le
 \frac{8d^2(1+2\ln(T))}{\Delta_{\min}} \log\left(1+\frac{TL^2}{\lambda_{\min}(\bm{V}_N)}\right) + 1.
\end{align*}
When the
smallest eigenvalue $\lambda_{\min}(\bm{V}_N)$ is greater than a threshold $(1/2{+}\ln(T))T L^2$, 
the regret is bounded by a constant $16d^2/\Delta_{\min}{+}1$. 
\label{mthm:linear_problem_dependent}
\end{restatable}

\noindent
Denote $\kappa {=} {TL^2}/{\lambda_{\min}(\bm{V}_N)}$ 
as the condition number. 
{
Theorem~\ref{mthm:linear_problem_dependent} implies that 
for a fixed $\kappa$, the regret in $T$ time slots
is $O(\log(T))$.}
Moreover, when the logged data contain enough 
samples, i.e., $\lambda_{\min}(\bm{V}_N)$ is greater than $(1/2+\ln(T))TL^2$,
regret is upper bounded by a constant.
Using our analytic framework, we observe a similar thresholding phenomena in
\cite{bu2019online} which focuses on the linear model.

\subsection{Non-parametric Forest-based Online Decision Making}

We generalize the linear outcome model 
(in Equation (\ref{eq:linear_form})) 
to the case that the mean of the outcome $y_t$ is 
a nonparametric function of $\bm{x}_t$.  
We use the non-parametric forest estimator to generalize algorithm $\CA_{\text{LR+LinUCB}}$ in two aspects: 
(1) replace the {\em LinUCB} with our forest-based online learning algorithm {\em $\epsilon$-Decreasing Multi-action Forest} (abbr. Fst) 
outlined in BanditOracle~\ref{alg:cf_oracle}; 
(2) replace {\em linear regression} with {\em Matching on Forest} (abbr. MoF) outlined in 
OfflineEvaluator~\ref{alg:forest_match}.  
We denote the new contextual decision algorithm 
as $\CA_{\text{Fst+MoF}}$.

\noindent{\bf $\epsilon$-decreasing multi-action forest (Fst).}
A multi-action forest $\CF$ is a set of $B$ multi-action decision trees.
{\NIPS It extends the regression forest of
  \cite{wager2018estimation} to consider {\em multiple
  actions} in a leaf. 
}
Each context $\bm{x}$ belongs to a leaf $L_b(\bm{x})$ in a tree $b{\in}[B]$,
{\NIPS and each leaf has multiple actions $a{\in}[K]$.}
Given the dataset $\CD {=} \{(\bm{x}_i, a_i, y_i)\}_{i=1}^D$, tree
$b$ estimates the outcome of an action $a$ under a context $\bm{x}$ as
\vspace{-0.05in}
\begin{align}
\hat{L}_b(\bm{x},a)\triangleq \frac{\sum\nolimits_{i\in [D]} \indicator{L_b(\bm{x_i}) =
  L_b(\bm{x})} \indicator{a_i=a} y_i}{\sum\nolimits_{i\in [D]} \indicator{L_b(\bm{x_i}) =
  L_b(\bm{x})} \indicator{a_i=a}}.
  \label{eq:forest_leaf_value}
\end{align}

\noindent 
BanditOracle~\ref{alg:cf_oracle} is the {\em $\epsilon$-decreasing multi-action forest} algorithm.
For a context $\bm{x}$, the algorithm first uses the average
of all trees as the estimated outcome (Line 4).
In the time slot $t$, with probability $1{-}\epsilon_t$, the algorithm chooses the action with the largest
estimated outcome. Otherwise, the algorithm randomly selects an action to
explore its outcome.
The oracle will update the data $\CD$ using the feedback (Line 8), and update
the leaf functions $\{L_b(\cdot)\}_{b=1}^B$ of the forest $\CF$ using the training algorithm in the paper~\cite{wager2018estimation} (Line 9).
\begin{bandit}
      \caption{\bf $\epsilon$-Decreasing\! Multi-action Forest (Fst)}\label{alg:cf_oracle}
    {\bf Variables:} the multi-action forest $\CF$ of $B$ trees, data
  $\CD$ with initial value $\emptyset$, $t$ with initial value $1$\\
  \Fn{\FPlay{$\bm{x}$}}{
    \For{$a\in[K]$}{ 
      $\hat{y}_a\gets \frac{1}{B} \sum_{b\in[B]} \hat{L}_b(\bm{x},a)$ 
    }
    $a_t{\gets}
    \begin{cases}
      \arg\max_{a\in[K]} \hat{y}_a & \text{w.p. } 1{-}\epsilon_t,
      \\
      \text{a random action in }[K] & \text{w.p. } \epsilon_t.
    \end{cases}
    $\\
    {\Return} $a_t$
  }
  \Fn{\FUpdate{$\bm{x},a,y$}}{
    $\mathcal{D}\gets \mathcal{D} \cup \{(\bm{x}, a, y)\}$ and $t\gets t+1$ \\
    $\CF\gets$\texttt{train\_forest}($\mathcal{D}$)\mycomment{learn tree splits.
    In practice, one can re-train the forest every $T_0$ time slots}
  }
  \end{bandit}

{\NIPS
\noindent
To analyze the regret of BanditOracle~\ref{alg:cf_oracle}, 
we need the following two definitions, which are
  adapted from Definition 2b and 4b of \cite{wager2018estimation}.
  \begin{definition}[honest]
    \label{def:honest}
   A multi-action tree on training samples $\{(\bm{x}_1,y_1, a_1), \ldots, (\bm{x}_s, y_s, a_s)\}$ is honest if (a) (standard-case) the tree
   does not use the responses $y_1,\ldots, y_s$ in choosing where to replace its
   splits; or (b) (double sample case) the tree does not use the responses in a
   subset of data called
   ``$\mathcal{I}$-sample'' to place splits, where {{\em``double
     sample''} and {\em ``$\mathcal{I}$-sample''} are defined in Section 2.4 of \cite{wager2018estimation}}.
  \end{definition}

  \begin{definition}[$\alpha$-regular]
    \label{def:alpha_regular}
    A multi-action tree grown by recursive partitioning is $\alpha$-regular for
    some $\alpha>0$ if either: (a) (standard case) (1) each split leaves at
    least a fraction $\alpha$ of training samples on each side of
    the split, (2) the leaf containing $\bm{x}$ has at least $m$ samples from
    each action $a\in [K]$ for some $m{\in} \mathbb{N}$, and (3) the
    leaf containing $\bm{x}$ has less than $2m-1$ samples for some action
    $a\in [K]$ or
    (b) (double-sample case) for a double-sample tree, (a) holds for the $\mathcal{I}$ sample.
  \end{definition}
}

\begin{restatable}[asymptotic regret of Fst]{theorem}{onlineForest}
{\NIPS  
Suppose that all potential outcome distributions $(\bm{x}_i,Y_i(a))$ for
  $\forall a\in[K]$ satisfy the same regularity assumptions as the pair
  $(\bm{x}_i,Y_i)$ did in Theorem~3.1 in \cite{wager2018estimation}\footnote{The condition is:
    $\mu(\bm{x},a)=\MBE[Y(a)|X=\bm{x}]$ and
    $\mu_2(\bm{x},a)=\MBE[Y(a)^2|X=\bm{x}]$ are Lipschitz-continuous, and
    finally that $\text{Var}[Y(a)|X=\bm{x}]>0$ and
    $\MBE[|Y(a)-\MBE[Y(a)|X=\bm{x}]|^{2+\delta}|X=\bm{x}]$ for some constants
    $\delta,M{>}0$ and for $\delta{=}1$, uniformly over all $\bm{x}{\in} [0,1]^d$.
    Here, we slightly modify the condition to add the case $\delta{=}1$.
  }.
  Suppose the trees in $\CF$ (Line 9) is honest,
  $\alpha$-regular with $\alpha\le 0.2$ in the sense of
  Definition~\ref{def:honest} and \ref{def:alpha_regular}, and symmetric
  random-split (in the sense of Definition 3 and 5 in
  \cite{wager2018estimation}).
}
 Denote
 $A{\triangleq}\frac{\pi^\prime}{d}\frac{\log((1-\alpha)^{-1})}{\log(\alpha^{-1})}$
 {\NIPS} where $\pi^\prime\in[0,1]$ is the constant ``$\pi$'' in Definition 3 of \cite{wager2018estimation}.
 Let $\beta {=}1{-}\frac{2A}{(2{+}3A)}$ and let the exploration rate to be $\epsilon_t{=}t^{-1/2(1-\beta)}$.
Then for any
small $\omega{>}0$, the asymptotic regret of Fst (do not use logged data) satisfies  
\begin{align*}
\lim\limits_{T\rightarrow +\infty} 
\frac{ R(T,\CA_{\text{Fst} + \CE_\emptyset } ) }{ T^{(1+\beta+\omega)/2} }= 0,
&& \text{hence  } 
\lim\limits_{T\rightarrow +\infty} 
\frac{ R(T,\CA_{\text{Fst} + \CE_\emptyset }) }{ T } = 0.
\end{align*}
\label{thm:online_causal_forest}
\end{restatable}

Theorem~\ref{thm:online_causal_forest} states that our 
online forest-based bandit algorithm $Fst$ achieves a sub-linear regret w.r.t. $T$. 
Note that our estimator can be biased.
We see by appropriate choices of the exploration rate
$\epsilon_t$, our algorithm $Fst$ balances both the
bias-variance tradeoff and the exploration-exploitation tradeoffs.
For readers who study causal inference, note that we do not need the
``overlap'' assumption~\cite{wager2018estimation} on the logged data. This is because our exploration
  probability $\epsilon_t$ ensures that each action is played with a
  non-zero probability.
\begin{comment}
We also notice that the regret bounds depends on $\alpha$. If users' context vectors are distributed uniformly
in the feature space (when $\alpha{=}0.5$), then the asymptotic regret upper
bound reaches the minimal.
\end{comment}

\noindent{\bf Matching-on-forest offline evaluator (MoF).}
OfflineEvaluator~\ref{alg:forest_match} describes the {\em Matching-on-Forest} offline evaluator.
It finds a ({\em weighted}) random ``nearest neighbor'' in the logs for the
context-action pair $(\bm{x}, a)$. For a decision tree $b\in [B]$, the ``nearest
neighbors'' of $(\bm{x}, a)$ is the data items in the same leaf
$L_b(\bm{x})$ which have the same action $a$.
If a data sample belongs to the
{nearest neighbors}
of $(\bm{x}, a)$ in more trees, then it will be returned by $MoF$ with a higher probability. 
  
\begin{evaluator}
      \caption{\bf  Matching on Forest ({\text{MoF}})}\label{alg:forest_match}
      {\bf Input:} a multi-action forest $\CF$ with leaf functions
  $\{L_b(\cdot)\}_{b=1}^B$, and the logged data $\CL$\\
    \Fn{\FGetOutcome{$\bm{x}, a$}}{
    $b \gets$ a uniformly random number in $\{1,2,\cdots, B\}$\\
    $\CI_{\text{matched}}\gets \{i~|~ L_b(\bm{x}_i) {=} L_b(\bm{x}), a_i{=}a\}$\\
    \If{$\CI \ne \emptyset$}{
      $i\gets $ a random sample from $\CI_{\text{matched}}$\\
      $\CL{\gets} \CL\backslash \{(\bm{x}_i,a_i,y_i)\}$\mycomment{delete item}\\
      {\Return} $y_i$
    }
         {\Return} NULL
      }
  \end{evaluator}

\section{Experiments} 
\label{sec:experiments}
We use real datasets from Yahoo, as well as synthetic data to
carry out our experiments\footnote{Code and Yahoo's data are in
  \cite{supplement}, which will be public once this paper is published.}.
First, we show that it is better to use both the logged data and the online feedbacks to
make decisions, compared with using just one of the data sources.
Second, we show why we need to judiciously use the logged data via our proposed method.   
Third, we discuss the practicability of our algorithms.

\subsection{Datasets and Experiment Settings}

{\bf Synthetic dataset.}
Each user's context $\bm{x}$ is drawn from $[-1,1]^d$ uniformly at random.  
Consider propensity scores $\MBP[\text{action}=a|\bm{x}]= ps(\bm{x},a)$ for all actions
$a\in\{0,\cdots, K-1\}$.   
Unless we vary it explicitly, 
we set the propensity score $ps(\bm{x},a)=\exp(s_a)/(\sum_{a=0}^{K-1} \exp(s_a))$ by default, 
where $s_a=\exp(-\bm{x}^T\bm{\theta}_a(\MBE[y|a] - \MBE[y|(a{+}1)\mod K]) )$.  
We generate the action $a\in\{0,\cdots, K{-}1\}$ according to the
propensity scores.  
We consider a reward function $y {=} f(\bm{x}, a)$ for each 
$(\bm{x}, a)$ pair.   
Unless we vary it explicitly, we set $f(\bm{x},a)=\bm{x}^T\bm{\theta}_a+b_a$
for some parameter $\bm{\theta}_a\in\MBR^d$ and bias $b_a=0.5\times a$.  
For the contextual-independent cases, 
the expected reward for an action $a$ is 
$\MBE[y | a] {=} \MBE_{\bm{x}}[f(\bm{x},a) | a]$ by marginalizing over the context $\bm{x}$.
By default, we set the number of arms as $K=3$.  
We present experiment results under other settings in our supplementary materials~\cite{supplement}.

\begin{figure*}
  \centering
  \begin{minipage}{0.245\linewidth}
   \captionsetup{width=0.94\textwidth}
    \includegraphics[width=\textwidth]{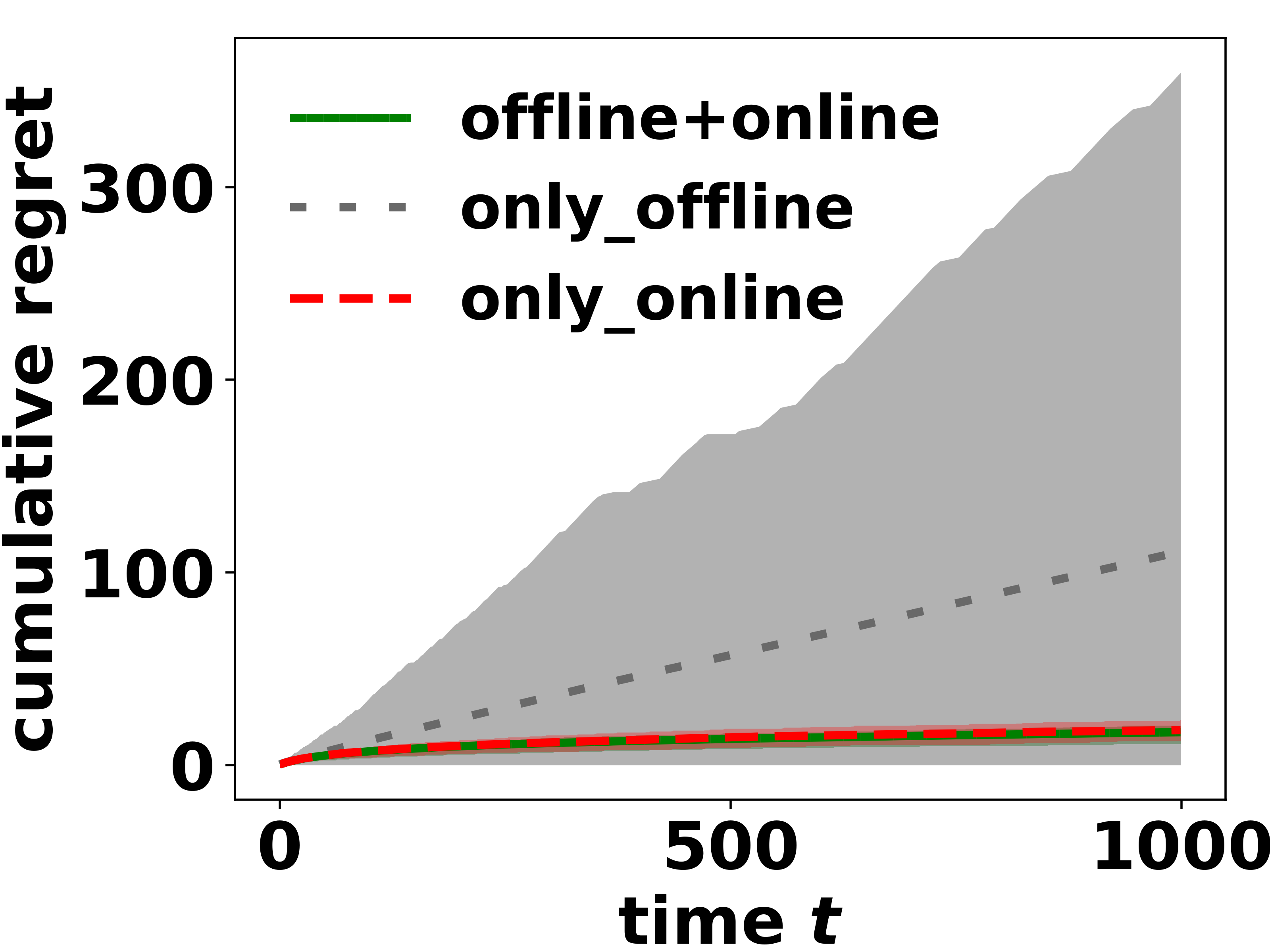} 
        \caption{Cumulative regrets of $\CA_{\text{UCB+EM}}$ \& variants ($K{=}2$)}
    \label{fig:exact_matching}
  \end{minipage}
  \begin{minipage}{0.245\linewidth}
   \captionsetup{width=0.97\textwidth}
    \includegraphics[width=\textwidth]{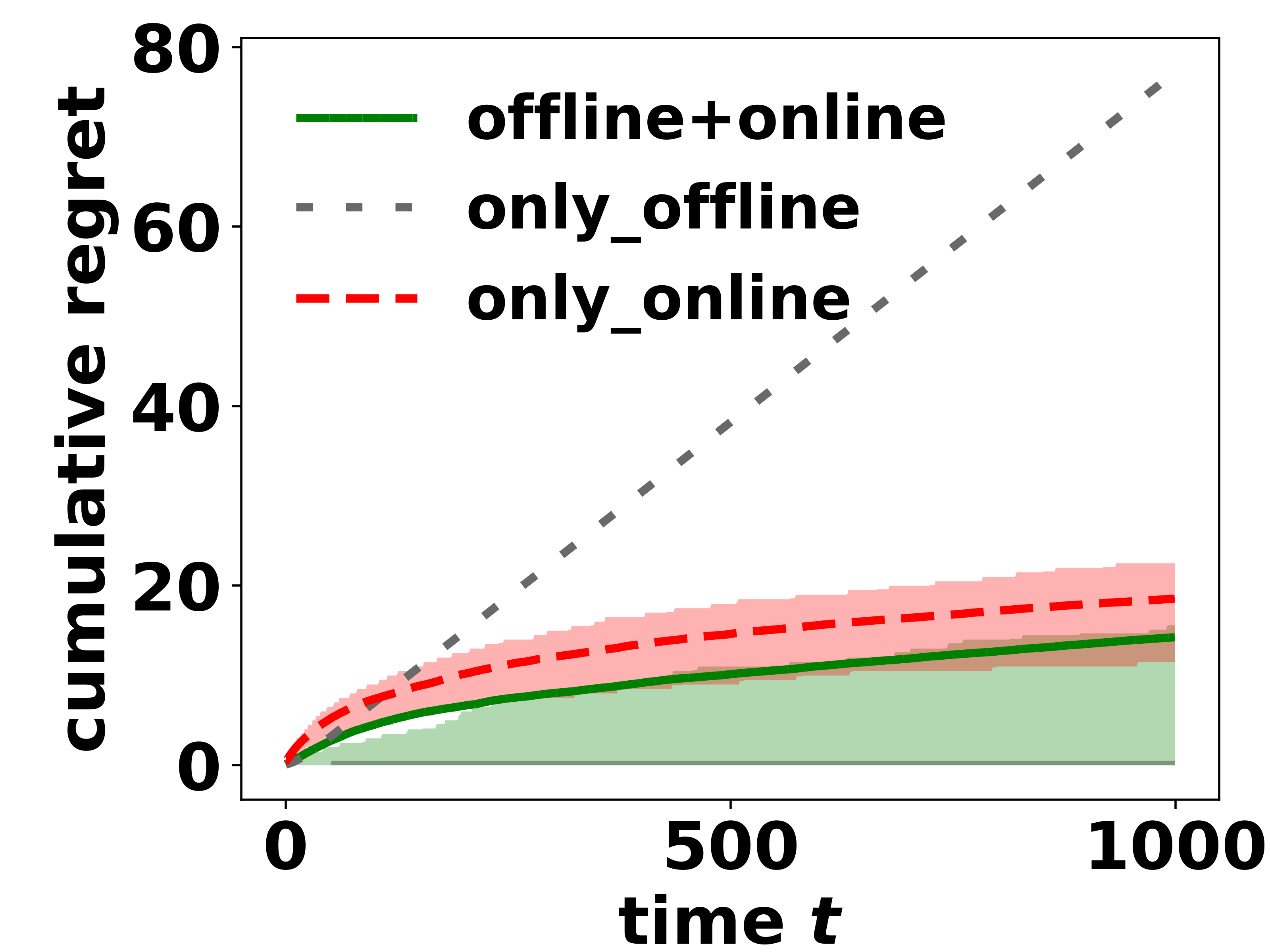} 
        \caption{Cumulative regrets of $\CA_{\text{UCB+PSM}}$ \& variants ($K{=}2$)}
    \label{fig:ps_matching}
  \end{minipage}
  \begin{minipage}{0.245\linewidth}
   \captionsetup{width=0.96\textwidth}
    \includegraphics[width=\textwidth]{{exp2_IPSW_UCB}.png} 
        \caption{Cumulative regrets of $\CA_{\text{UCB+IPSW}}$ and its variants}
    \label{fig:IPSW}
  \end{minipage}
  \begin{minipage}{0.245\linewidth}
   \captionsetup{width=0.95\textwidth}
    \includegraphics[width=\textwidth]{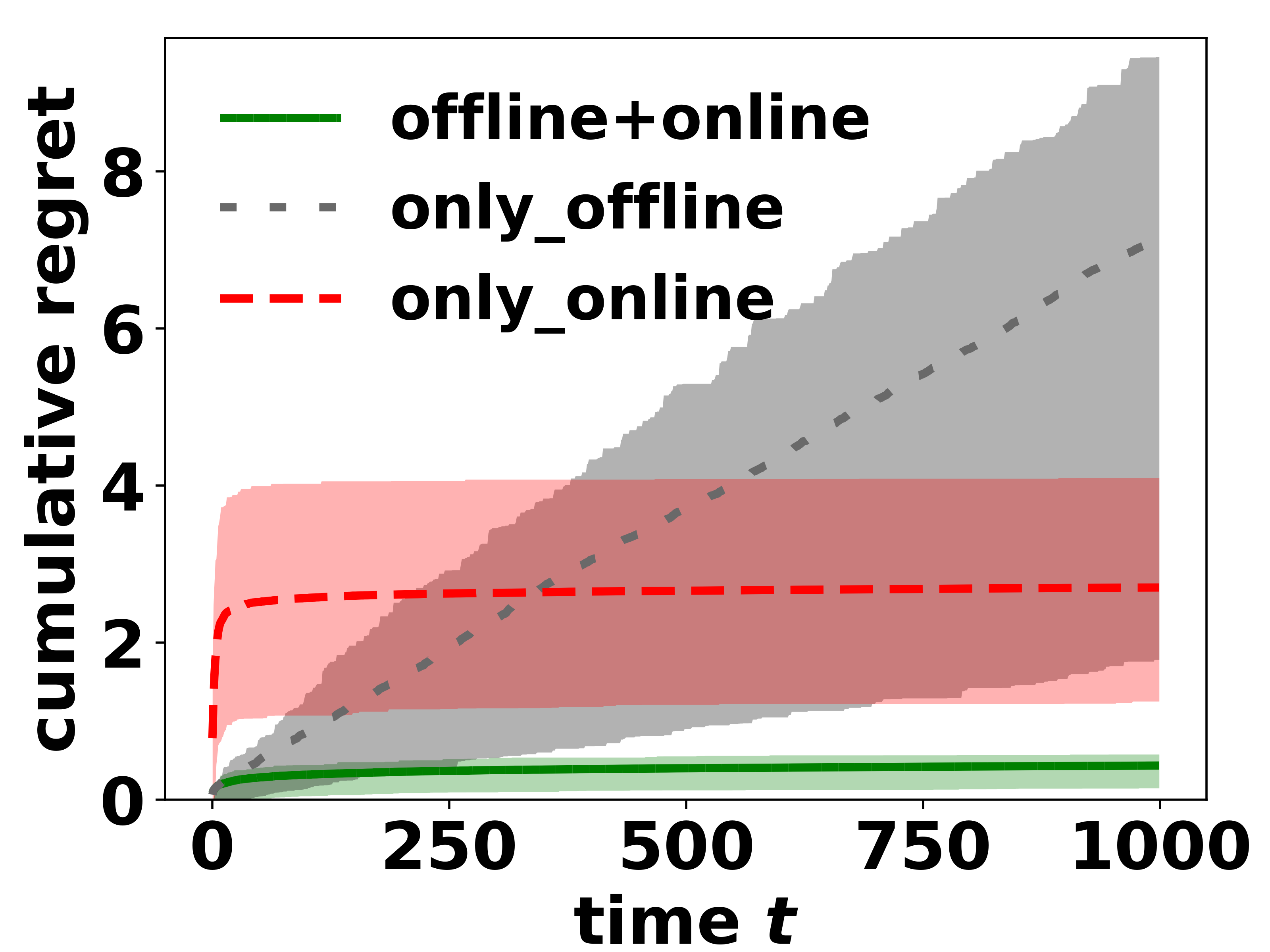} 
        \caption{Cumulative regrets of $\CA_{\text{LinUCB+LR}}$, linear $f$}
    \label{fig:linear}
  \end{minipage}
\end{figure*}

\begin{figure*}
  \centering
  \begin{minipage}{0.245\linewidth}
    \includegraphics[width=\textwidth]{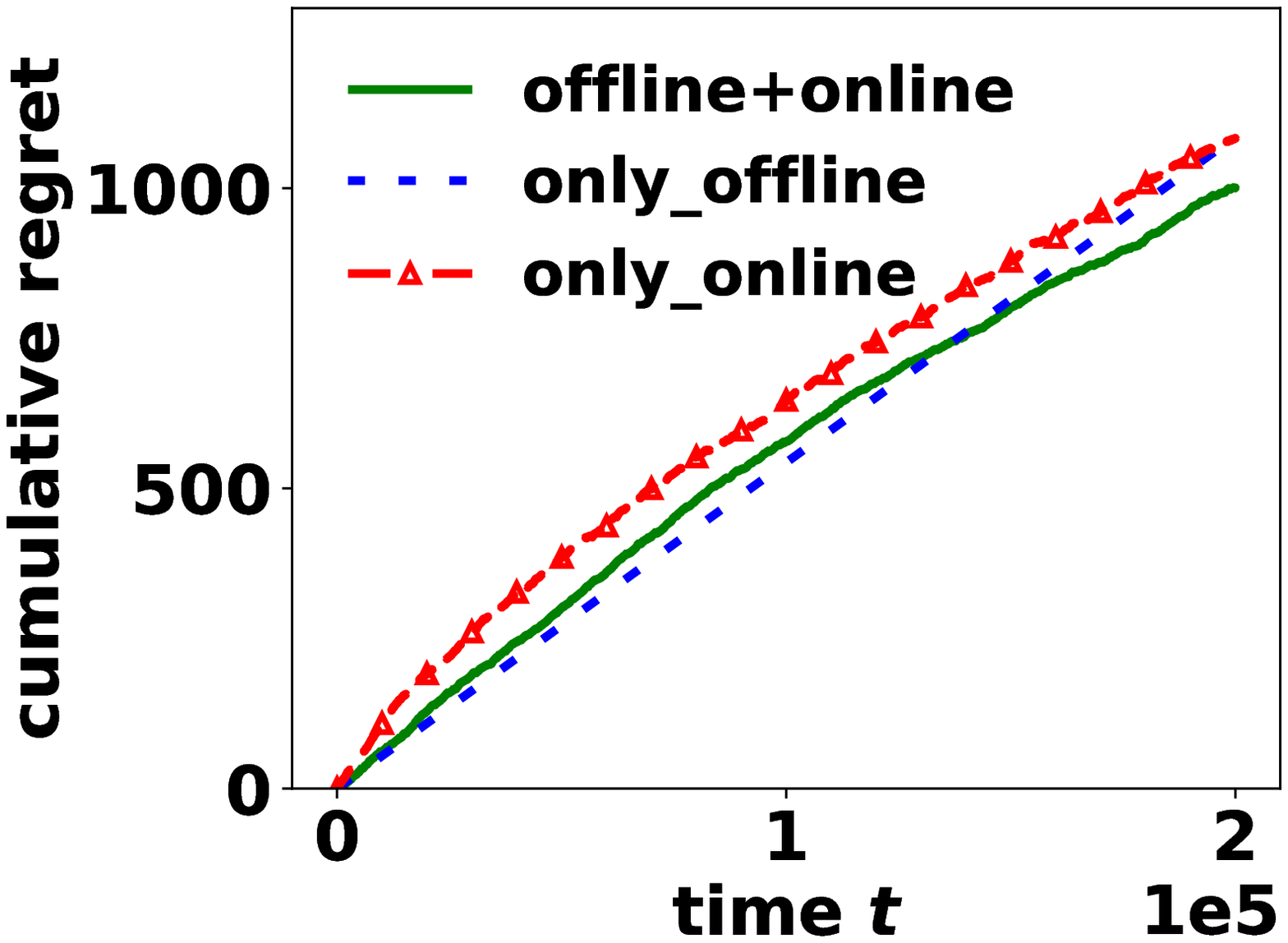}
        \captionsetup{width=0.95\textwidth}
    \vspace{-0.16in}
    \caption{Cumulative regret of $\CA_{\text{UCB+IPSW}}$ [Yahoo, context-independent]}
    \label{fig:yahoo_ipsw}
  \end{minipage}
  \begin{minipage}{0.245\linewidth}
    \includegraphics[width=\textwidth]{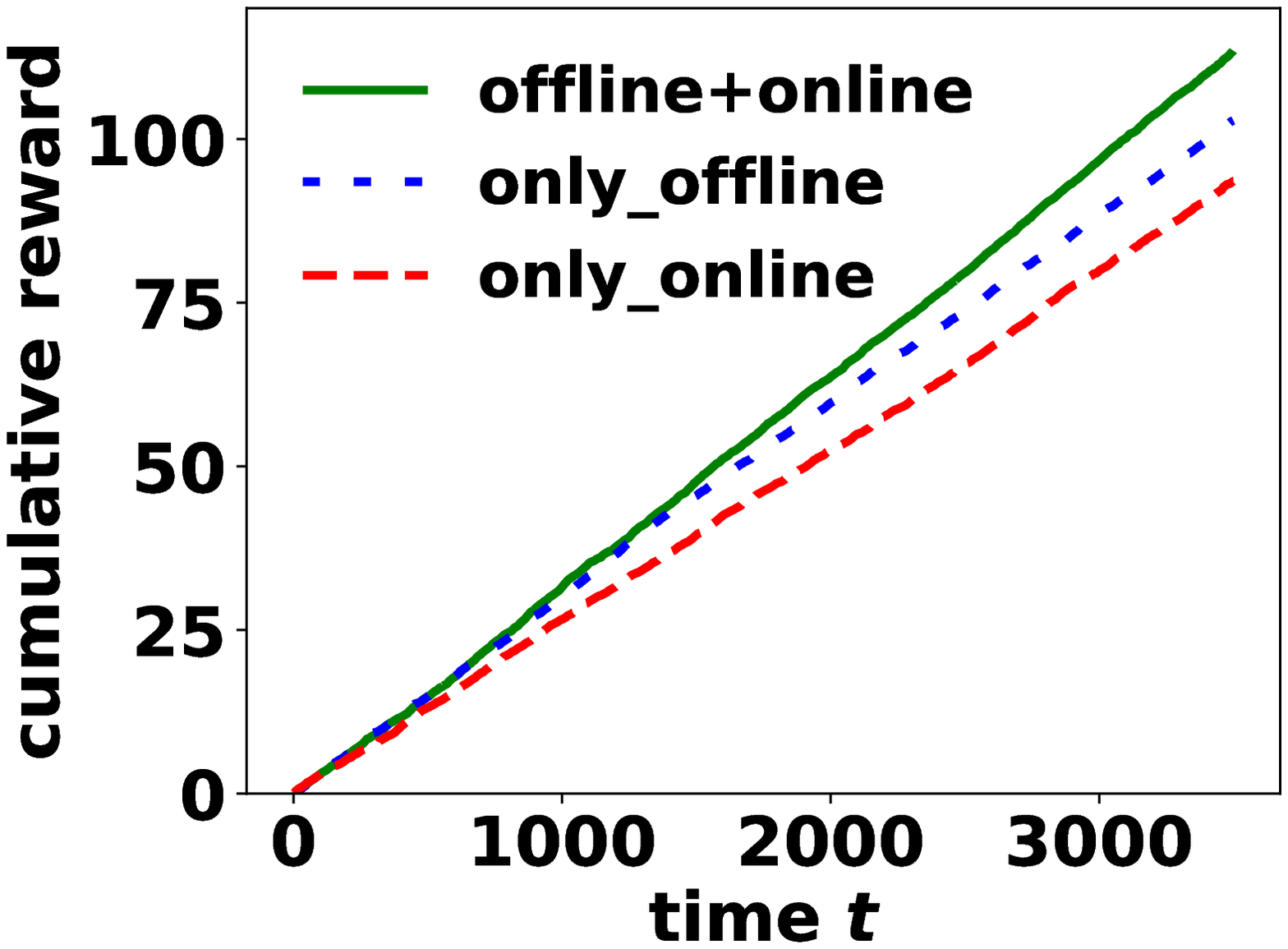}
    \vspace{-0.16in}
    \captionsetup{width=0.88\textwidth}
    \caption{Reward of $\CA_{\text{LinUCB+LR}}$ and its variants [Yahoo, contextual]}
  \label{fig:yahoo_linUCB}
  \end{minipage}
  \begin{minipage}{0.245\linewidth}
    \includegraphics[width=\textwidth]{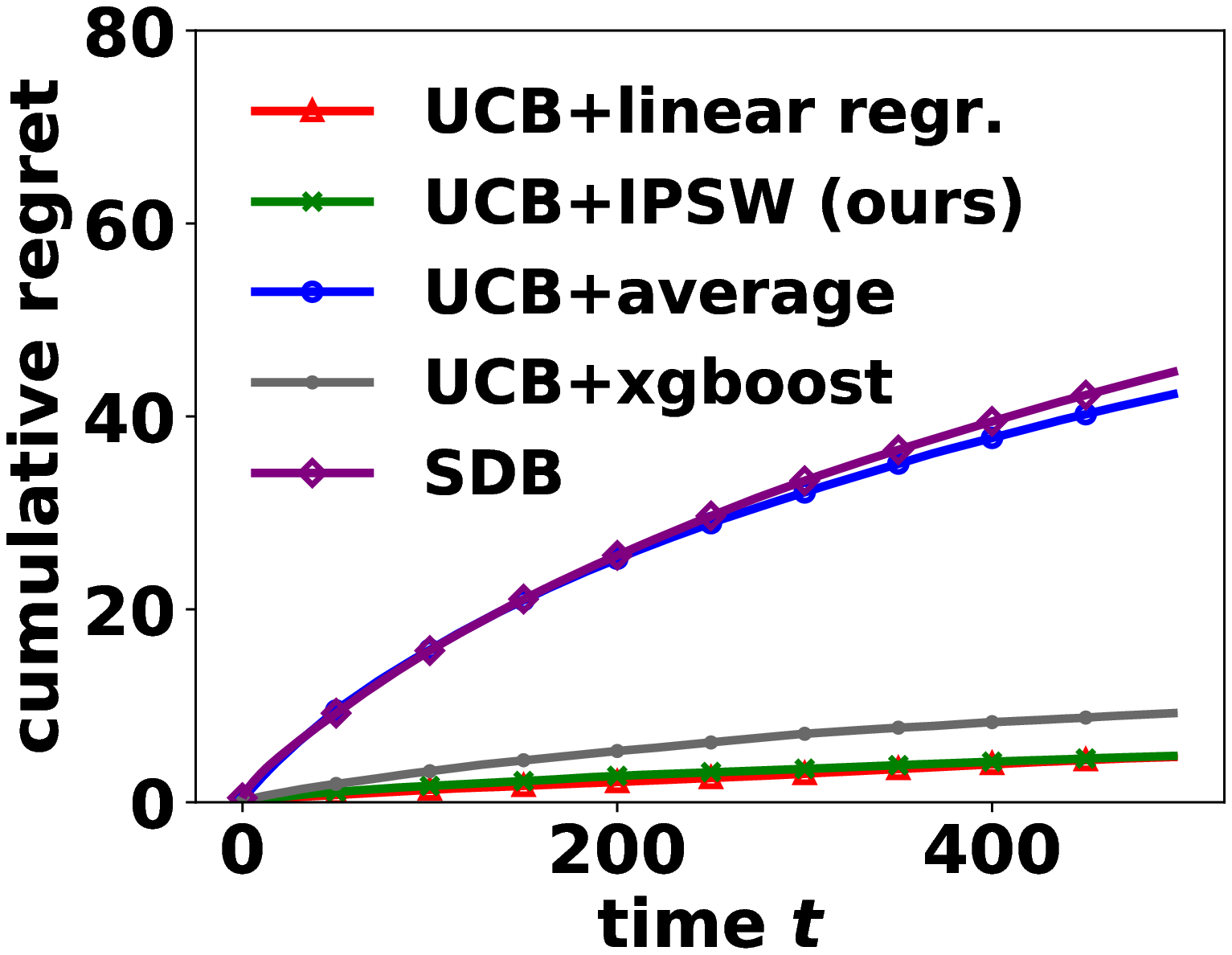}
    \vspace{-0.25in}
    \captionsetup{width=0.9\textwidth}
    \caption{Different algorithms on synthetic data, linear function
      $f$}
    \label{fig:supervised_UCB_continuous}
  \end{minipage}
  \begin{minipage}{0.245\linewidth}
    \includegraphics[width=\textwidth]{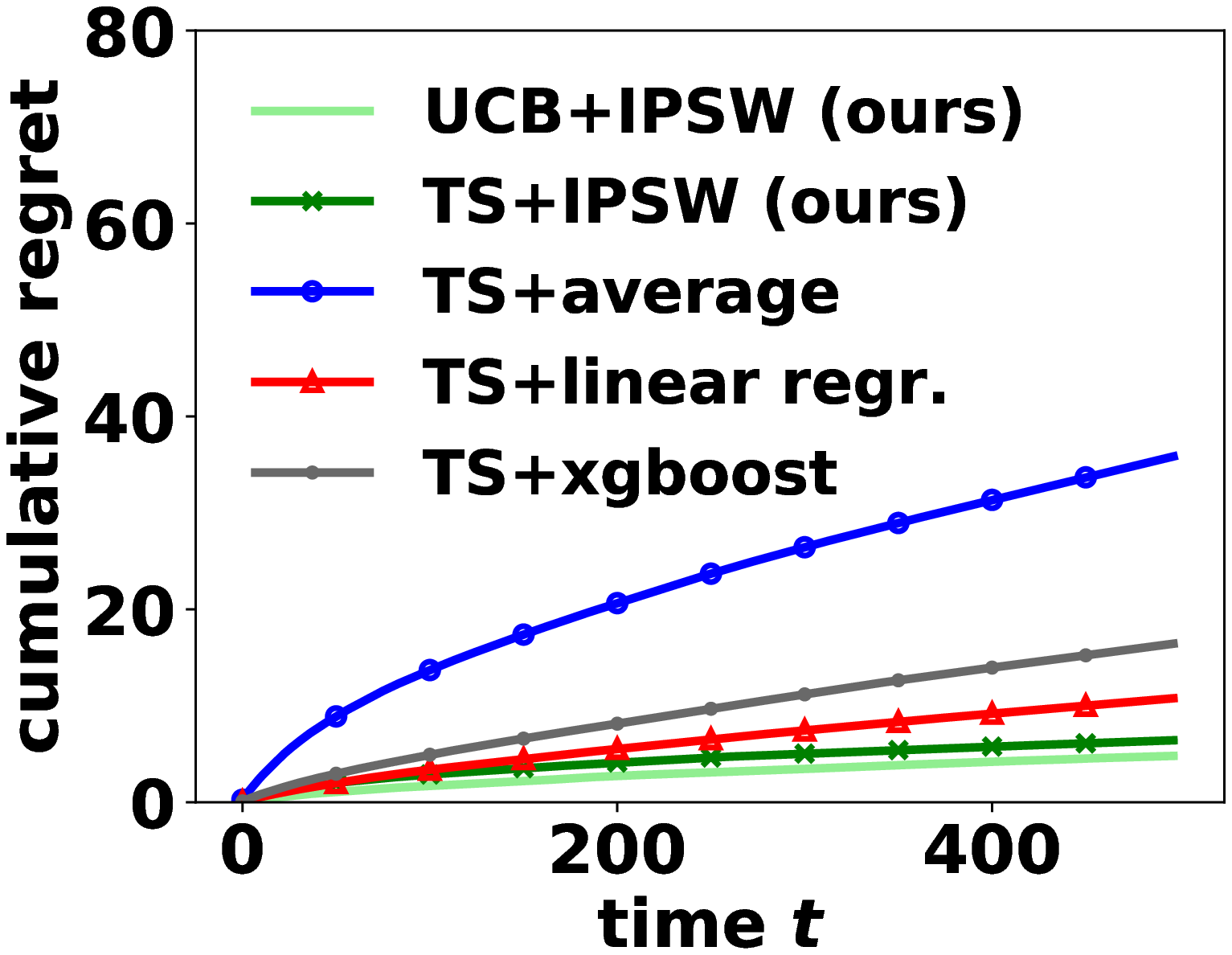}
    \vspace{-0.25in}
    \captionsetup{width=0.9\textwidth}
    \caption{Different algorithms using Thompson Sampling, linear function $f$}
    \label{fig:supervised_TS_continuous}
  \end{minipage}
\end{figure*}

\noindent
{\bf Yahoo's news recommendation data.}
The publicly available Yahoo's news recommendation dataset~\cite{yahoo} contains 100,000 rows of logs,
where we split 20\% of them as the logged data and 80\% of them as the online feedbacks. Each
row contains: (1) six user features, (2) candidate news IDs, (3) the selected
news ID, 
(4) whether the user clicks the news.
Since the user features in this dataset were learned via a linear
model~\cite{yahoo}, the Yahoo's data favors LinUCB~\cite{li2010contextual} for
contextual decisions.
We use the evaluation protocol of~\cite{li2010contextual} and run the algorithms for 50 times to take the average.

\begin{figure*}
  \begin{minipage}{0.245\linewidth}
    \includegraphics[width=\textwidth]{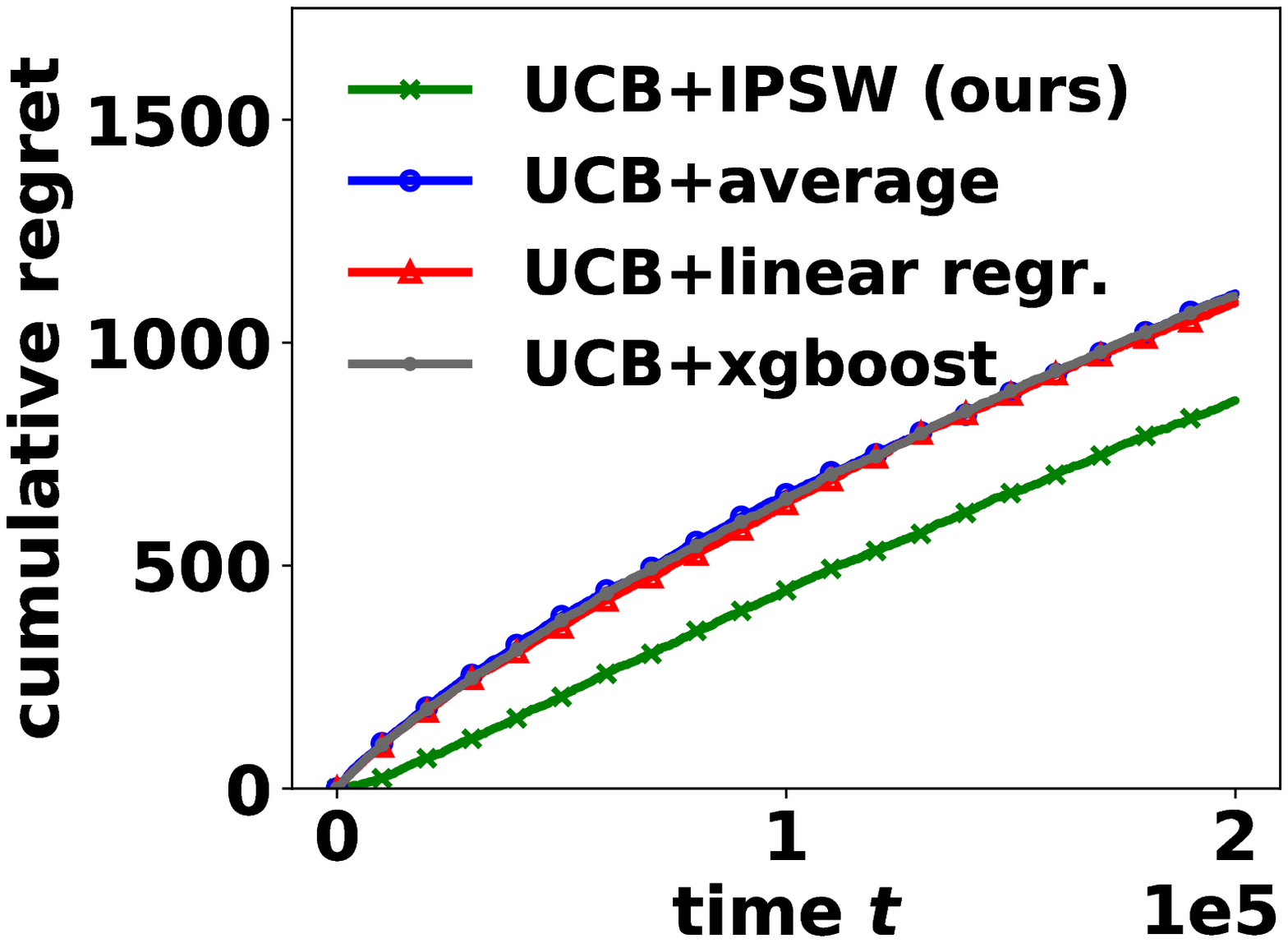}
    \vspace{-0.25in}
    \captionsetup{width=0.85\textwidth}
    \caption{Regrets of different algorithms [Yahoo, context-independent]}
    \label{fig:supervised_UCB_yahoo}
  \end{minipage}
  \begin{minipage}{0.245\linewidth}
    \captionsetup{width=0.88\textwidth}
    \includegraphics[width=\textwidth]{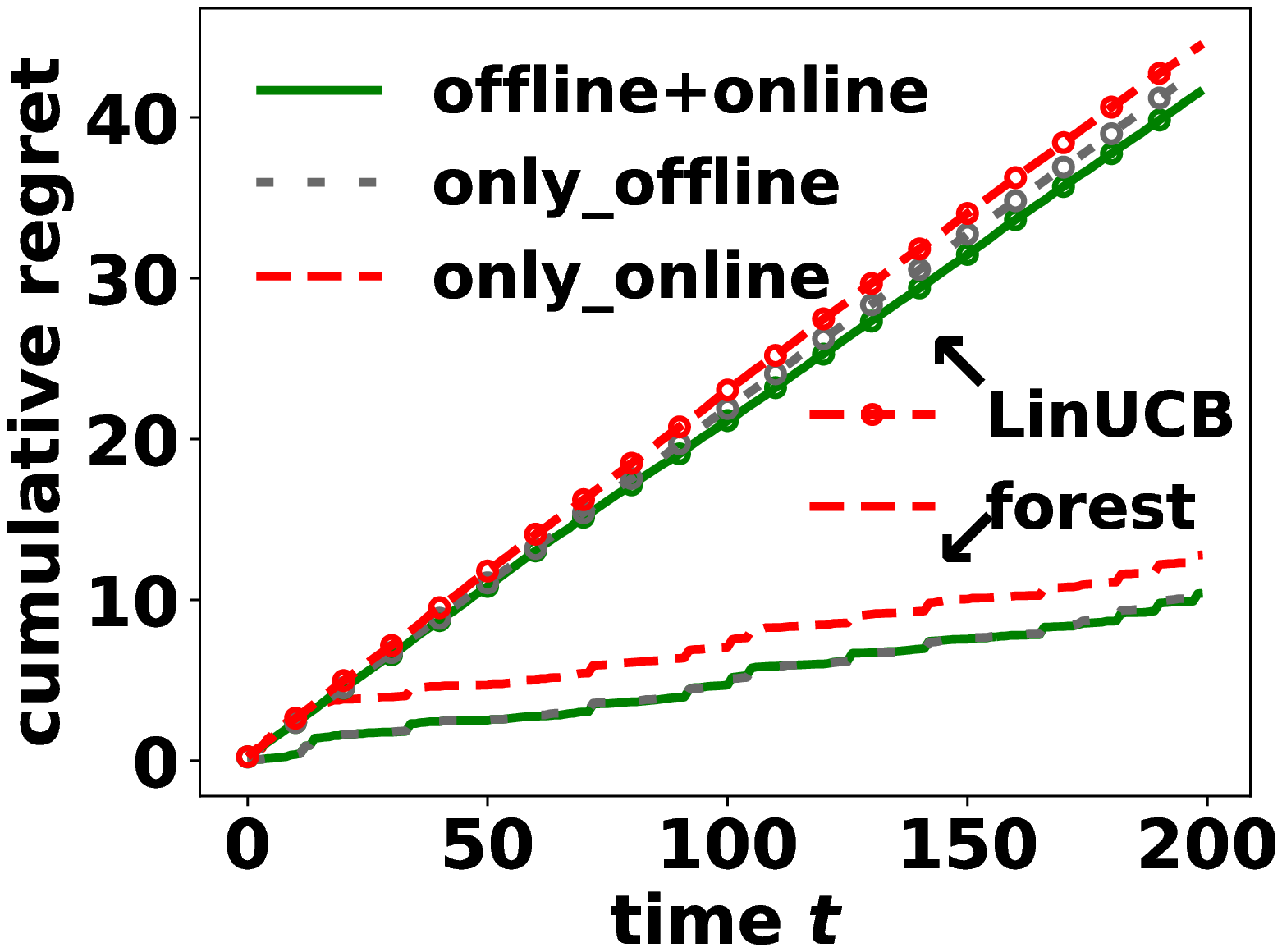}
    \vspace{-0.25in}
    \caption{Regrets of $\CA_{\text{Fst+MoF}}$, $\CA_{\text{LinUCB+LR}}$ and
      their variants,
      non-linear $\tilde{f}$}
    \label{fig:causal_forest}
  \end{minipage}
  \begin{minipage}{0.245\linewidth}
    \includegraphics[width=\textwidth]{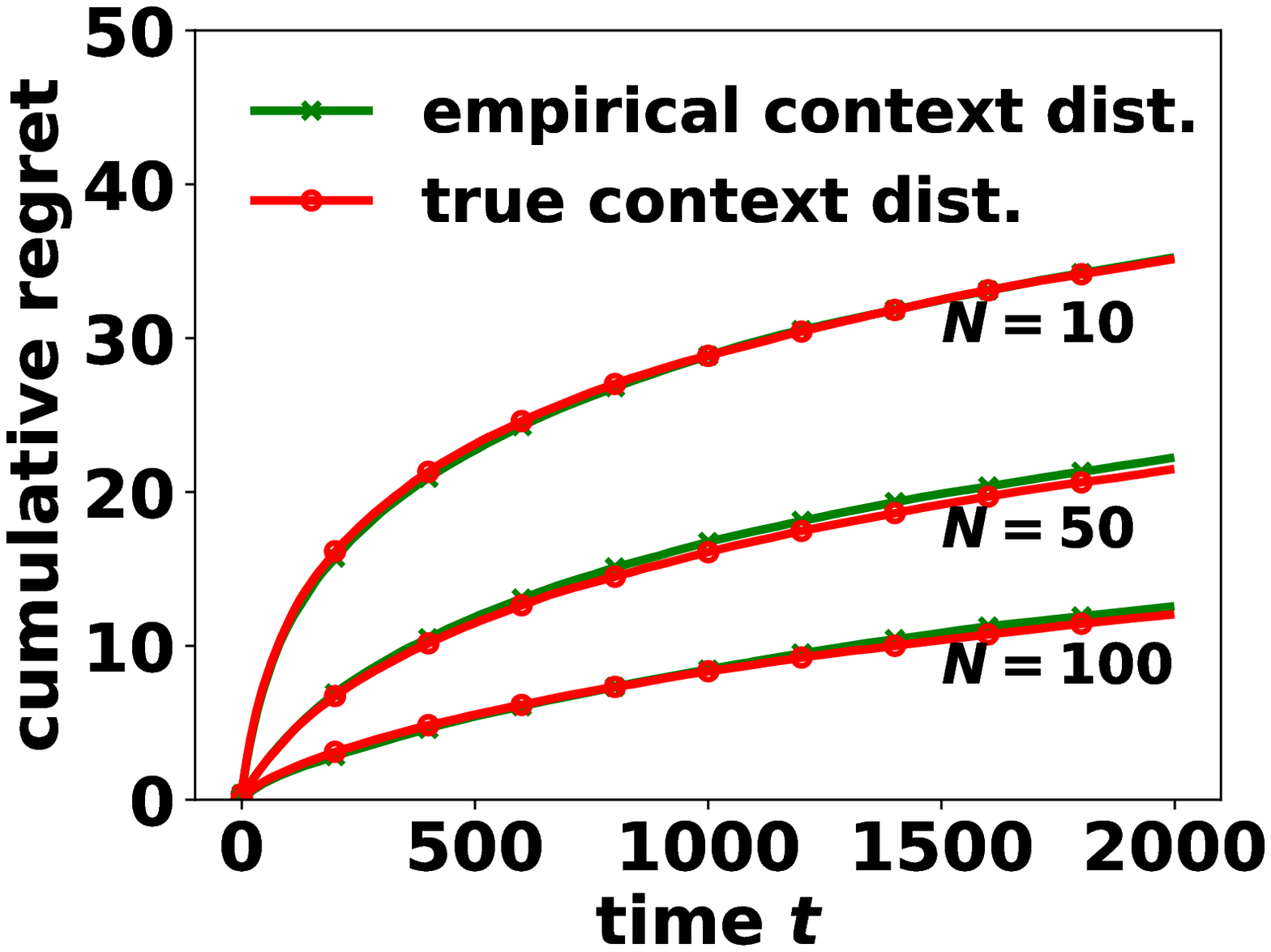}
    \vspace{-0.25in}
    \captionsetup{width=0.98\textwidth}
    \caption{Empirical context distribution
      ($\CA_{\text{UCB+IPSW}}$)}     \label{fig:empirical_context_dist}
  \end{minipage}
  \begin{minipage}{0.245\linewidth}
    \includegraphics[width=\textwidth]{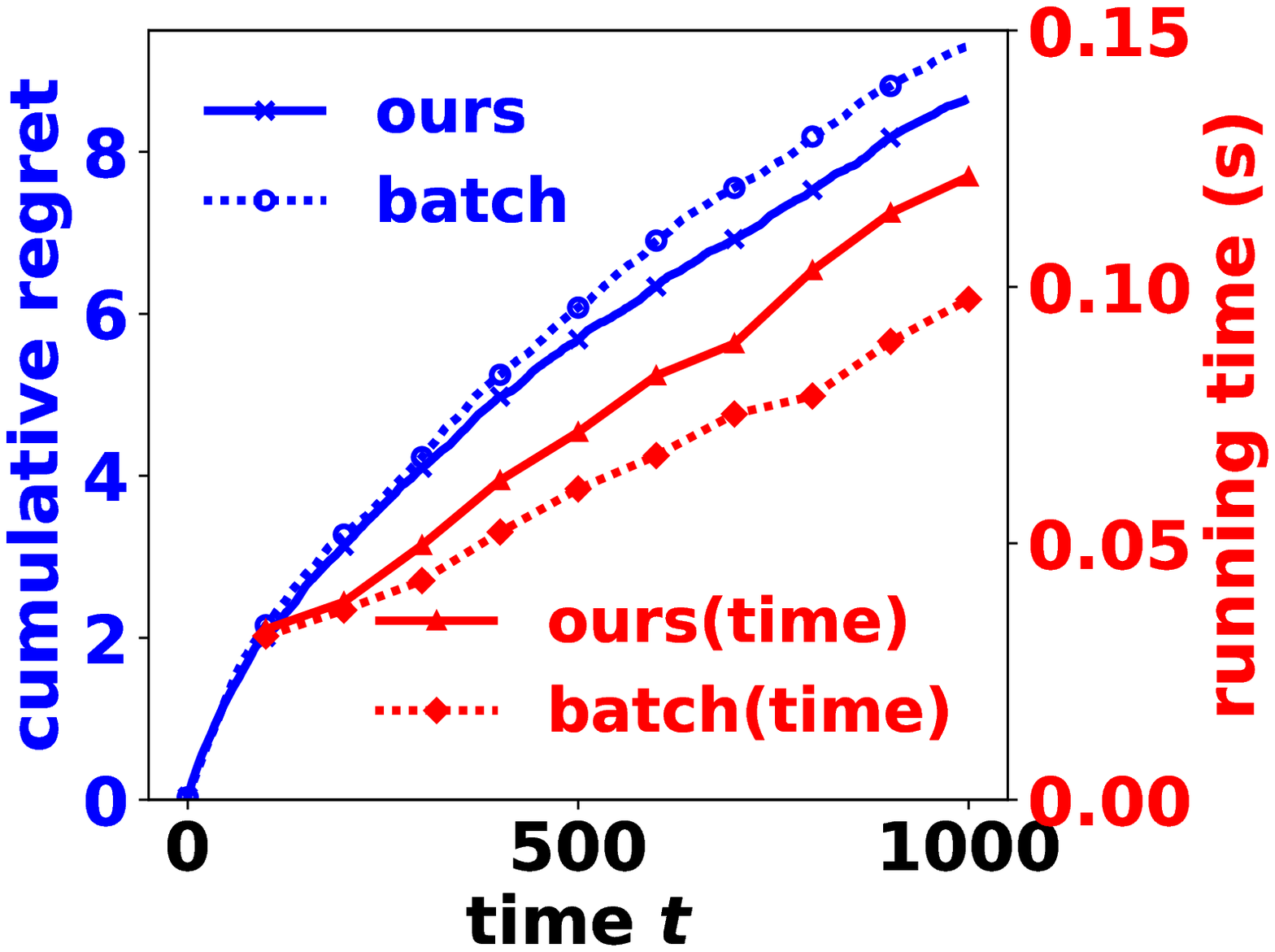}
    \vspace{-0.25in}
    \captionsetup{width=0.85\textwidth}
    \caption{Batch method vs. ours ($\CA_{\text{UCB+IPSW}}$)}
    \label{fig:batch_mode}
  \end{minipage}
\end{figure*}

\vspace{-0.05in}
\subsection{Using Both Offline and Online Data}

We compare the performance of algorithm $\CA_{\CO+\CE}$ 
(or $\CA_{\CO_c +\CE}$)
with its two variants that do not combine 
offline and online data: (1)
online bandit algorithm $\CO$ (or $\CO_c$)
that only uses online feedbacks; (2)
offline causal inference algorithm $\CE$
that only uses logged data.
{
}

\noindent
{\bf 
Exp1: Synthetic data.} 
We run each algorithm 500 times to get the average regret. We also plot the
20-80 percentiles as the confidence interval.
In Figure~\ref{fig:exact_matching}, \ref{fig:ps_matching} and \ref{fig:IPSW}, 
we have 100 logged data points.
We observe that 
our ``offline+online'' algorithms always have smaller regrets than the
``only\_online'' variants.  
This is because using logged data to
warm-start reduces the cost of online exploration.  
The regret for the  ``only\_offline'' version increases linearly in time,
with a large variance. This is because the
decisions can be either always right or always wrong depending on the initial decision.
In particular, in Figure~\ref{fig:ps_matching} and \ref{fig:IPSW}, the
80-percentile of the regrets for the ``only\_offline'' variants are always zero,
although the average regret is high.
We set $K=2$ for $\CA_{\text{UCB+EM}}$ and
$\CA_{\text{UCB+PSM}}$ because they cannot work well for more actions~\cite{supplement}.
We also set the context dimensions $d=2K$.
Figure~\ref{fig:exact_matching} shows that using the offline data does
not reduce the regret under the offline evaluator $EM$, because it
is difficult to find exactly matched logged data 
point for contexts in high dimensions.
In Figure~\ref{fig:ps_matching}, algorithm
$\CA_{\text{UCB+PSM}}$ improves the efficiency to use the logged data, 
and reduces the regret.  
Algorithm $\CA_{\text{UCB+IPSW}}$ can work for $K=3$ and further reduces the regret, as shown in Figure~\ref{fig:IPSW}.

We also investigate the contextual decision case. In Figure~\ref{fig:linear},
recall that by default our outcome function
$\MBE[y]=f(\bm{x},a)=\bm{\theta}_a^T\cdot\bm{x}$ is linear w.r.t. the contexts $\bm{x}$.
We see our ``offline+online'' algorithm $\CA_{\text{LinUCB+LR}}$ has the
smallest regret which is nearly zero,
because it uses the logged data to reduce the cost of online exploration.

\noindent
{\bf Exp 3: Yahoo's dataset. }
Figure~\ref{fig:yahoo_linUCB} shows that our 
``offline+online'' $\CA_{\text{LinUCB+LR}}$ improves the rewards by
21.1\% (or 10.0\%) compared to the ``only\_online'' {\em LinUCB} (or the
``only\_offline'' {\em LR} algorithm).

Although Yahoo's data were prepared to evaluate contextual
decisions~\cite{li2010contextual},
in Figure~\ref{fig:yahoo_ipsw} we restrict the decisions to be
context-independent. Our ``offline+online'' $\CA_{\text{UCB+IPSW}}$ has
a lower regret than the
``only\_online'' UCB algorithm. Our $\CA_{\text{UCB+IPSW}}$ has a lower regret
than the ``only\_offline'' IPSW algorithm when $T$ is large.

\noindent{\bf Lessons learned.}
Our algorithms that use both data sources achieve the
largest rewards or the smallest regret on both real and synthetic datasets, for
both context-independent and contextual decisions. 

\subsection{Proper Usage of the Offline Data}
\label{sec:exp:proper}
Besides our causal inference approach to use the offline logged data, there are
other heuristic methods which can use both data sources. 
We will show that our proposed method has a superior performance over the following heuristics.

\begin{enumerate}
\item {\bf Historical average in data (historicalUCB~\cite{shivaswamy2012multi}).}
This method uses the empirical averages of each action in the logged data as the initial
values for the online bandit oracle. 
\item {\bf Linear regression.} Instead of simply calculating the average,
  another way is to use supervised learning algorithm to {\em ``learn''} from offline
  data. The linear regression method learns a total number of $K$ linear models for each actions where
   features are the contexts and labels are outcomes.
\item {\bf Xgboost.}
  Xgboost~\cite{chen2016xgboost} is another supervised learning algorithm that
  often performs well for tabular data. The Xgboost method learns a total number
  of $K$ models for
  the $K$ actions.
\item {\bf Stochastic Delayed Bandits (SDB~\cite{mandel2015queue}).} Stochastic delayed bandit is a method
  proposed for bandit problem with delayed feedback. It can deal with bandit
  with logged data when we treat the logged data as the delayed feedbacks.
\item {\bf Thompson sampling with informed prior.} 
  Thompson sampling~\cite{agrawal2012analysis} is a Bayesian online decision
  algorithm. 
With logged data, one can use the historical data to give a
  prior distribution for each action. For example, one can use the average reward for each action to calculate the prior.
\end{enumerate}
All the above heuristics fall within our framework where
different heuristics to use the offline data are different offline evaluators.

\noindent 
{\bf Exp 4: Our method vs. others on synthetic data. }
Figure~\ref{fig:supervised_UCB_continuous} compares our algorithm and the
baseline heuristics (1)-(4) on the synthetic data. 
Recall that by default, the outcome $y=\bm{x}^T\bm{\theta}_a+b_a$ is the linear
function w.r.t. the context $\bm{x}$.
We observe that our algorithm $\CA_{\text{UCB+IPSW}}$ and 
the linear regression method have the smallest cumulative regret. 
The linear regression method performs comparatively
well because linear regression is unbiased when the reward is a linear function~\cite{LR_unbiased}.
Xgboost performs worse than our algorithm, 
because it cannot guarantee to unbiasedly estimate the rewards.
Using historical average to initialize UCB (i.e.
historicalUCB~\cite{shivaswamy2012multi}) or using the stochastic delayed bandit 
result in the highest regrets, because they ignore the impacts of the
confounders.

Figure~\ref{fig:supervised_TS_continuous} compares different heuristics to
get the informed prior for the Thompson Sampling (TS) algorithm~\cite{agrawal2012analysis}. All these
heuristics are instances in our framework where the
online learning oracle is Thompson Sampling. 
Our algorithms $\CA_{\text{TS+IPSW}}$ and $\CA_{\text{UCB+IPSW}}$ that use the causal inference
algorithm IPSW has the lowest regret.

\noindent
{\bf Exp6: Our method vs. others on Yahoo's data.}
In Figure~\ref{fig:supervised_UCB_yahoo}, we compare different algorithms'
regrets on Yahoo's data. 
Here, we randomly delete some data rows to simulate the selection bias in the
logged data. 
In particular, we delete a logged row with a probability of 0.9 if the average reward for the
chosen article is ranked among the top-3 and the reward is 1, or if the average reward for the
chosen article is not among the top-3 and the reward is 0.
We see that our algorithm $\CA_{\text{UCB+IPSW}}$ achieves the lowest regret under this setting.
The linear regression does not perform well because the reward in Yahoo's data
is not a perfectly linear function of the contexts~\cite{wang2016learning}.

\noindent{\bf Exp7: Linear vs. forest models for contextual decision.}
In Figure~\ref{fig:causal_forest}, we conduct experiments on synthetic data.
We set the reward  
$y = \tilde{f} (\bm{x}, a) {\triangleq}
(\sum_{j=1}^d\indicator{\bm{x}\ge (\bm{\theta}_a)_j})/d + 0.5{\times}
\indicator{a=1}$ to be a nonlinear function of the context $\bm{x}$, where $d=10$. 
We see our non-parametric forest-based algorithm $\CA_{\text{Fst+MoF}}$ can reduce the regrets of by over 75\%
(from around 40 to less than 10) compared to $\CA_{\text{LR+LinUCB}}$.

The features in Yahoo's dataset were learned using a linear model, and we compare the linear
and forest models in the supplement~\cite{supplement}.

\noindent{\bf Lessons learned.} One needs to use the offline data properly to
reduce the regret in decisions. Our methods that combine causal
inference and online bandit learning achieve the smallest regret.
For contextual decisions, when the reward is not a linear function of the context,
the forest-based model outperforms the linear model.

\subsection{Practical Considerations}

\noindent 
{\bf Exp8: Relaxing knowledge on context distribution. }
Recall that in our framework Algorithm~\ref{alg:framework}, we propose to use
the empirical distribution of the contexts from both offline and online data.
In Figure~\ref{fig:empirical_context_dist}, we compare the regret
using empirical and true context distribution using synthetic data, where we run
the algorithms for 2,000 time to take the average.
For various number of logged data $N\in\{10,50,100\}$, algorithms that use empirical context
distribution or true context distributions have similar regrets.
This shows the soundness to use empirical
context distribution in our framework.
We do not use real data, because for real data we do not know the true context distribution.

\noindent
{\bf Exp9: Comparison to batch method.}
One variant of our algorithmic framework is to use the logged data all in a batch before the online decisions. 
In contrast, in our Algorithm~\ref{alg:framework}, we use the logged data before each online
decision round $t$.
On synthetic data, Figure~\ref{fig:batch_mode} shows that our method and the batch method have
similar cumulative regrets, although our method is slightly better when $t$ is large.
The running time for the two methods increase linearly as the number
of online rounds $t$ increases. This shows that both methods are scalable w.r.t.
$t$.
Our Algorithm~\ref{alg:framework} has 
lower regret when $t$ is large, but is slower compared to its batch variant.
We also point out that the batch method do not have theoretical regret guarantee.
We do the comparison on real data in our supplement~\cite{supplement}.

\noindent{\bf Unobserved confounders.}
For real data Yahoo, probably we do not observe all the
confounders~\cite{wang2016learning}\cite{liu2013introduction}. Our experiments
show that in these real datasets, our algorithms still have the lowest
regrets. 
Please refer to our supplement~\cite{supplement} for more experiments discussing
the impact of unobserved confounders.

\section{Related works} 

Offline causal inference~(e.g. \cite{rubin2005causal}\cite{stuart2010matching}\cite{pearl2000causality}) 
focuses on observational logged data and asks
``what the outcome would be if we had done another action?''.  
Pearl formulated a Structural Causal Model (SCM) framework to model and infer causal effects\cite{pearl2000causality}. 
Rubin proposed another alternative,i.e., Potential Outcome (PO) framework\cite{rubin2005causal}.
Researchers propose various techniques for causal inference.
Matching (e.g. \cite{mccaffrey2004propensity}\cite{stuart2010matching}) and weighting
(e.g. \cite{austin2011introduction}\cite{kallus2018balanced}\cite{hansen1982large}) 
are techniques that deal with the imbalance of action's distributions in offline data.
Other techniques include ``doubly robust''\cite{dudik2011doubly} that
combines regression and causal inference, and
``differences-in-differences'' \cite{bertrand2004much}.
Recently, several works studied the individualized treatment effects~\cite{wager2018estimation}\cite{athey2019generalized}.
Offline policy evaluation is closely related to offline causal inference.  
It estimates the performance (or ``outcomes'') of a policy, 
which prescribes an action for each context~\cite{swaminathan2015counterfactual}\cite{li2015offline}.
We also use offline policy evaluation to evaluate the performances of 
contextual bandit algorithms\cite{li2012unbiased}.  
The offline policy evaluators can be used as the ``offline evaluator'' in our
framework. For example, the Inverse Propensity Score Weighting method in this
paper is commonly used in offline policy evaluation~\cite{swaminathan2015counterfactual}.
Our paper is orthogonal to the above works 
in that we focus on combining (or unifying) 
offline causal inference with online bandit learning algorithms 
to improve the online decision accuracy.  
Our work points out if we ignore the online feedbacks, these offline approaches
can have a poor decision performance.  
Offline causal inference algorithms can be seen as special cases of our
framework.  

Many works studied the stochastic multi-armed bandit problem.
Two typical algorithms are UCB \cite{auer2002finite} and Thompson sampling \cite{dong2018information}.  
LinUCB is a parametric variants of UCB \cite{dani2008stochastic} tuning for 
linear reward functions.   
For the contextual bandit problem, LinUCB algorithm has a regret of $O(\sqrt{T\log(T)})$ 
\cite{chu2011contextual}\cite{abbasi2011improved} and 
was applied to news article recommendation \cite{li2010contextual}.  
The {\em Thompson sampling causal forest} by~\cite{dimakopoulou2017estimation} and
{\em random-forest bandit} by \cite{feraud2016random} were non-parametric contextual bandit
algorithms, but these works did not provide regret bound.
{
Guan  \textit{et al.}  proposed a non-parametric online bandit
algorithm using {\em k-Nearest-Neighbor} \cite{guan2018nonparametric}. 
Our causal-forest based algorithm 
improves their bounds in a high-dimensional setting. 
} 
Lattimore \textit{et al.} used the causal structure of a problem to
find online interventions \cite{lattimore2016causal}.  
Our paper is orthogonal to the above works 
in that we focus on developing a generic framework 
to combine offline causal inference with these online bandit learning algorithms 
such that offline logged data can be used to speed up theses bandit algorithms 
with provable regret bounds.   
In addition, we propose a novel {\em $\epsilon$-greedy causal forest} algorithm, 
and prove regret upper bound for it  
(to the best of our knowledge, this is the first regret bound for forest based online bandit algorithms).

Several works aimed at using logged data to help online decision making. 
The historicalUCB algorithm \cite{shivaswamy2012multi} is 
a special case of our framework, while they ignored users' contexts. 
Bareinboim \textit{et al.} \cite{bareinboim2015bandits} 
and Forney \textit{et al.} \cite{forney2017counterfactual} combined the observational data, 
experimental data and counterfactual data, to solve the
MAB problem with unobserved confounders. 
They considered a
different problem of maximizing the ``intent-specific reward'', 
and they did not analyze the regret bound.
Zhang \textit{et al.} \cite{zhang2019warm} used adaptive weighting to robustly combine supervised
learning and online learning.
They focused on correcting the bias of supervised learning via online feedbacks,
while we use causal inference
methods to synthesize unbiased feedbacks to speed up online bandit algorithms. 
Our experiments in Section~\ref{sec:exp:proper} show that using
historicalUCB~\cite{shivaswamy2012multi}, SDB~\cite{mandel2015queue} or the supervised learning algorithm~\cite{zhang2019warm} to initialize the online
learning algorithms can result in higher regrets than our method.

\section{Conclusions} 
This paper studies how to use the logged data to make
better online decisions. 
We unify the offline causal inference and online bandit algorithms into a single framework, 
and consider both context-independent and
contextual decisions. 
We introduce five novel algorithm instances that incorporate causal
inference algorithms including matching, weighting, causal forest, and bandit
algorithms including UCB and LinUCB. 
For these algorithms, we present regret bounds under our framework. In particular, we give the first regret
analysis for a forest-based bandit algorithm.
Experiments on two real datasets and synthetic data show that our algorithms that can use both logged data 
and online feedbacks outperform algorithms that only use either of the data sources.
We also show the importance to judiciously use the offline data via our methods.

Our framework can alleviate the cold-start problem of online learning,
and we show how to use the results of offline causal inference to make online decisions.
 Our unified framework can be applied to all previous applications of offline causal inference
 and online bandit learning, such as A/B testing with logged data, recommendation
systems~\cite{wang2018deconfounded}\cite{li2010contextual} and online
advertising~\cite{bottou2013counterfactual}.

\newpage
{\Huge \bf Appendices}
\appendix

\setcounter{theorem}{7}

\section{More Theoretical Results}
\label{appendix:framework}
\subsection{General Lower Bound on The Regret}

\begin{restatable}[General lower bound]{theorem}{generalLowerBound}
  \label{mthm:general_lower_bound}
  Suppose for any bandit oracle ${\CO}$, $\exists$ a
  non-decreasing function
  $h(T)$, s.t. $R(T, \CA_{\CO+\CE_\emptyset})\ge h(T)$ for $\forall T$.
  Suppose the offline
  estimator $\CE$ returns unbiased outcomes $\{{y}_j\}_{j=1}^N$ w.r.t. $\{({\bm{x}}_j,{a}_j)\}_{j=1}^N$. Then for any
  contextual-independent algorithm ${\CA_{\CO+\CE}}$, we
  have:
  \begin{align*}
  R(T, {\CA}_{\CO+\CE}) \ge h(T) - \sum\nolimits_{j=1}^N\left( \max_{a\in[K]}\MBE[y|a] - \MBE[y|a={a}_j] \right).
  \end{align*}
 For any contextual algorithm ${\CA}_{\CO_c+\CE}$, we have
  \begin{align*}
   R(T, {\CA}_{\CO_c+\CE}) \ge h(T) {-} \hspace{-0.05in}\sum\nolimits_{j=1}^N \hspace{-0.05in}
   \left( \max_{a\in[K]}\MBE[y|a,{\bm{x}}_j] {-} \MBE[y|a{=}{a}_j{,} {\bm{x}}_j] \right).
  \end{align*}
            \vspace{-0.1in}
\end{restatable}

\noindent Theorem~\ref{mthm:general_lower_bound} shows how we can apply the
regret {\em ``lower
bound''} of online bandit oracles (e.g. \cite{bubeck2013bounded}) to derive a regret lower
bound with logged data . 
When an algorithm's upper bound meets the lower bound, we get a {\em nearly optimal}
online decision algorithm that uses the logged data.
The proof of Theorem~\ref{mthm:general_lower_bound} is in Section~\ref{sec:proof:general_bound}.

\begin{definition}[The value of logged data]
 The online learning oracle $\CO$ has a regret upper bound $g(T)$ after $T$ time
 slots. Suppose the regret of an algorithm $\CA$ that uses logged data is upper
 bounded by $R(T,\CA)$.
 Then, we call $g(T)-R(T,\CA)$ the ``value of logged data'' in time $T$.
\end{definition}
The ``value of logged data'' quantifies the reduction of regret by using the
logged data. The following corollary gives a lower bound on the ``value of
logged data'' for large $T$.
\begin{corollary}
Suppose conditions in Theorem 1 hold.
Suppose the offline evaluator returns $\{\tilde{y}_j\}_{j=1}^N$ w.r.t. $\{(\tilde{\bm{x}}_j{,} \tilde{a}_j)\}_{j=1}^N$ till time $T$.
If an online bandit oracle satisfies the ``no-regret'' property, i.e. $\exists$
a regret upper bound $g(T)$, such that $\lim_{T\rightarrow
    \infty}g(T)/T {=} 0$ (and $g$ is concave), then the difference of regret
  bounds (before and after using offline data) has the following limit for a context-independent algorithm $\CA_{\CO+\CE}$:
  \begin{align*}
  \lim_{T\rightarrow +\infty} g(T){-}R(T,\CA_{\CO+\CE}) {\ge} \sum\limits_{j=1}^N\left( \max_{a\in[K]}\MBE[y|a] {-} \MBE[y|a{=}\tilde{a}_j] \right).
  \end{align*}
  For a contextual algorithm $\CA_{\CO_c+\CE}$, the limit of such difference
  \begin{align*}
    \lim_{T\rightarrow +\infty} g(T){-}  R(T,\CA_{\CO_c+\CE}) {\ge} 
     \sum\limits_{j=1}^N \hspace{-0.04in}
    \left( \max_{a\in[K]}\MBE[y|a,\tilde{\bm{x}}_j] {-} \MBE[y|a{=}\tilde{a}_j, \tilde{\bm{x}}_j] \right).
  \end{align*}
  \label{corollary:reduction_regret}
\end{corollary}

\subsection{Problem independent regret upper bound on $\CA_{\text{LinUCB+LR}}$}
\begin{restatable}[Linear regression+LinUCB, problem-independent]{theorem}{linearProblemIndependent}
Suppose we have $N$ offline data points. With a probability at least $1-\delta$,
the psuedo-regret (here, $V_0=\bm{I}_d$ is a $d\times d$ identity matrix)
\begin{align*}
 {R}(T,\CA_{\text{LinUCB+LR}}&)  \le 
\sqrt{8 (N{+}T) \beta_T(\delta) \log \frac{ \texttt{trace}(V_0){+}(N{+}T)L^2
  }{\texttt{det} (V_0) } }  \\
& -
 \sqrt{8\beta_T(\delta)} \min\{1, ||\bm{x}||_{\min}\} \frac{2}{L^2} \left( \sqrt{1+NL^2} - 1\right).
\end{align*}
Here, $\{\beta_{t}(\delta)\}_{t=1}^T$ is a non-decreasing sequence where
$\beta_t(\delta)\ge 2d(1+2\ln(1/\delta))$. In addition, $L{=}||\bm{x}||_{\max}$ is the
maximum of $l_2$-norm of the context in any time slot. 
\label{mthm:ContextDepend:LinPI}
\end{restatable}

\noindent
The regret upper bound of Theorem \ref{mthm:ContextDepend:LinPI} consists of two terms. 
The first term that is from the online bandit oracle is $O(\sqrt{(N{+}T)\log(N{+}T)})$. The
second term is the reduction of regret by matching logged data which is $-\Omega(\sqrt{N\log(N+T)})$.
Comparing with the regret bound $O(\sqrt{T\log(T)})$ for only using the online feedbacks \cite{abbasi2011improved}, the regret bound
changes from $O(\sqrt{T\log(T)})$ to $O(\sqrt{(N+T)\log(N+T)}) -
\Omega(\sqrt{N\log(N+T)})$. To illustrate the reduction, we observe that
$\sqrt{N+T}-\sqrt{N}=\sqrt{T}\frac{\sqrt{T}}{\sqrt{N+T}+\sqrt{N}}\le \sqrt{T}$,
where ``$\sqrt{N+T}-\sqrt{N}$'' is for our regret bound with logged data, and
``$\sqrt{T}$'' is for the previous bound without logged data. 
\section{More Experiments and Code Explaination}
\label{appendix:experiment}

\subsection{Code and experiment settings}

Note that we provide the code for reproducibility and one can find the detailed experiment settings
in the code. Thus, this section serves as a document of our code.

When we run one experiment, we run the corresponding \texttt{python} scripts in
the \texttt{/experiments} folder. Figure~\ref{fig:calling_graph} illustrates the
Call Graph of one experiment.

\noindent{\bf Code for the {\em $\epsilon$-decreasing multi-action forest.}}
We modify the R package ``grf'' to implement our multi-action forest. In
particular, we implement the \texttt{BanditPrediction.cpp} in
\texttt{grf/core/src} that extends the regression forest (or causal forest) to
allow multiple actions under a leaf node.
In a typical call for the bandit predictor, the following functions are called
in sequence in the file \texttt{r-package/grf/R/causal\_forest.R}.
The order of functions being called is \texttt{predict\_action}$\rightarrow$\texttt{causal\_predict\_action}.
Note that although we still use the name \texttt{causal\_forest} in the names of
our multi-action forest for convenience, our multi-action forest does not call 
the predictor of ``causal forest'' but use our own implementation instead.

\begin{figure}[htb]
  \centering
  \includegraphics[width=0.5\textwidth]{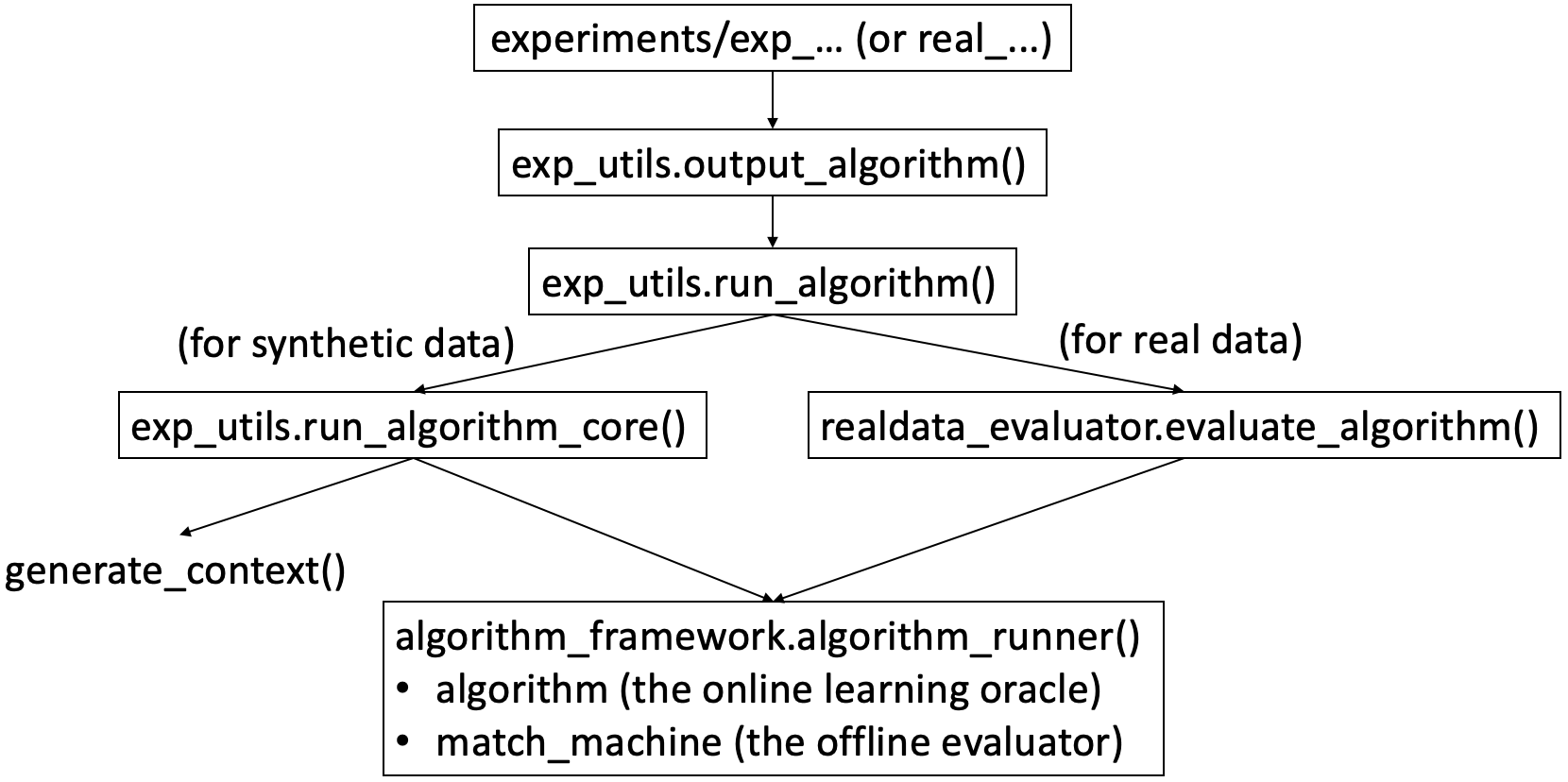}
  \caption{Call Graph of one experiment (in the code)}
  \label{fig:calling_graph}
\end{figure}

\subsubsection{Settings on the simulation}
To do the simulation, we need to simulate an online environment and use it to
generate the logged data.
To have a unified framework for both the context-independent case and the
contextual case, we first have a model to generate the outcome w.r.t. the
context and action, and then get the average outcome w.r.t. the actions by
summing over all contexts. The simulation code is in \texttt{environment.py}.

Note that our method to generate the outcomes for the context-independent case
is not restrictive, because the expected reward for each action can be
arbitrary. Also, the distribution of reward for each action can be arbitrary by setting
different distribution of the contexts.

\subsection{Thompson Sampling}
BanditOracle~\ref{alg:ts_gauss} is the Thompson Sampling algorithm where the
reward of the actions are assumed to be Gaussian random variables.
Figure~\ref{fig:supervised_TS_continuous} in the main paper uses
BanditOracle~\ref{alg:ts_gauss}.
When the reward is of binary values (e.g. in the Yahoo's dataset), one can use
the BanditOracle~\ref{alg:ts_bernoulli} which assume the rewards are Bernoulli
random variables. 
For the Bernoulli Thompson sampling,
the mean of the reward has a Beta-distributed posterior distribtion.

\setcounter{bandit}{3}
\begin{bandit}
    \caption{\bf Thompson Sampling (Gaussian)}\label{alg:ts_gauss}
  {\bf Member variables:} the average outcome $\bar{y}_a$ of each action 
 $a{\in}[K]$, and the number of times $n_a$ that action $a$ was played.\\
  \SetKwFunction{FPlay}{{\bf play}}
  \MFn{\FPlay{$\bm{x}$}}{
    $R_a\gets $ a random variable with normal distribution $\CN(\bar{y}_a,
    \beta^2/(n_a+1))$, for $\forall a\in [K]$.\\
    $r_a\gets $ is a sample from $R_a$.\\
    \Return $\arg\max_{a\in[K]} r_a$
  }
  \SetKwFunction{FUpdate}{{\bf update}}
  \MFn{\FUpdate{$\bm{x}, a, y$}}{
    $\bar{y}_a \gets (n_a\bar{y}_a+y) / (n_a+1)$,~~~~  $n_a \gets n_a + 1$ 
  }
\end{bandit}

\begin{bandit}
    \caption{\bf Thompson Sampling (Bernoulli)}\label{alg:ts_bernoulli}
  {\bf Member variables:} the number of ``1'''s $s_a$ (success) in the feedback for each
  action $a\in[K]$, and the number of ``0'''s $f_a$ (failure) in the feedback for each
  action $a\in[K]$.\\
  \SetKwFunction{FPlay}{{\bf play}}
  \MFn{\FPlay{$\bm{x}$}}{
    $R_a\gets $ a random variable with beta distribution $Beta({s}_a, f_a)$, for $\forall a\in [K]$.\\
    $r_a\gets $ is a sample from $R_a$.\\
    \Return $\arg\max_{a\in[K]} r_a$
  }
  \SetKwFunction{FUpdate}{{\bf update}}
  \MFn{\FUpdate{$\bm{x}, a, y$}}{
    \If{$y=1$}{
      $s_a \gets s_a+1$
    }\Else{
      $f_a \gets f_a+1$
    }
  }
\end{bandit}

\subsection{Propensity Score Matching for More Than Two Actions}
In the main paper, we consider the $\CA_{\text{UCB+PSM}}$ algorithm only for two
actions $K=2$. Here, we keep other settings as default and change the number of actions.
Figure~\ref{fig:psm_2actions}-\ref{fig:psm_5actions} show the cumulative regrets
for the $\CA_{\text{UCB+PSM}}$ algorithm for the number of actions $K=2$ to $K=5$.

Note that the ``only\_online'' algorithm UCB is not affected by the offline
evalutor. Therefore, the ``only\_online'' curve can serve as the baseline.
First, we observe that when $K>2$, the ``only\_offline'' PSM algorithm has a
high regret, which is much higher than the regret for $K=2$.
 Second, when $K>2$, the cumulative regret for the ``offline+online'' algorithm
 $\CA_{\text{UCB+PSM}}$ can be higher than that of the ``only\_online'' UCB
 algorithm. In other words, the propensity score matching offline evaluator does
 not help reduce the regret by using the offline data. This is because it is
 difficult to find matched samples with similar propensity vector and our
 stratification strategy introduces further bias on the estimated reward.
 Moreover, when $K>2$, the regret for the ``only\_offline'' PSM algorithm does
 not necessarily depend on the number of actions $K$. This is
 because PSM algorithm cannot effectively use the offline data and the decision
 depends on some other non-informative factors such as how the values are stratified.
 
\noindent{\bf Lessons learned.} The original original version of propensity score matching algorithm (with
stratification) is not suitable for more than two actions.

\begin{figure*}
  \begin{minipage}{0.245\linewidth}
    \includegraphics[width=\textwidth]{exp3_PSmatch_UCB.png}
    \caption{$\CA_{\text{UCB+PSM}}$, $K=2$}
    \label{fig:psm_2actions}
  \end{minipage}
  \begin{minipage}{0.245\linewidth}
    \includegraphics[width=\textwidth]{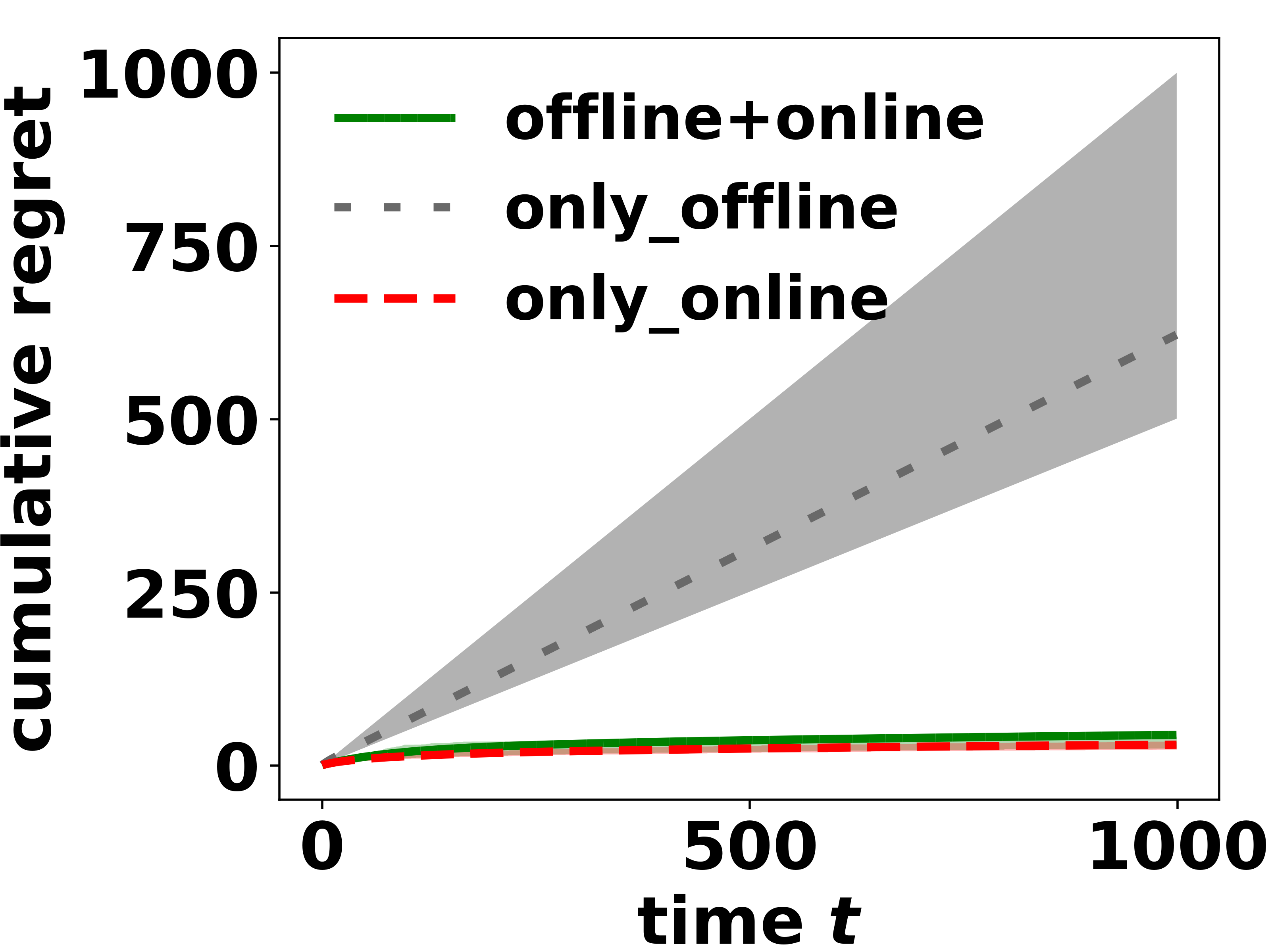}
    \caption{$\CA_{\text{UCB+PSM}}$, $K=3$}
    \label{fig:psm_3actions}
  \end{minipage}
  \begin{minipage}{0.245\linewidth}
    \includegraphics[width=\textwidth]{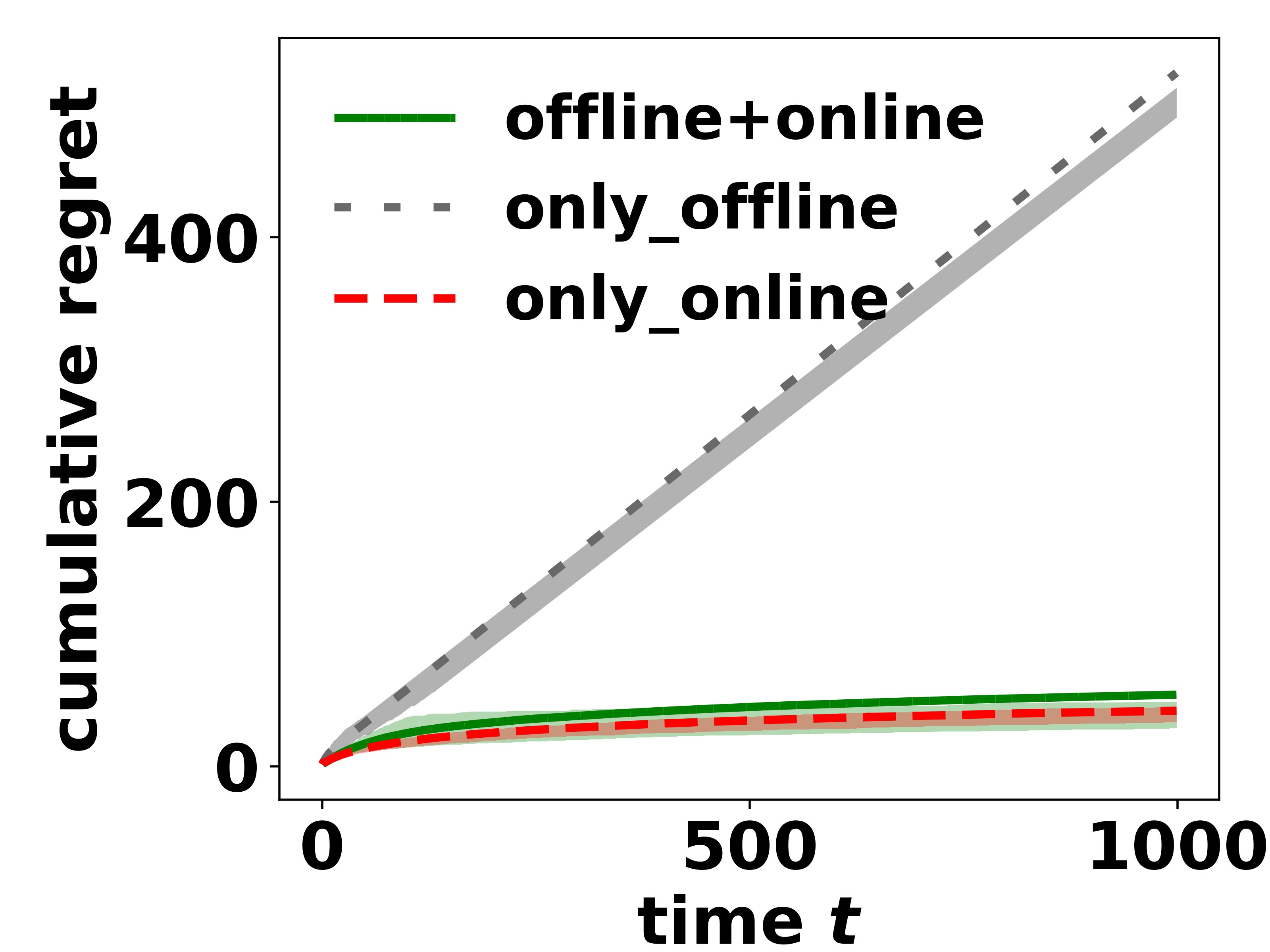}
    \caption{$\CA_{\text{UCB+PSM}}$, $K=4$}
    \label{fig:psm_4actions}
  \end{minipage}
  \begin{minipage}{0.245\linewidth}
    \includegraphics[width=\textwidth]{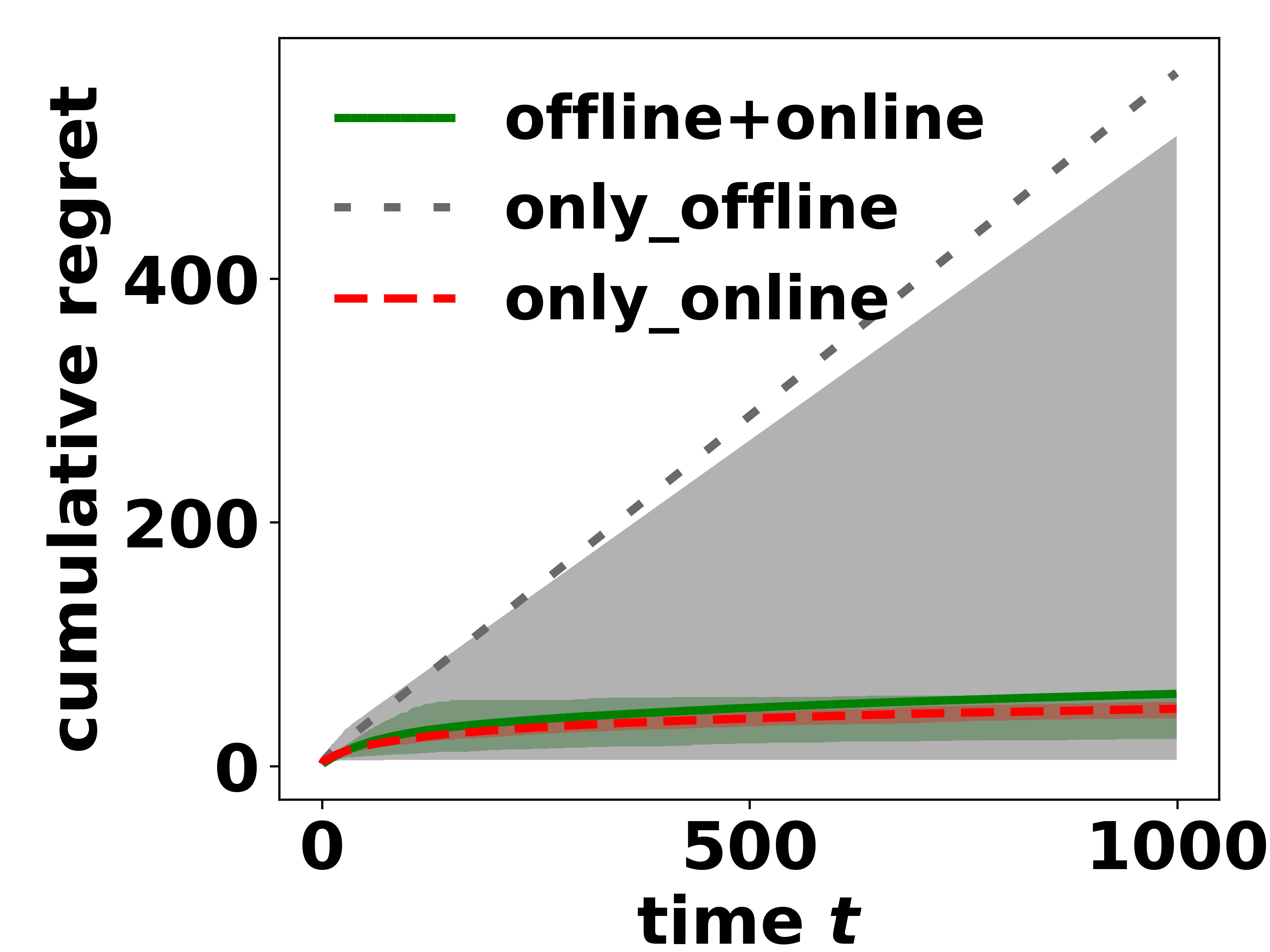}
    \caption{$\CA_{\text{UCB+PSM}}$, $K=5$}
    \label{fig:psm_5actions}
  \end{minipage}
\end{figure*}

\subsection{Experiment on Other Settings of Synthetic Data}
We will extend the default experiment settings in three aspects: (1) the number of
actions, (2) the propensity score function $ps(\bm{x},a)$, and (3) the outcome
function $f(\bm{x},a)$.

\noindent{\bf The number of actions.}
In Figure~\ref{fig:ipsw_3actions}-\ref{fig:ipsw_8actions}, we increase the
number of actions from 3 to 8 for the $\CA_{\text{UCB+IPSW}}$ algorithm. 
First, we observe that for each number of actions, our $\CA_{\text{UCB+IPSW}}$ algorithm
always has a lower regret compared to its two variants. 
Second, we observe that as the number of actions increases, the difference
between the regret of the ``offline+online'' algorithm $\CA_{\text{UCB+IPSW}}$
and the regret of the ``only\_online'' UCB algorithm becomes smaller. This is because when we have
more actions, we need more logged data so that the numbers of logged data are
sufficient for each actions.

\begin{figure*}
  \begin{minipage}{0.245\linewidth}
    \includegraphics[width=\textwidth]{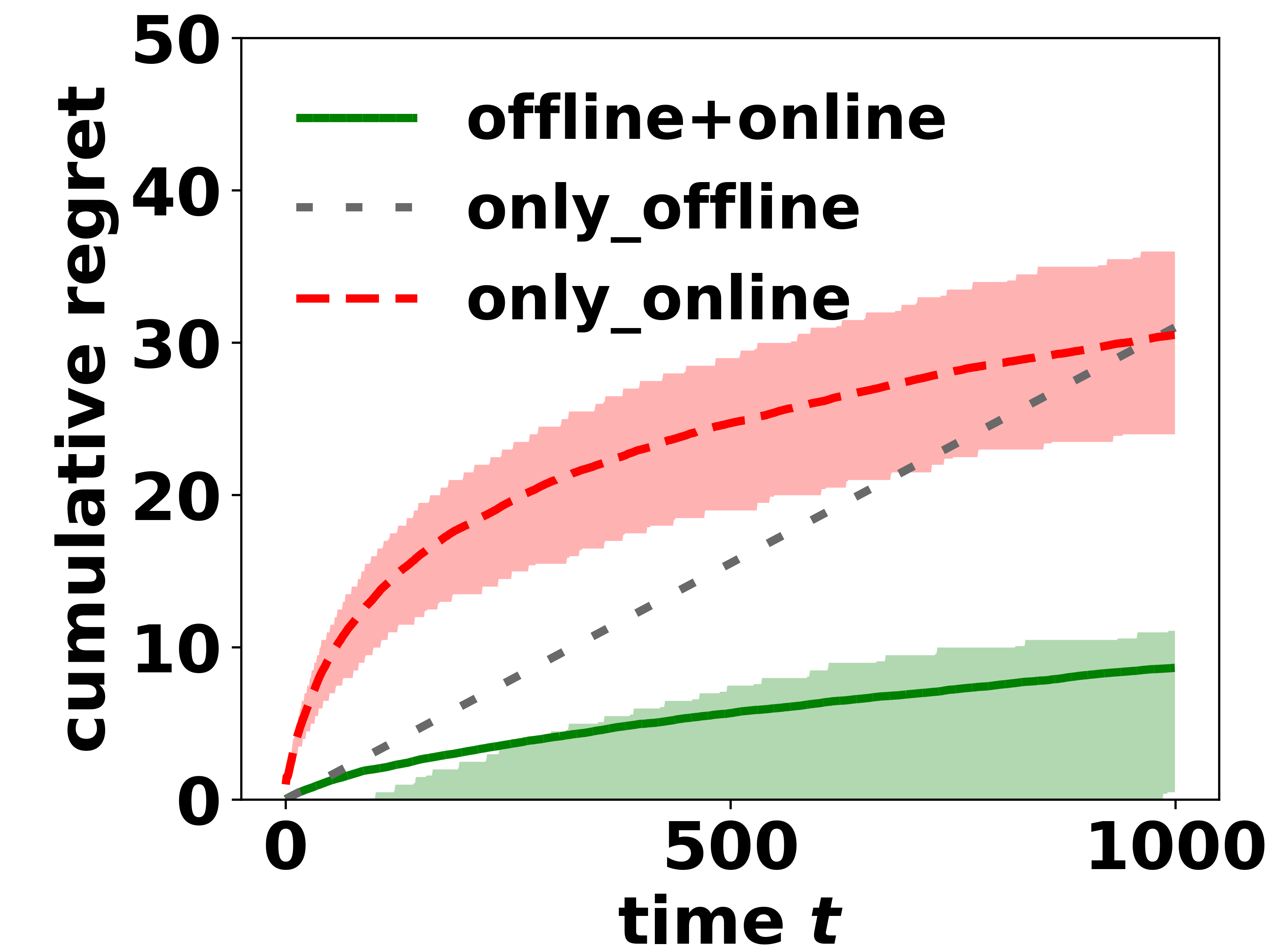}
    \caption{$\CA_{\text{UCB+IPSW}}$, $K=3$}
    \label{fig:ipsw_3actions}
  \end{minipage}
  \begin{minipage}{0.245\linewidth}
    \includegraphics[width=\textwidth]{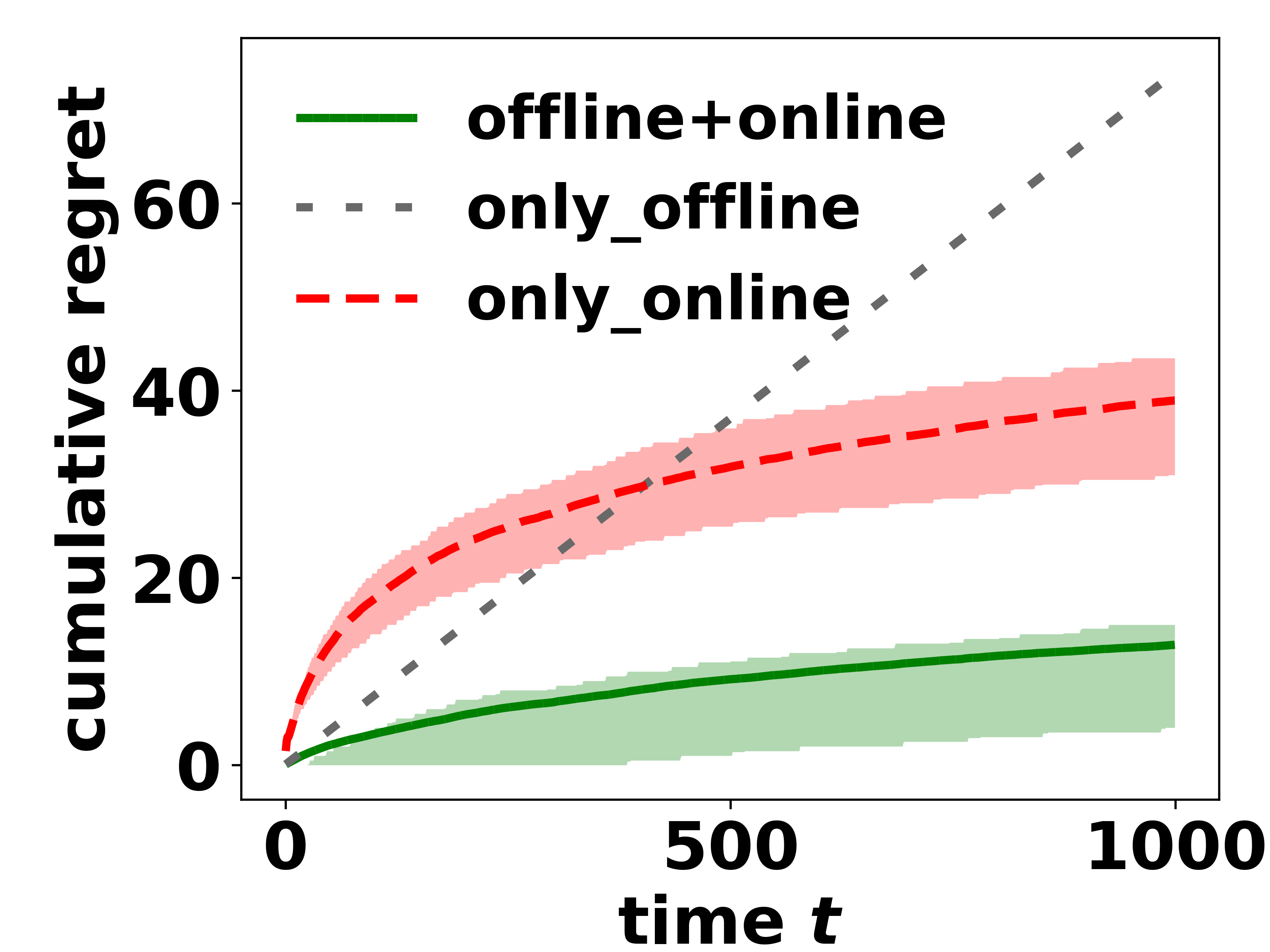}
    \caption{$\CA_{\text{UCB+IPSW}}$, $K=4$}
    \label{fig:ipsw_4actions}
  \end{minipage}
  \begin{minipage}{0.245\linewidth}
    \includegraphics[width=\textwidth]{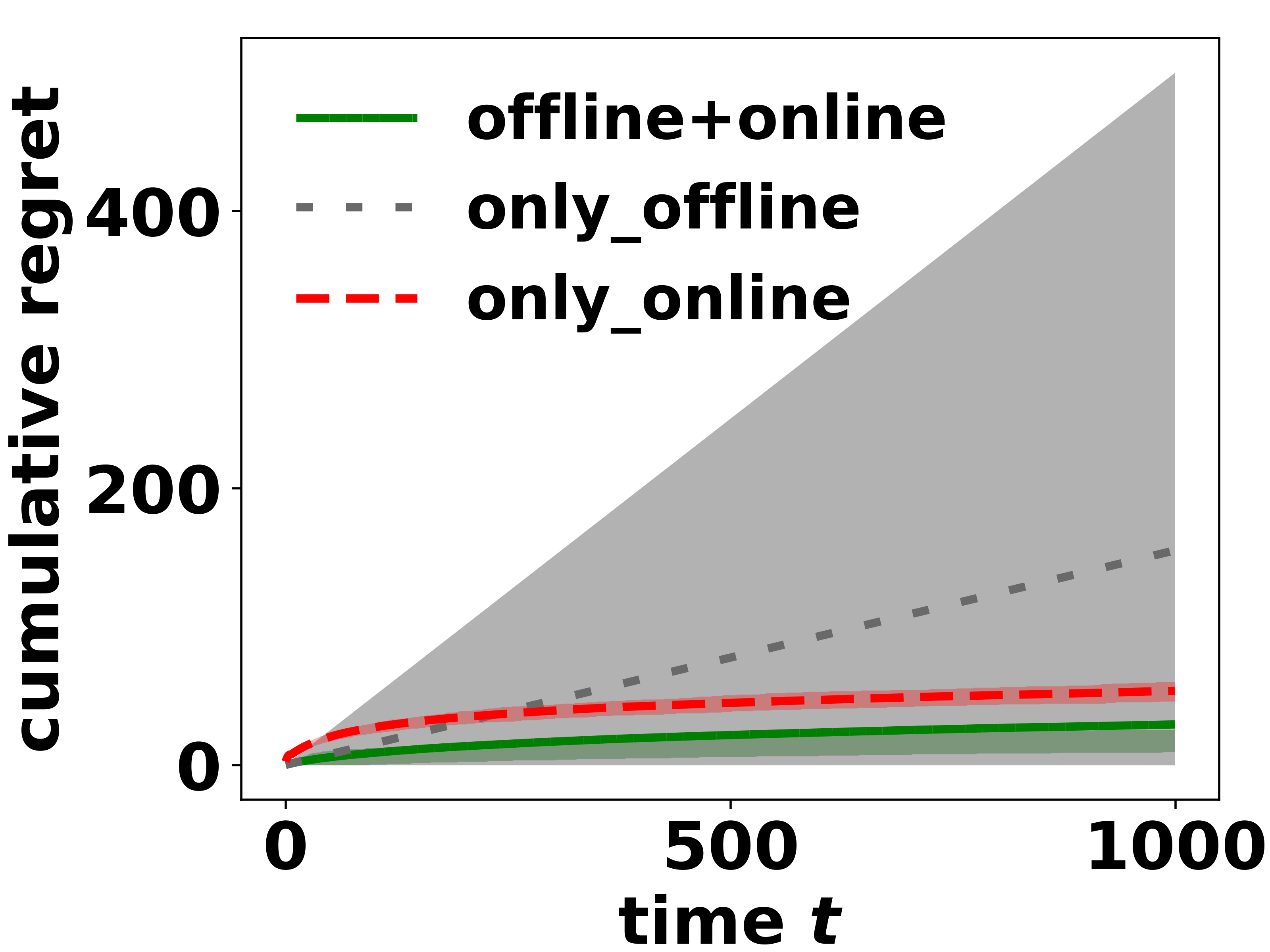}
    \caption{$\CA_{\text{UCB+IPSW}}$, $K=6$}
    \label{fig:ipsw_6actions}
  \end{minipage}
  \begin{minipage}{0.245\linewidth}
    \includegraphics[width=\textwidth]{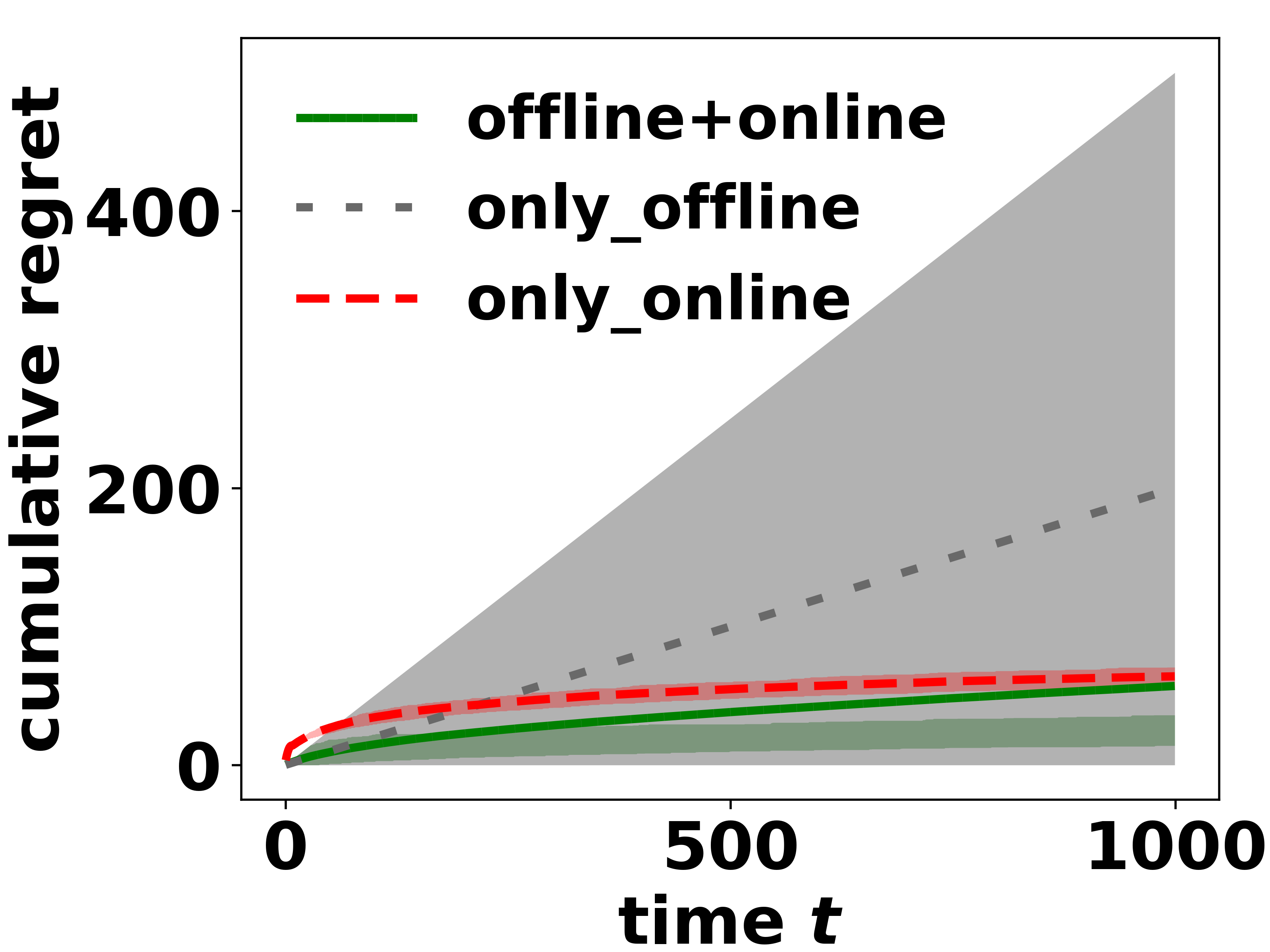}
    \caption{$\CA_{\text{UCB+IPSW}}$, $K=8$}
    \label{fig:ipsw_8actions}
  \end{minipage}
\end{figure*}

\noindent{\bf The propensity score function.}
In the main paper, we set the propensity score function to $ps(\bm{x},a)=\exp(s_a)/(\sum_{a=0}^{K-1} \exp(s_a))$, 
where $s_a=\exp(\rho \bm{x}^T\bm{\theta}_a(\MBE[y|a] - \MBE[y|(a{+}1)\mod K]) )$
and $\rho=-1$.
The parameter $\rho$ controls the correlation between the action and the outcome
given the contexts. Negative $\rho$ indicates the following negative correlation: when $\rho<0$, if an action has a higher expected reward,
then the samples of this action will be selected with a higher probability if
the sample reward is lower.
In the following experiment, we explore more settings where $\rho=0$ or $\rho=1$. Here, $\rho=0$ means
that each action will have the same propensity score, i.e., each action will be
selected with equal probability.

\begin{figure*}
  \begin{minipage}{0.245\linewidth}
    \includegraphics[width=\textwidth]{exp2_IPSW_UCB.png}
    \caption{$\CA_{\text{UCB+IPSW}}$, $\rho=-1$}
    \label{fig:ipsw_correlation-1}
  \end{minipage}
  \begin{minipage}{0.245\linewidth}
    \includegraphics[width=\textwidth]{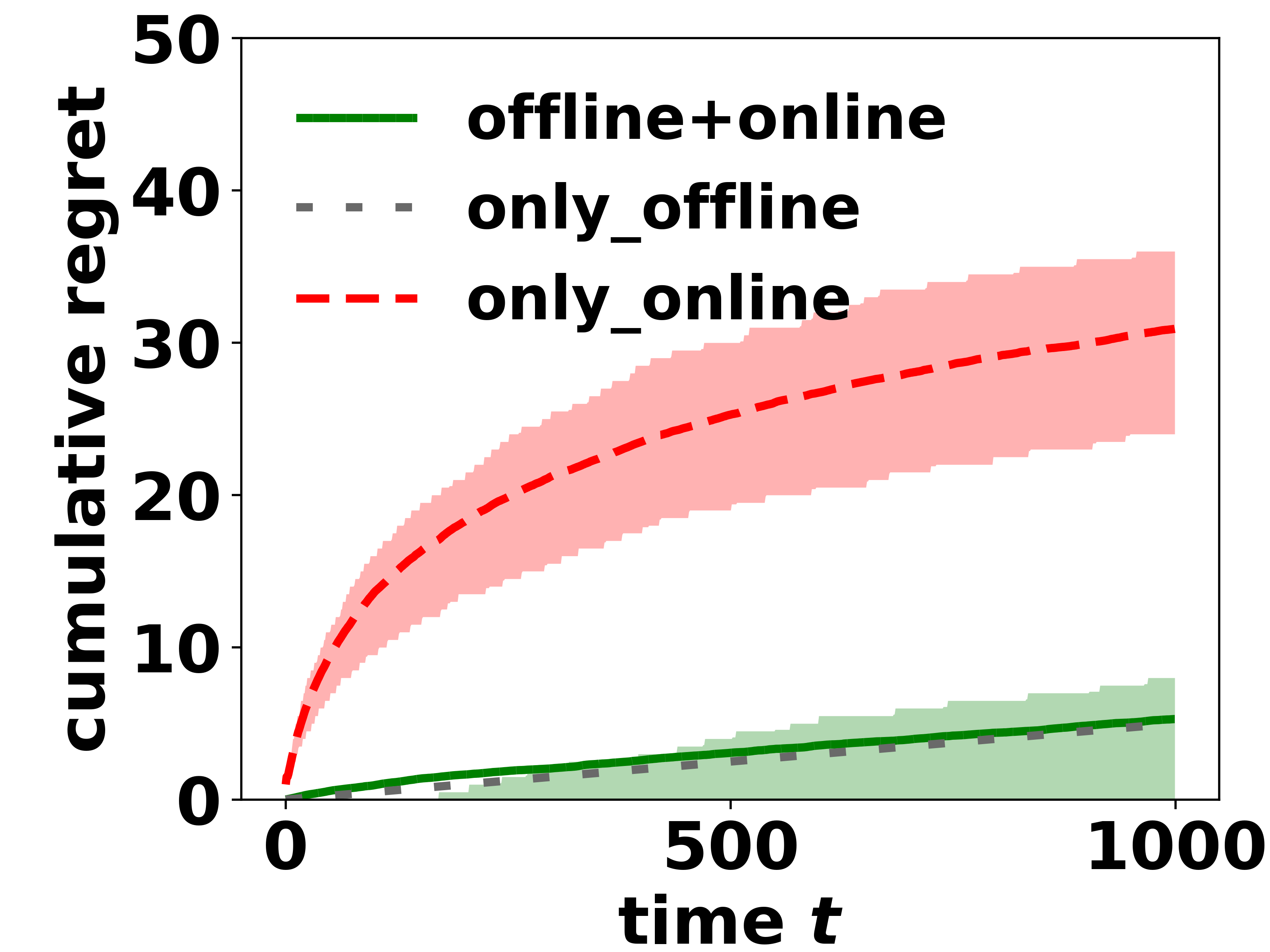}
    \caption{$\CA_{\text{UCB+IPSW}}$, $\rho=0$}
    \label{fig:ipsw_correlation0}
  \end{minipage}
  \begin{minipage}{0.245\linewidth}
    \includegraphics[width=\textwidth]{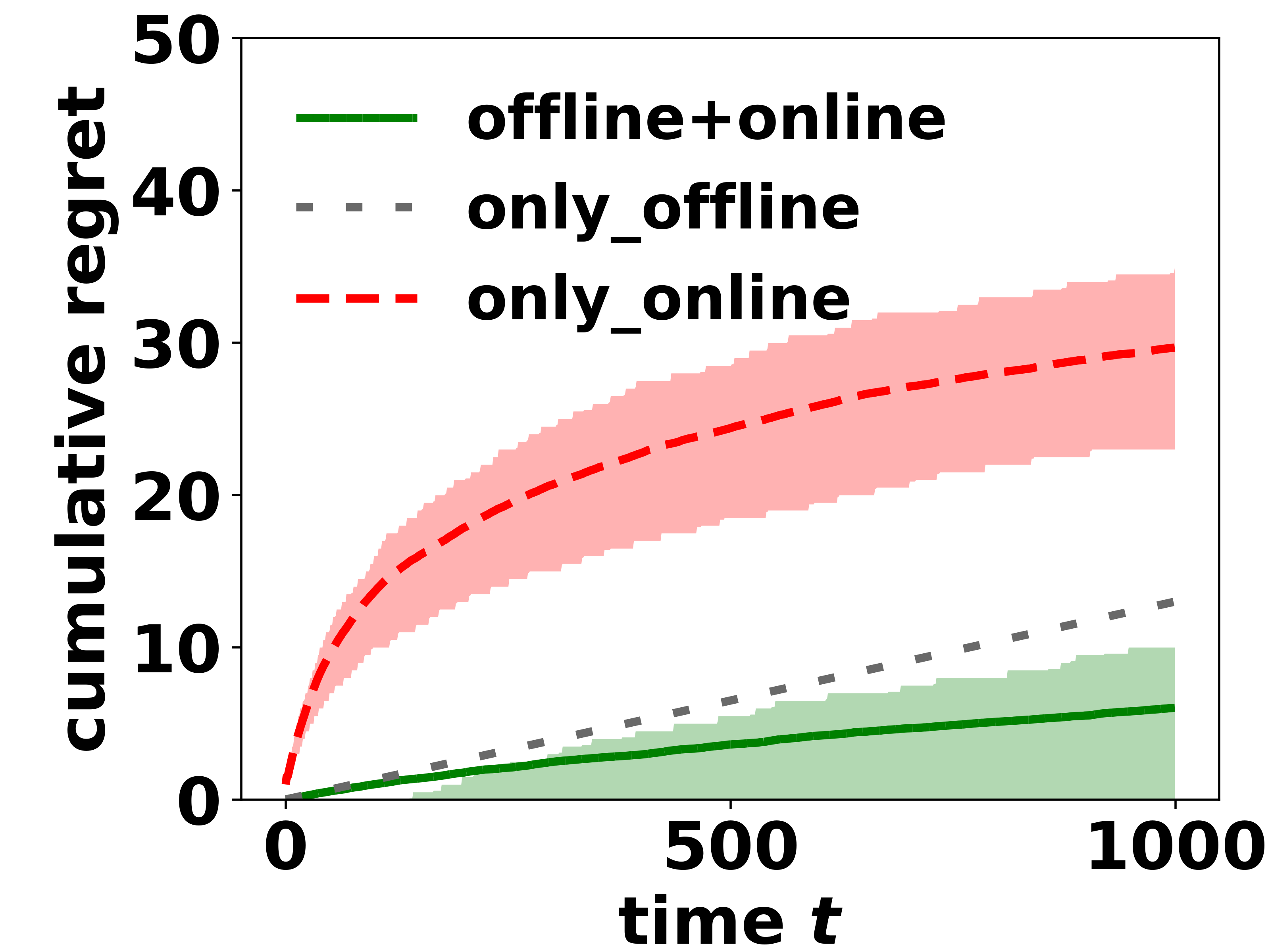}
    \caption{$\CA_{\text{UCB+IPSW}}$, $\rho=1$}
    \label{fig:ipsw_correlation1}
  \end{minipage}
\end{figure*}

\noindent{\bf The outcome function.}
In our main paper, the default outcome function is the linear function
$y=f(\bm{x},a)=\bm{x}^T\bm{\theta}_a+b_a$. Here, we consider two variants of the
outcome function. The first is the sigmoid function
$y=1/(1+\exp(-\bm{x}^T\bm{\theta}_a+b_a))$. The second is the binary outcome
$y\in\{0,1\}$ where $y=1$ with probability
$1/(1+\exp(-\bm{x}^T\bm{\theta}_a+b_a))$. We point out that the expected reward
for the ``sigmoid'' and the ``binary'' settings are the same.

Figure~\ref{fig:ipsw_linear}, \ref{fig:ipsw_sigmoid} and \ref{fig:ipsw_binary}
are the results for the linear outcome function, the sigmoid outcome function
and the binary outcome function respectively. We observe that the outcome function
significantly affects the performance of the algorithms. For sigmoid outcome
function, our ``offline+online'' algorithm and the ``only\_offline'' algorithm
almost have zero regret. It means that the 100 logged samples provide enough
information for the decision maker to distinguish the action with the highest
expected reward.
When the outcome is binary, our ``offline+online'' algorithm has a lower regret
than the ``only\_online'' UCB algorithm. Although the sigmoid function and the
binary outcome function correspond to the same expected reward for each action,
the regret is higher for the binary outcome because the binary outcome function
implies a larger variance of the outcome.

\begin{figure*}
  \begin{minipage}{0.245\linewidth}
    \includegraphics[width=\textwidth]{exp2_IPSW_UCB.png}
    \captionsetup{width=0.75\textwidth}
    \caption{$\CA_{\text{UCB+IPSW}}$, linear function}
    \label{fig:ipsw_linear}
  \end{minipage}
  \begin{minipage}{0.245\linewidth}
    \includegraphics[width=\textwidth]{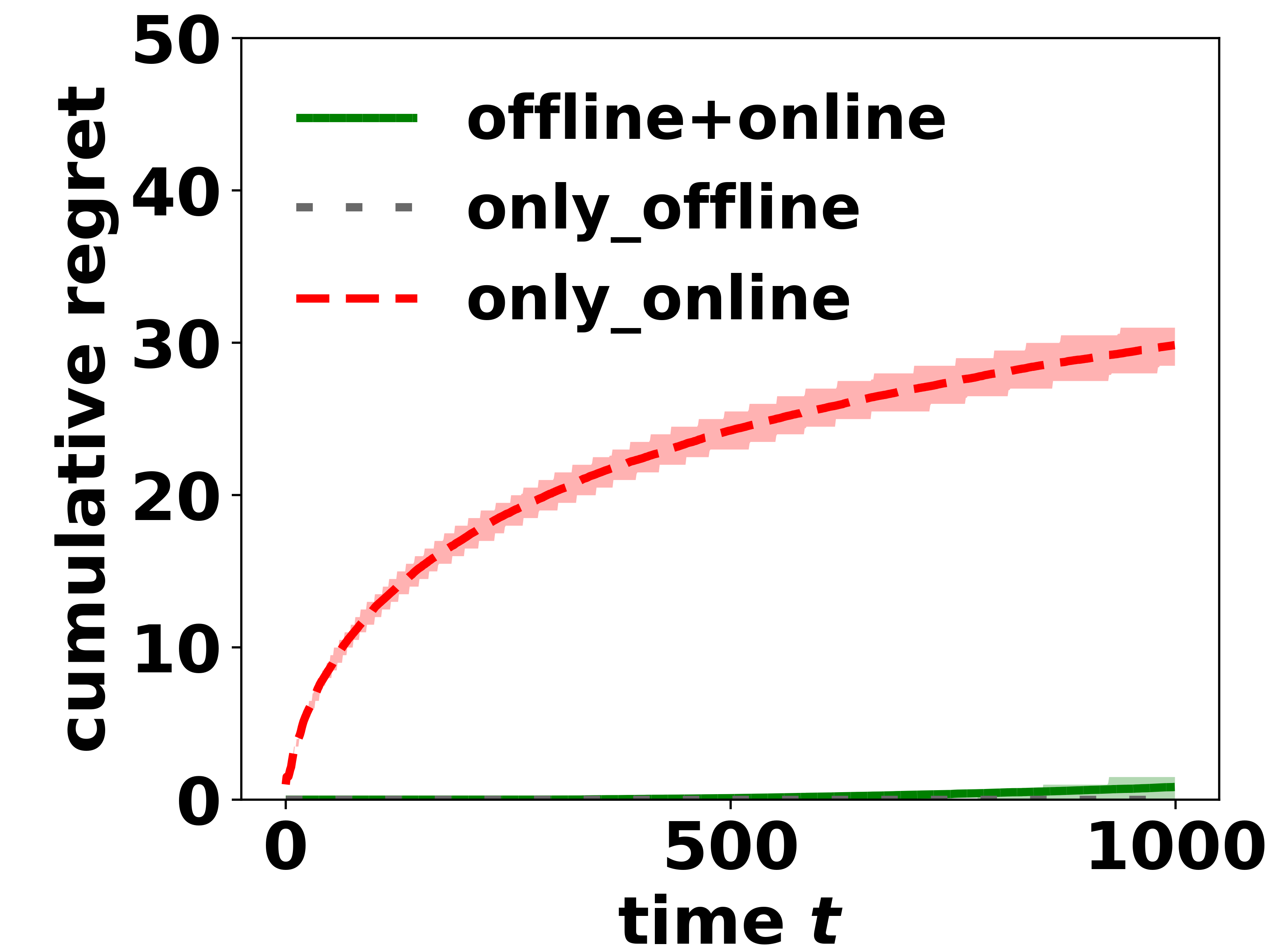}
    \captionsetup{width=0.75\textwidth}
    \caption{$\CA_{\text{UCB+IPSW}}$, sigmoid function}
    \label{fig:ipsw_sigmoid}
  \end{minipage}
  \begin{minipage}{0.245\linewidth}
    \includegraphics[width=\textwidth]{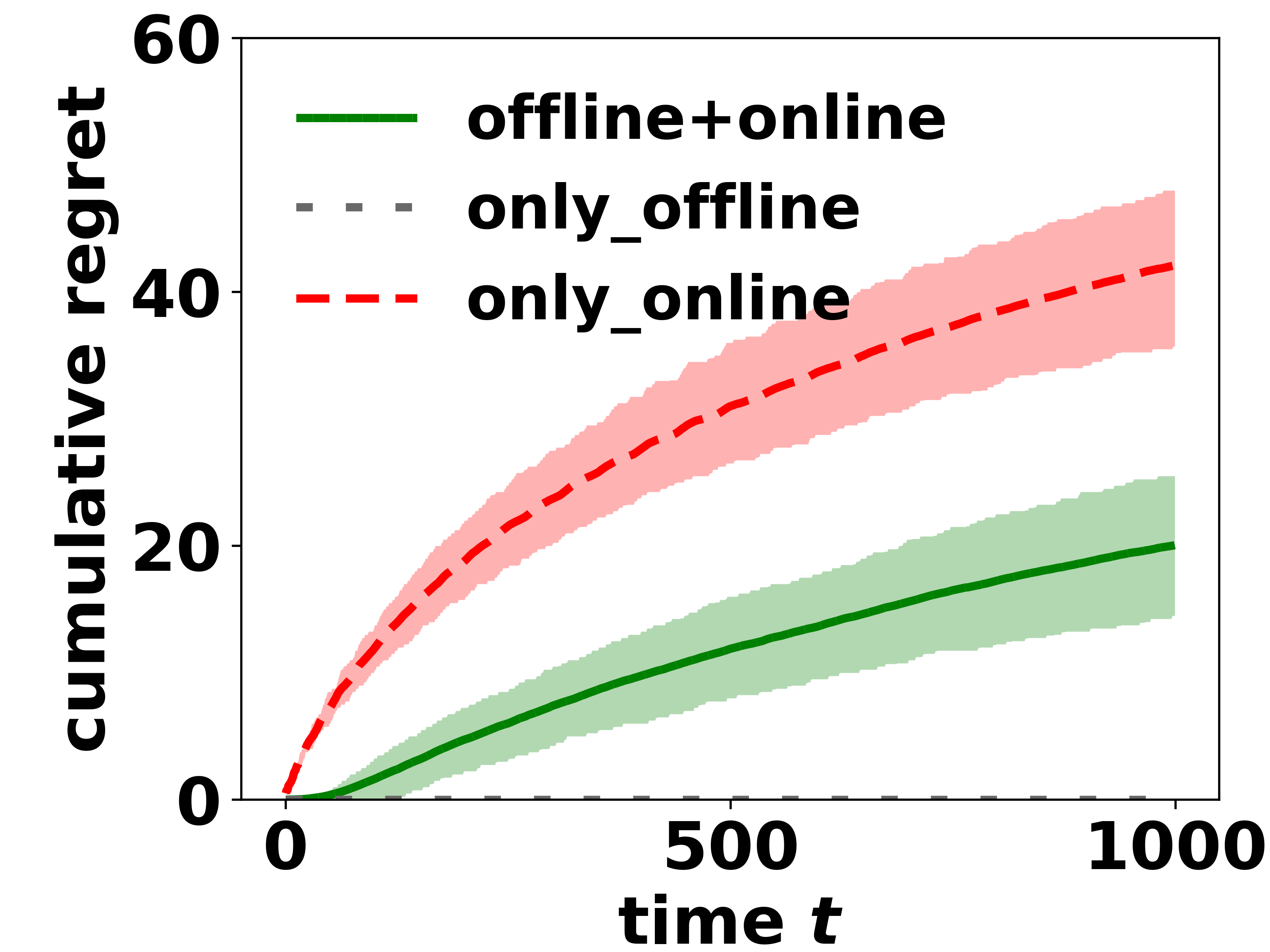}
    \captionsetup{width=0.75\textwidth}
    \caption{$\CA_{\text{UCB+IPSW}}$, binary outcome}
    \label{fig:ipsw_binary}
  \end{minipage}
\end{figure*}

\subsection{Linear vs. Forest Model on Yahoo's Data}
In Figure~\ref{fig:yahoo_LinUCB} and Figure~\ref{fig:yahoo_forest}, we compare
the cumulative reward for $\CA_{\text{LinUCB+LR}}$ and $\CA_{\text{Fst+MoF}}$ on
Yahoo's data. We see that the two algorithms result in similar cumulative regrets.
Recall that the user features in the Yahoo's data were learned via a linear
model. In other words, our non-paramtric forest model achieves comparable
performance with the LinUCB even on the ``linear'' dataset.

 \begin{figure*}
  \begin{minipage}{0.67\linewidth}
    \vspace{-0.02in}
    \captionsetup{width=0.9\textwidth}
   \begin{centering}
   \hspace{0.3in}\includegraphics[width=0.85\textwidth]{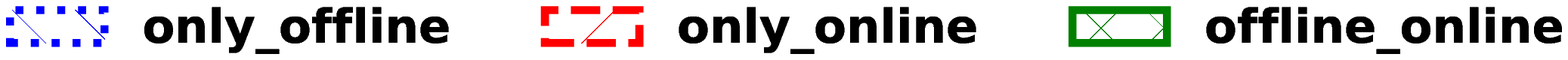}
   \end{centering}
   \includegraphics[height=0.32\textwidth]{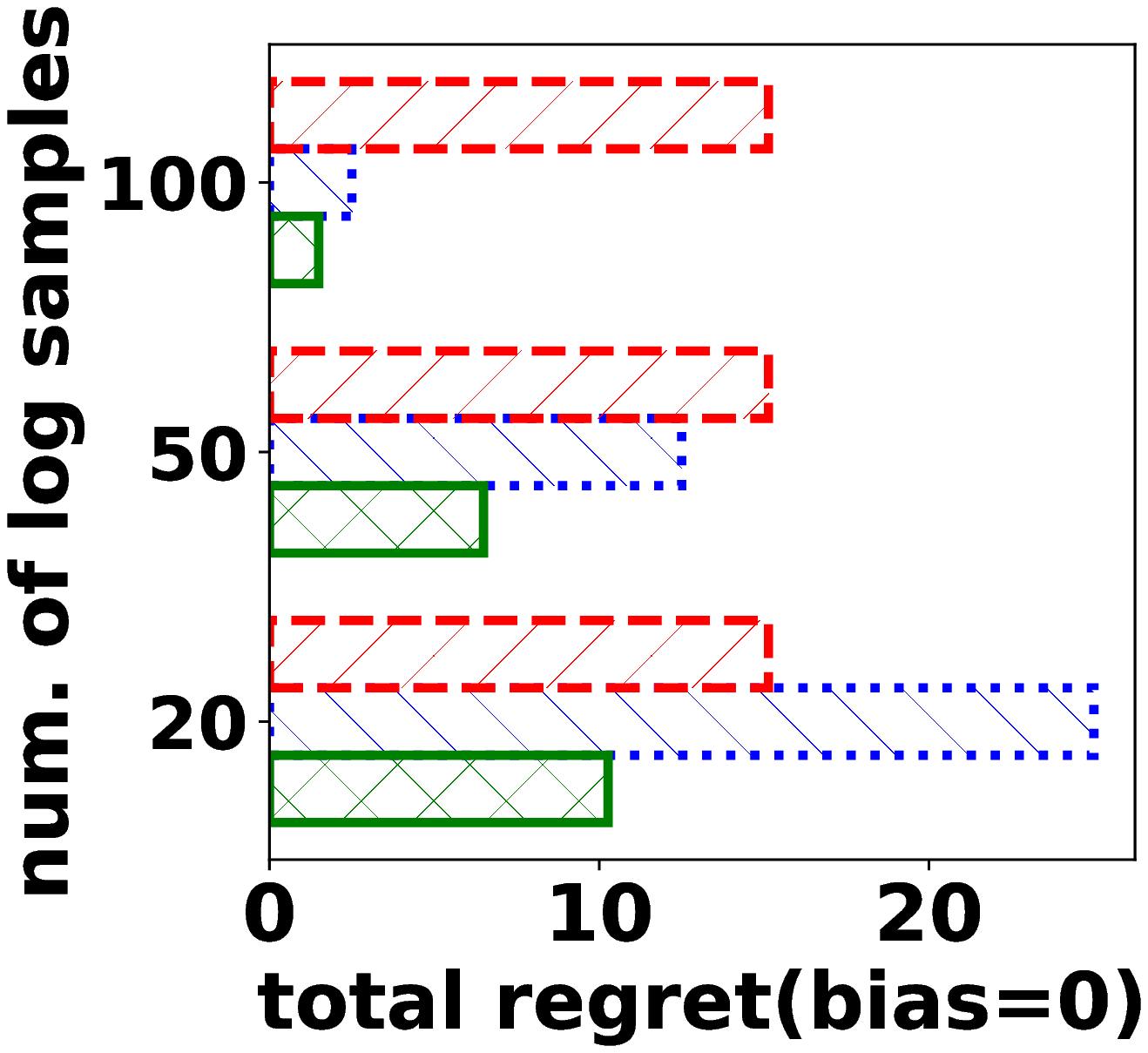}
          \includegraphics[height=0.32\textwidth]{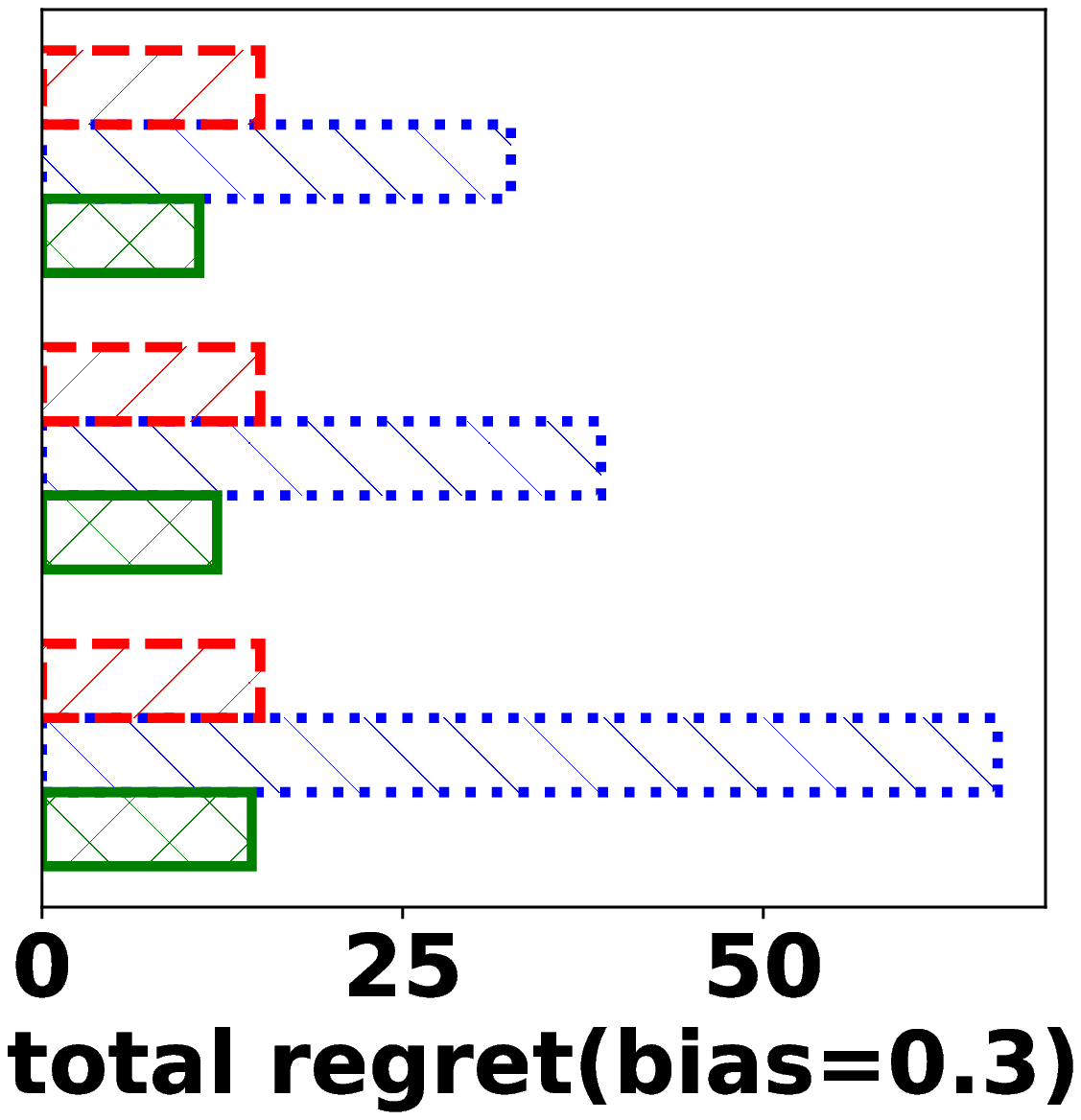}
         \includegraphics[height=0.32\textwidth]{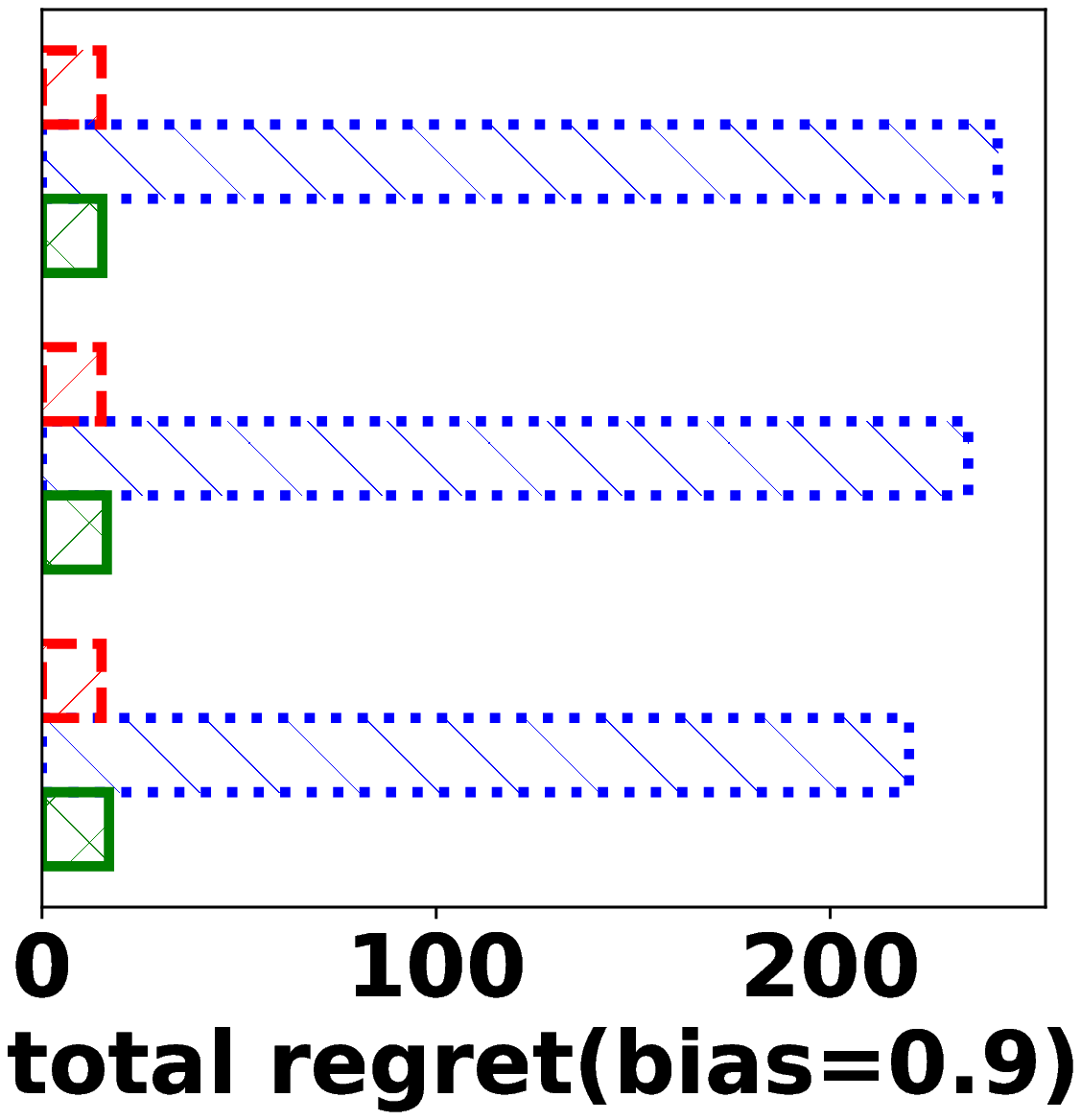} 
      \caption{The impact of the bias and the number of logged samples on the total
     regrets for $\CA_{\text{UCB+IPSW}}$ ($T$=500)}
   \label{fig:impact_quality_quantity}
  \end{minipage}
  \begin{minipage}{0.32\linewidth}
    \includegraphics[width=\textwidth]{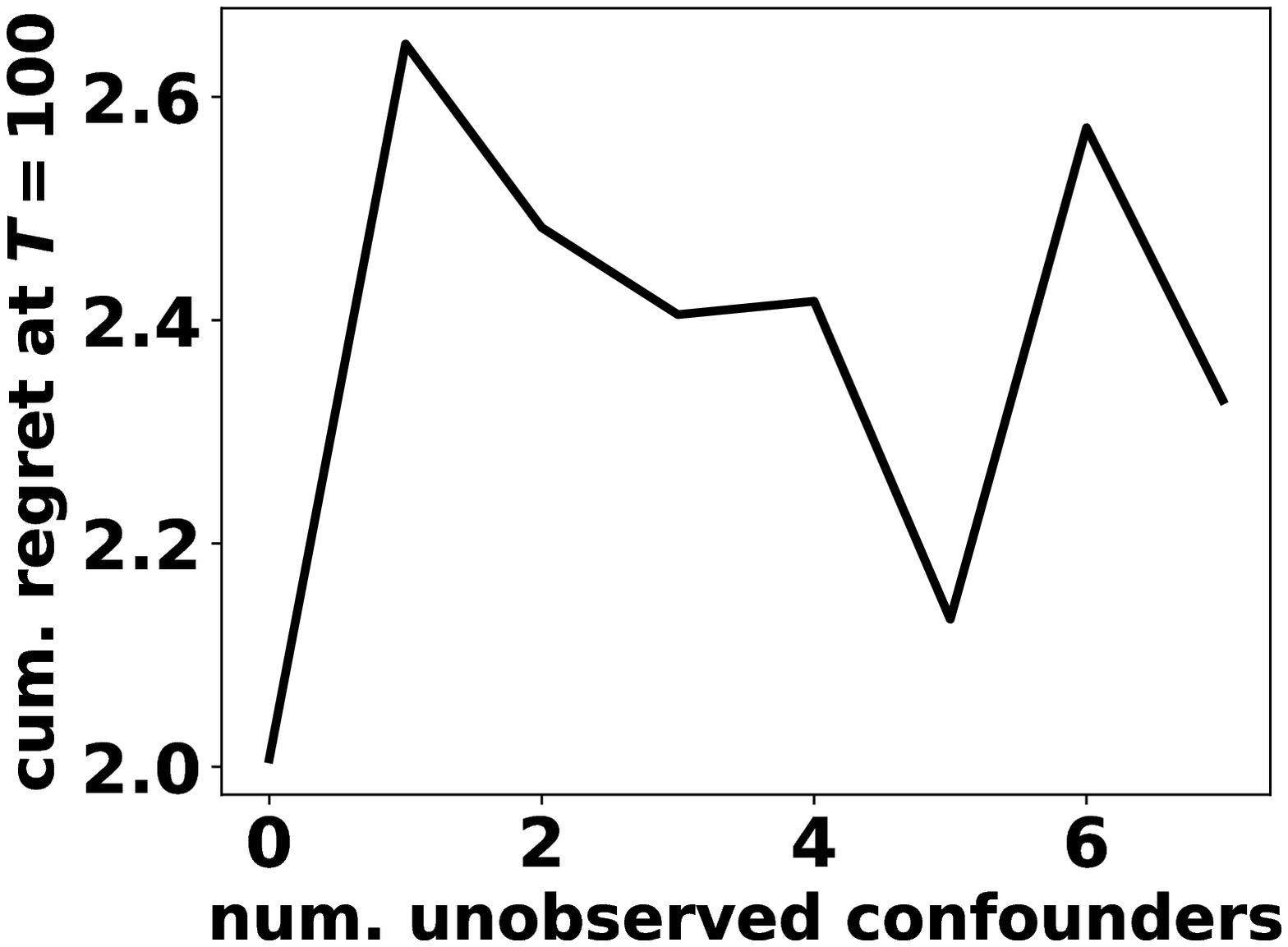}
    \caption{The impact of unobserved confounders for $\CA_{\text{UCB+IPSW}}$}
    \label{fig:unobserved_confounders}
  \end{minipage}
\end{figure*}

\begin{figure*}
  \centering
\begin{minipage}{0.245\linewidth}
\includegraphics[width=\textwidth]{real5_yahoo_linUCB50.eps}
\captionsetup{width=0.9\textwidth}
\vspace{-0.1in}
\caption{$\CA_{\text{LinUCB+LR}}$ on Yahoo's data}
\label{fig:yahoo_LinUCB}
\end{minipage}
\begin{minipage}{0.245\linewidth}
\includegraphics[width=\textwidth]{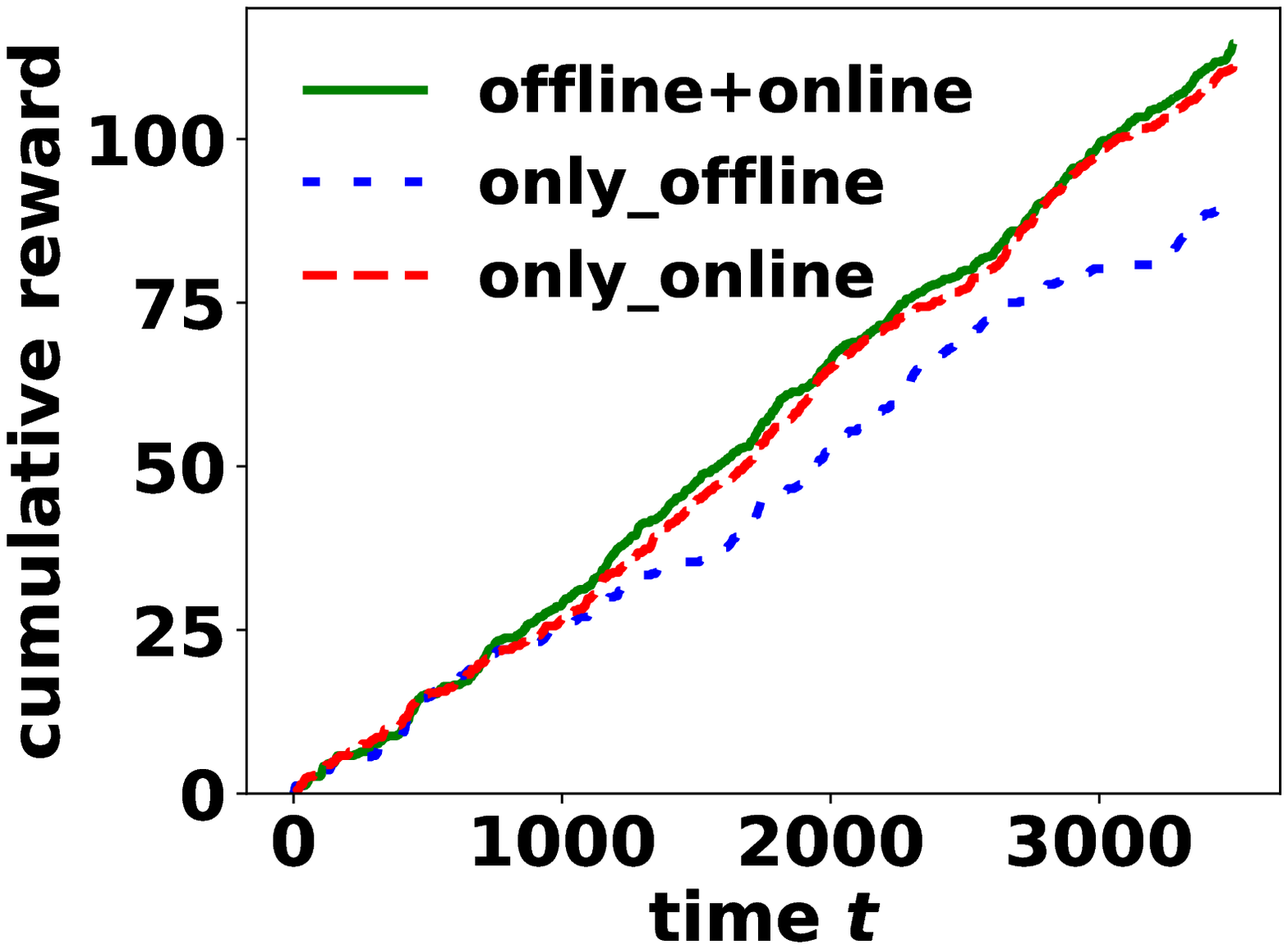}
\captionsetup{width=0.8\textwidth}
\vspace{-0.1in}
\caption{$\CA_{\text{Fst+MoF}}$ on Yahoo's data}
\label{fig:yahoo_forest}
\end{minipage}
\begin{minipage}{0.245\linewidth}
  \includegraphics[width=\textwidth]{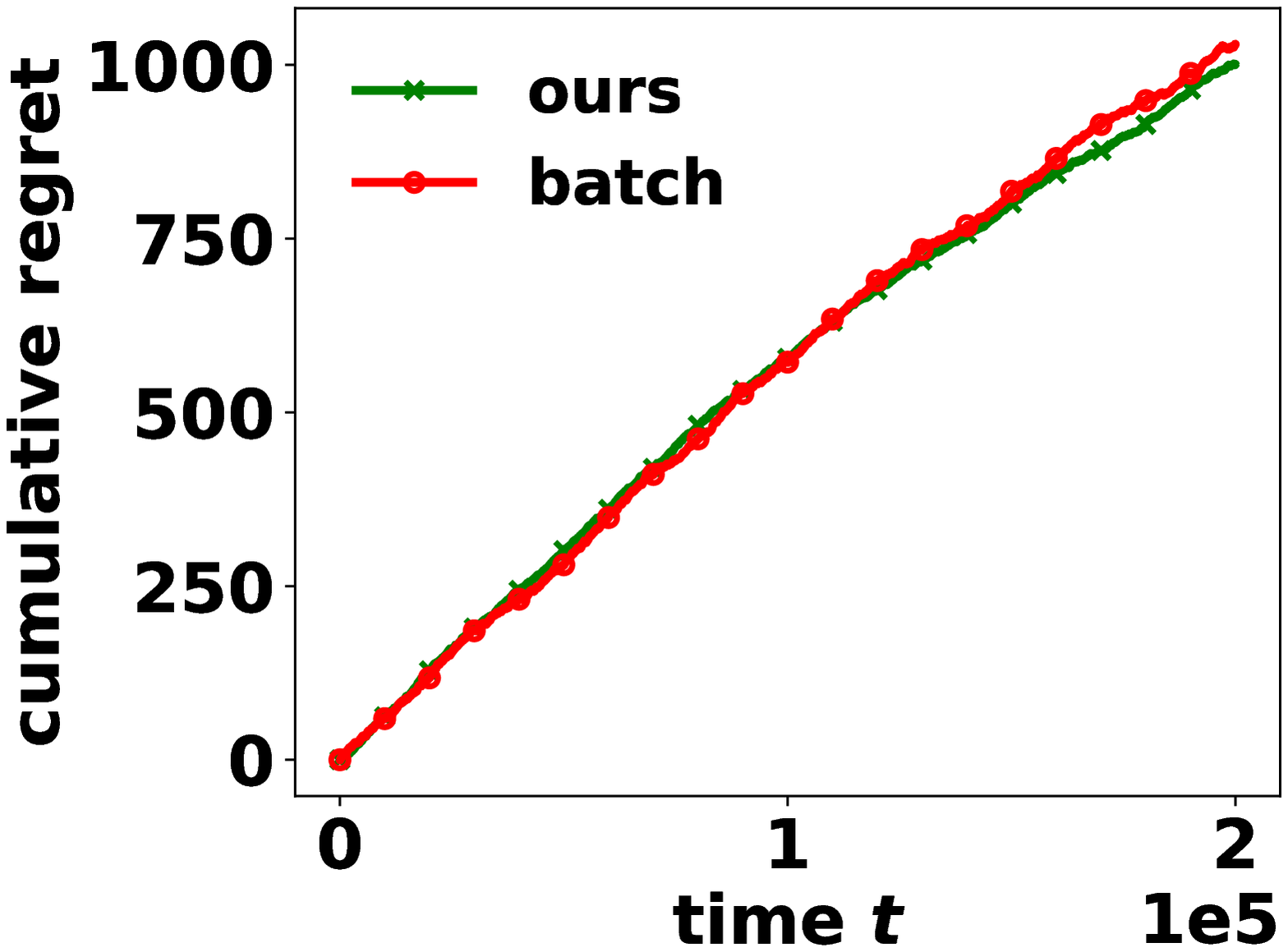}
  \caption{Batch mode vs. our method on Yahoo's data}
  \label{fig:yahoo_batch_mode}
\end{minipage}
\end{figure*}

\subsection{Comparison to Batch Method on Real Data}
The batch version of our algorithmic framework is outlined as Algorithm~\ref{alg:batch}. There are several
differences between the batch variant and our original algorithmic framework in
Algorithm~\ref{alg:framework}. First, in the online phase (Line 13-16) of the batch variant, we do
not use the offline data. Second, in Line~7 of Algorithm~\ref{alg:batch}, the
action $a$ is not generated by the online learning oracle, but is a fixed value
inside the for-loop.
Because not all the actions are generated by the bandit oracle, we cannot
directly use the theoretical results of existing bandit algorithms.

In Figure~\ref{fig:yahoo_batch_mode}, we show that the cumulative regrets for
the batch method and our method are almost indisdinguishable on Yahoo's dataset. This further
validate our observation in the main paper on the synthetic data.

\setcounter{algocf}{1}
\begin{algorithm}[htb]
  \caption{\bf Algorithmic Framework - Batch Variant}\label{alg:batch}
    Initialize the \textit{OfflineEvaluator} with logged data $\CL$\\
  Initialize the \textit{BanditOracle}\\
    \mycomment{The offline phase}\\
  \For{$a\in[K]$}{
    \While{True}{
     $\bm{x} {\gets} \textit{context\_generator()}$ \mycomment{from CDF $F_{\bm{X}}(\cdot)$}\\
                  $y \gets \textit{OfflineEvaluator.}\textbf{get\_outcome}(\bm{x}, a)$\\
   \If{$y \ne$ {\em NULL} }{
        \textit{BanditOracle.}\textbf{update}$(\bm{x}, a, y)$
}\Else(\mycomment{offline evaluator cannot synthesize a feedback}){
        \textbf{break} 
      }
    }
  }
  \mycomment{The online phase}\\
  \For{$t=1$ to $T$}{
       $a_t \gets \textit{BanditOracle.}\textbf{play}(\bm{x}_t)$
    \mycomment{online play} \\
    $y_t\gets$ the outcome from the online environment \\
    \textit{BanditOracle.}\textbf{update}$(\bm{x}_t, a_t, y_t)$
  } 
\end{algorithm}

\subsection{Experiments on Unobserved Confounders}
We first directly analyze the imapct of unobserved confounders on the regret.
Then, we notice that the unobserved confounders create bias in the estimated
reward which relates to the ``quality of the logged data''. 
Therefore, in the second part, we discuss the impact of the quantity and quality of logged data.

\noindent{\bf The imapct of unobserved confounders.}
In Figure~\ref{fig:unobserved_confounders}, we randomly choose a number of
confounders and hide them as unobserved. We see that the cumulative regret
becomes the lowest when there are no unobserved confounders. When there exists
unobserved confounders, the regret do not have a clear relationship with the
number of unobserved confounders.
This is because when there is some missing information, we do not know whether
each part of the missing information has positive or negative impacts on the
cumulative regrets. 

\noindent{\bf Impact of the quantity and quality of logged data.}
Here, we explore the situations where the offline evaluator may return biased samples.
In the ideal case, in terms of quantity we have a sufficiently large number of
data for each action, and in terms of quality the
data records all the confounding factors. 
In reality, these conditions may not hold. 

In Figure~\ref{fig:impact_quality_quantity}, we investigate the impacts of
both the quantity and quality of data, where we focus on the
context-indepedent algorithm $\CA_{\text{UCB+IPSW}}$.
Recall that the expected rewards for the two actions are $0$ and $0.5$.
Now, in the logged data we add a bias to the first action, and its expected reward becomes ``$0+$bias''. 
We observe that when the bias is 0 or 0.3,
the ``offline+online'' variant $\CA_{\text{UCB+IPSW}}$ has the lowest regret. This is because with small bias, the logged data is
still informative to select the better action. However, when the bias is
as large as 0.9, the ``only\_online'' variant (i.e. UCB) achieves the lowest
regret, because the offline estimations are misleading.
The impact of the number of logged samples depends on the bias. 
In the case of zero bias (the left figure), if we have a large number of logged samples (e.g.
100),
then our $\CA_{\text{UCB+IPSW}}$ algorithm and the ``only\_offline'' IPSW algorithm have low regrets because they use logged data.
 But when logged data has high bias (the right figure), more logged samples
 result in a higher regret for algorithms $\CA_{\text{UCB+IPSW}}$ and IPSW that
 use the logged data.

\section{Proofs}
\label{appendix:proof}

In our main paper, we have Theorem~\ref{thm:general_upper_bound},
\ref{thm:ps_matching}, \ref{thm:online_causal_forest}. We give proofs of these
three theorems in Section~\ref{sec:proof:general_bound},
\ref{sec:proof:context_independent}, \ref{sec:proof:online_forest}.

\subsection{General Regret Upper and Lower Bounds (Theorem~\ref{thm:general_upper_bound} and Theorem~\ref{mthm:general_lower_bound})}
\label{sec:proof:general_bound}

Now, we prove the general upper bound of our framework.\footnote{We have a 
  technical condition that regret bounds of the online bandit oracle $g(T)$
  only depends on expected rewards of each arm (e.g. the regret bound of
  UCB~\cite{auer2002finite} only depends on the expected reward).}

\begin{proof}[\bf Proof of Theorem~\ref{thm:general_upper_bound}]
The proof follows the idea described in Section 3.2.
Online learning oracle is called for $N+T$ times, including $N$ times with
synthetic feedbacks and $T$ times with real feedbacks. Denote the total
pseudo-regret in these $N+T$ time slots as $R(\CA_{\CO+\CE_\emptyset},N+T)$. Because the condition
(2) ensures that our offline evaluator returns unbiased i.i.d. samples in different time slots, the online
bandit oracle cannot distinguish these offline samples from online samples. (This
is because the regret bound only depends on the expected rewards of each arm and
the offline evaluator $\CE$ is unbiased.) Then according to
the regret bound of the online learning oracle, we have
\begin{align}
  \label{eq:total_bound}
  R(\CA_{\CO+\CE_\emptyset},N+T) \le g(N+T).
\end{align}
Moreover, we could decompose the total {\em expected} regret of the online learning oracle as
\begin{align}
  \label{eq:decompose}
R(\CA_{\CO+\CE_\emptyset},N{+}T) {=} \sum_{j=1}^N (\max_{a\in[K]}\MBE[y|a]{-}\MBE[y|a{=}\tilde{a}_j]) {+} R(\CA_{\CO+\CE}, T)
\end{align}
On the right hand side of (\ref{eq:decompose}), the first term $\sum_{j=1}^N (\max_{a\in[K]}\MBE[y|a]-\MBE[y|a=\tilde{a}_j])$ is the cummulative regret of the bandit oracle in
the offline phase, and the second term $R(\CA_{\CO+\CE},T)$ is the cumulative regret in
the online phase.
Combining (\ref{eq:total_bound}) and (\ref{eq:decompose}), we get 
\begin{align*}
 R(\CA_{\CO+\CE},T) \le   g(N+T) - \sum_{j=1}^N (\MBE[y|a^\ast]-\MBE[y|\tilde{a}_j]),
\end{align*}
which concludes our proof for the context-independent case.
For the contextual case, the proof is similar and we only need to replace
$\MBE[y|a]$ with $\MBE[y|a,\bm{x}]$.
\end{proof}

\begin{proof}[\bf Proof of Corollary~\ref{corollary:reduction_regret}]
Based on Theorem 1, we only need to show $\lim_{T\rightarrow +\infty} g(N+T)-g(T) = 0$.
Before we start our proof, we want to point out that regret bounds of many
bandit algorithms have ``no-regret'' property. For example, the regret bound
$g(T)$ for UCB is proportional to $\log(T)$, the regret bound $g(T)$
for EXP3 is proportional to $\sqrt{T}$. These functions w.r.t. $T$ are sub-linear
and concave. These functions are concave because as the
oracle receives more online feedbacks, it makes better decisions and thus has
less regret per time slot. For the concave function, $\frac{g(N+T)-g(T)}{N}$ is
decreasing in $T$. We claim that $\lim_{T\rightarrow
  +\infty}\frac{g(N+T)-g(T)}{N}=0$. Otherwise, there will be a $l> 0$, such
that $\frac{g(N+T)-g(T)}{N}\ge l$, for $T\ge T_0$ where $T_0$ is a constant. It
means that gradient of $g(T)$ is larger than $l$ when $T$ is large.
Then, $\lim_{T\rightarrow +\infty}{g(T)}/{T}\ge l$ which contradicts to the
``no-regret'' property.

Then $N{\times} \lim\limits_{T\rightarrow +\infty}\frac{g(N{+}T){-}g(T)}{N} {=} N{\times} 0 {=} 0$.
Now we have 
\begin{align*}
 & \lim_{T\rightarrow +\infty}  g(T)-R(T,\CA_{\CO+\CE}) \\
= & \lim_{T\rightarrow +\infty}\left( g(T) {-} g(N{+}T) \right) 
{+} \lim_{T\rightarrow +\infty}\left( g(N{+}T){-}R(T,\CA_{\CO+\CE}) \right)\\
\ge & 0 + \sum\nolimits_{j=1}^N\left( \max_{a\in[K]}\MBE[y|a] - \MBE[y|a=\tilde{a}_j] \right),
\end{align*}
which completes our proof for the context-independent case.
For the contextual case, the proof is similar and we only need to replace
$\MBE[y|a]$ with $\MBE[y|a,\bm{x}]$.
\end{proof}

\generalLowerBound*
\begin{proof}[\bf Proof of Theorem~\ref{mthm:general_lower_bound}]

After decomposing the total regret to the offline phase and online phase,
we have for any bandit oracle ${\CO}$
\begin{align}
 R&(T,\CA_{\CO+\CE})  {=} R(T{+}N,{\CA_{\CO+\CE_\emptyset}}) {-} \sum\limits_{j=1}^N\left( \max_{a\in[K]}\MBE[y|a] {-} \MBE[y|a{=}{a}_j] \right)\nonumber\\
 & \ge h(T{+}N) {-} \sum\limits_{j=1}^N\left( \max_{a\in[K]}\MBE[y|a] {-} \MBE[y|a{=}{a}_j] \right).
 \label{eq:thm2_1}
\end{align}
Next, for a non-decreasing function $h(\cdot)$ we have 
\begin{align}
h(T+N)\ge h(T). 
  \label{eq:thm2_2}
\end{align}

 Combining (\ref{eq:thm2_1}) and (\ref{eq:thm2_2}), we have
\begin{align*}
R(T,\CA_{\CO+\CE})\ge h(T) - \sum\nolimits_{j=1}^N\left( \max_{a\in[K]}\MBE[y|a] - \MBE[y|a={a}_j]\right),
\end{align*}
which concludes our proof for the unbiased estimators.
For the contextual case, the proof is similar and we only need to replace
$\MBE[y|a]$ with $\MBE[y|a,\bm{x}]$.
\end{proof}

\subsection{Regret Bounds for Context-Independent Algorithms
  $\CA_{\text{UCB+EM}}$ and $\CA_{\text{UCB+PSM}}$
  (Theorem~\ref{mthm:exact_matching} and Theorem~\ref{thm:ps_matching})}
\label{sec:proof:context_independent}

\begin{proof}[\bf Proof of Theorem~\ref{mthm:exact_matching}]
{\NIPS
  The proof consists of three steps. The first step is to decompose the regret
  as ``the total regret'' - ``the virtual regret''. In the second step, we give
  a bound to the virtual regret. In the third step, we bound the total regret.
}

  The idea of the proof is similar to the proof of the general upper bounds
  Theorem~\ref{thm:general_upper_bound}. According to Assumption~\ref{assumption:ignorability} (ignorability), the exact-matching offline
  evaluator returns unbiased outcomes. Since all the
decisions are made by the online learning oracle, we can apply the regret bound
of the UCB algorithm, and minus the regrets of {\em virtual} plays
for the samples returned by the exact matching evaluator.

\noindent{\underline{\bf Step 1:}}
As usual, to analyze a UCB-like algorithm, we count the number of times
we draw each arm.
\begin{definition}
  $\lambda_a$ is defined as the expected number of rounds that the $a_{th}$ arm is pulled by
  the online learning oracle.\end{definition}
We say an {\em``offline evaluator returns the $a_{th}$ arm''} if 
$\CI(\bm{x},a)\ne
\emptyset$ in Line 5 of OfflineEvaluator~\ref{alg:exact_match} ($\CE_{\text{EM}}$), and meanwhile, the context-action
pair $(\bm{x},a)$ is
{\em matched} by the offline evaluator. Otherwise, if $\CI(\bm{x},a)=
\emptyset$ in Line 5 of OfflineEvaluator~\ref{alg:exact_match}, we say
$(\bm{x},a)$ is {\em unmatched}.
\begin{definition}
Let $M_a$ be the number
  of times that the offline evaluator returns the $a_{th}$ arm. 
\end{definition}
 Recall that $\Delta_a=\MBE[y|a^\ast]-\MBE[y|a]$. Then, the expected regret
  \begin{align}
    {R}(\CA_{\text{UCB+EM}}, T) = \sum_{a\in[K]} \MBE[ (\lambda_a - M_a)] \Delta_a.
    \label{eq:thm3_1}
  \end{align}

  Now, we count the number of times $M_a$ that an action $a$ is matched by the
  exact matching offline evaluator.
  Denote $M(\bm{x}^c,a)$ as the number of times the pair $(\bm{x}^c,a)$ is matched by the
  offline evaluator, hence $\sum_{c\in[C]} M(\bm{x}^c,a)=M_a$. We note that $M_a$ is
  the number of ``virtual plays''.

\noindent{\underline{\bf Step 2: (lower bound of $M_a$)} }
The lower bound of $M_a$ corresponds to the lower bound of
    {\em regret of virtual play}.
Note that when some context-action pair $(\bm{x},a)$ is unmatched, the matching
process for action $a$ will stop.
We consider the following two cases: (1)
  the matching process does not stop at $T$. In this case the expected
  number $\MBE[M(\bm{x}^c,a)]= \lambda_a {\MBP}[\bm{x}^c]$, because the context and action
  are generated independently for the context-independent decisions.
  (2) the matching process terminates before $T$.
  In this case, we run out of the samples with $(\bm{x}^{\tilde{c}},a)$. Suppose
  the unmatched context-action pair is
  $(\bm{x}^{\tilde{c}},a)$ (there are still samples for some other context $\bm{x}$), then the expected number of matched sample for some other
  context $\bm{x}^{{c}}$ is
  $\MBE[M(\bm{x}^c,a)]=N(\bm{x}^{\tilde{c}},a)\frac{\MBP[\bm{x}^c]}{\MBP[\bm{x}^{\tilde{c}}]}$.
  This is because the $M(\bm{x}^{\tilde{c}},a)=N(\bm{x}^{\tilde{c}}, a)$ and
  $\frac{\MBE[M(\bm{x}^c,a)]}{\MBE[M(\bm{x}^{\tilde{c}},a)]} = \frac{\MBP[\bm{x}^c]}{\MBP[\bm{x}^{\tilde{c}}]}$.
 The unmatched context can be any $\bm{x}^c$ $\forall c\in[C]$. Consider the worst case, then $M(\bm{x}^c,a)\ge
  \min_{\tilde{c}\in[C]}
  N(\bm{x}^{\tilde{c}},a)\frac{\MBP[\bm{x}^c]}{\MBP[\bm{x}^{\tilde{c}}]}$.
  Note that when $\tilde{c}=c$, we have
$\frac{N(\bm{x}^{\tilde{c}}{,}a)\MBP[\bm{x}^c]}{\MBP[\bm{x}^{\tilde{c}}]} = N(\bm{x}^{\tilde{c}},a)$.
Combining the counts of $M(\bm{x}^c,a)$ in the above two cases, we have
\begin{align}
\MBE[M_a] \ge \sum_{c\in[C]}\min \left\{ \min_{\tilde{c}\in[C]}
  \frac{N(\bm{x}^{\tilde{c}},a)\MBP[\bm{x}^c]}{\MBP[\bm{x}^{\tilde{c}}]}, \lambda_a{\MBP}[\bm{x}^c
]\right\}.
  \label{eq:thm3_2}
\end{align}

  Combine (\ref{eq:thm3_1}) and (\ref{eq:thm3_2}), and we note 
  $\lambda_a{=}\lambda_a \sum_{c\in[C]}{\MBP}[\bm{x}^c]$ (because
  $\sum_{c\in[C]}{\MBP}[\bm{x}^c]=1$ by definition), then
  \begin{align*}
   {R} & (\CA_{\text{UCB+EM}}, T) \le  \sum_{a\in[K]} \Delta_a \times\nonumber\\
  & \left(
      \sum_{c\in[C]} \hspace{-0.05in}
    \MBE\left[
      \max\{ \lambda_a{\MBP}[\bm{x}^c] {-}
\min_{\tilde{c}\in[C]}
  \frac{N(\bm{x}^{\tilde{c}},a)\MBP[\bm{x}^c]}{\MBP[\bm{x}^{\tilde{c}}]}, 0\}
    \right]\right).
  \end{align*}
  
We have the following equality:
\begin{align*}
  & \max\{ \lambda_a{\MBP}[\bm{x}^c] - 
\min_{\tilde{c}\in[C]}\frac{N(\bm{x}^{\tilde{c}},a)\MBP[\bm{x}^c]}{\MBP[\bm{x}^{\tilde{c}}]}
,0 \} \\
= & \max\{ l_a {\MBP}[\bm{x}^c] {+} (\lambda_a-l_a) {\MBP}[\bm{x}^c]
{-}\min_{\tilde{c}\in[C]}\frac{N(\bm{x}^{\tilde{c}},a)\MBP[\bm{x}^c]}{\MBP[\bm{x}^{\tilde{c}}]}
 , 0 \}  \\
= & \max\{ l_a {\MBP}[\bm{x}^c]
{-}\min_{\tilde{c}\in[C]}\frac{N(\bm{x}^{\tilde{c}},a)\MBP[\bm{x}^c]}{\MBP[\bm{x}^{\tilde{c}}]}
, 0 \} {+} (\lambda_a-l_a) {\MBP}[\bm{x}^c].
\end{align*}
where we define 
\begin{align}
  l_a \triangleq \lceil{ (8\ln (T+ \MBE[\sum_{a\in[K]} M_a]))/\Delta_a^2}
\rceil.
\label{eq:thm3_3}
\end{align}
 Then, $l_a \ge \MBE[\lceil{ \MBE[8\ln (T+\sum_{a\in[K]} M_a)] }\rceil]$ 
because $\ln(\cdot)$ is a concave function (according to Jensen's inequality,
the right term takes the expectation out). 
According Assumption \ref{asum:StableUnit:offline} and
\ref{asum:StableUnit:online} (stable unit in offline and online cases) and ``the
reward $y$ is bounded in $[0,1]$'', we can apply the results in paper of
Auer et al.\cite{auer2002finite} and $\MBE[\lambda_a-l_a]\le
1+\frac{\pi^2}{3}$ for some sub-optimal action $a\ne a^\ast$.
Therefore, we have 
\begin{align}
  {R}(&\CA_{\text{UCB+EM}}, T) \le 
\sum_{a\in[K]} \left(
  (1+\frac{\pi^2}{3}) +  \right. \nonumber\\
  &\left.  \sum_{c\in[C]}   \max\{ l_a {\MBP}[\bm{x}^c]
{-}\min_{\tilde{c}\in[C]}\frac{N(\bm{x}^{\tilde{c}},a)\MBP[\bm{x}^c]}{\MBP[\bm{x}^{\tilde{c}}]}
, 0 \}  \right)\Delta_a.
 \label{eq:thm3_4}
\end{align}

\noindent{\underline{\bf Step 3: (upper bound of $M_a$)}}
  To get an upper bound for $l_a$, we now give an upper bound for the expected number of samples that
  are matched, i.e. $\MBE[\sum_{a\in[K]}M_a]$.
  Recall that we denote the number of matched samples with context
  $\bm{x}^c$ and arm $a$ as $M(\bm{x}^c, a)$. Then, because it cannot exceed
  the number of data samples, we have ``the trivial bound''
  \begin{align}
    \label{eq:ub1}
    \MBE[M(\bm{x}^c, a)]\le  N(\bm{x}^c,a).
  \end{align}
  Also, because the expected number of matched
  samples cannot exceed the expected number of times the action is selected, we
  have ``the refined bound''
  \begin{align}
  \MBE[M(\bm{x}^c,a)]\le \MBE[\lambda_a] \MBP[\bm{x}^c].
    \label{eq:ub2}
  \end{align}
  Therefore, combining (\ref{eq:ub1}) and (\ref{eq:ub2}), we have 
\[
\MBE[M(\bm{x}^c,a)]\le \max\{N(\bm{x}^c,a), \lambda_a\MBP[\bm{x}^c]\}.
\]
  Then, 
\begin{align}
& \MBE[\sum_{a\in[K]} M_a]
\le \sum_{c\in[C]}\sum_{a\in[K]} \min\{N(\bm{x}^c,a), \lambda_a\MBP[\bm{x}^c]\} \nonumber\\
= & -\sum_{c\in[C]}\sum_{a\in[K]} \max\{- N(\bm{x}^c,a), - \lambda_a\MBP[\bm{x}^c]\} \nonumber\\
  = &
      \sum_{c\in[C]}\sum_{a\in[K]} N(\bm{x}^c,a) {-} \nonumber\\
\hspace{0.0in}& \sum_{c\in[C]}\sum_{a\in[K]} \max\left\{N(\bm{x}^c{,}a) {-} N(\bm{x}^c{,}a), N(\bm{x}^c{,}a) {-} \MBE[\lambda_a]\MBP[\bm{x}^c]\right\}\nonumber\\
= & N - \sum_{c\in[C]}\sum_{a\in[K]} \max\{0, N(\bm{x}^c,a) - \MBE[\lambda_a]\MBP[\bm{x}^c]\}\nonumber\\
\le & N \hspace{-0.02in} {-} \hspace{-0.07in}\sum_{c{\in}[C]}\hspace{-0.02in}\sum_{a{\in}[K]}
 \hspace{-0.07in}\max\{0{,} N(\bm{x}^c{,}a) {-} (8\frac{\ln(T{+}N)}{\Delta_a^2}{+}1{+}\frac{\pi^2}{3})\MBP[\bm{x}^c]\}
.
\label{eq:thm3_5}
\end{align}
Recall that $N$ is the number of all logged samples. The last equation is because $\lambda_a\le
8\frac{\ln(T+N)}{\Delta_a^2}+1+\frac{\pi^2}{3}$ according to the
paper~\cite{auer2002finite}.

Plug-in (\ref{eq:thm3_3}) and (\ref{eq:thm3_5}) to (\ref{eq:thm3_4}), then we have the upper bound claimed by our Theorem.
\end{proof}

\begin{proof}[\bf Proof of Theorem~\ref{thm:ps_matching}]
 The proof is similar to the proof of Theorem~\ref{mthm:exact_matching} for $\CA_{\text{UCB+EM}}$. The only difference is
 that for propensity score matching, the only context to be matched is the
 propensity score.

First, we will show that by matching the propensity score, the expected reward
in each round for each arm is not changed.

The expected reward when we choose action $a$ is
\begin{align*}
  \MBE[y|a] = \sum_{\bm{x}\in \CX}\MBP[\bm{x}] \MBE[y|a,\bm{x}], 
\end{align*}
 where $\MBE[y|a,\bm{x}]$ is the expected reward when the
context is $\bm{x}$ and the action is $a$.
We then consider the expected reward when we use the propensity score matching strategy.
Let us denote the propensity score of choosing an action $\tilde{a}$ under context $\tilde{\bm{x}}$
as
\begin{align*}
p(\tilde{\bm{x}},\tilde{a})=\MBP[a=\tilde{a}|\bm{x}=\tilde{\bm{x}}].
\end{align*}
By the propensity matching procedure, the expected reward of choosing an action $\tilde{a}$ is
\begin{align*}
  & \sum_{\bm{x}\in \CX}\MBP[\bm{x}] \MBE[y|\bm{p} {=} \bm{p}(\bm{x}), a{=}\tilde{a}] \\
= &     
 \sum_{\bm{x}\in \CX} \MBP[\bm{x}]\left( \sum_{c\in[Q]}\indicator{\bm{p}(\bm{x}){=}\bm{p}_c} \MBE[y|\bm{p}{=}\bm{p}_c, a{=}\tilde{a}]\right) \\
=& \sum_{c\in[Q]}\sum_{\bm{x}\in \CX} \MBP[\bm{x}] \indicator{\bm{p}(\bm{x}){=}\bm{p}_c} \MBE[y|\bm{p}{=}\bm{p}_c, a{=}\tilde{a}].
\end{align*}
and we have
\begin{align*}
  \MBE[y|\bm{p}{=}\bm{p}_c, a{=}\tilde{a}] {=}
  &
  \frac{
\sum_{\bm{x}\in\CX} \MBE[y|\bm{x},\tilde{a}]\times  \MBP[\bm{x}]
  \indicator{\bm{p}(\bm{x})=\bm{p}_c} \bm{p}_c(\tilde{a})
}{
\sum_{\bm{x}\in\CX} \MBP[\bm{x}]\times 
  \indicator{\bm{p}(\bm{x})=\bm{p}_c} \bm{p}_c(\tilde{a})
} \\
{=} &
 \frac{
\sum_{\bm{x}\in\CX} \MBE[y|\bm{x}, \tilde{a}] \times \MBP[\bm{x}]
  \indicator{\bm{p}(\bm{x})=\bm{p}_c} 
}{
\sum_{\bm{x}\in\CX} \MBP[\bm{x}]
 \times \indicator{\bm{p}(\bm{x})=\bm{p}_c}
}.
\end{align*}
Therefore, we have the expected reward
\begin{align}
    & \sum_{\bm{x}\in \CX}\MBP[\bm{x}] \MBE[y|\bm{p}{=}\bm{p}(\bm{x}), a{=}\tilde{a}] \nonumber\\
  {=} & \hspace{-0.02in} \sum_{c\in[Q]} \sum_{\bm{x}\in \CX} \MBP[\bm{x}]\indicator{\bm{p}(\bm{x}){=}\bm{p}_c} \hspace{-0.00in}
 \frac{
\sum_{\bm{x}\in\CX} \MBE[y|\bm{x}, \tilde{a}] \MBP[\bm{x}]
  \indicator{\bm{p}(\bm{x})=\bm{p}_c} 
}{
\sum_{\bm{x}\in\CX} \MBP[\bm{x}]
  \indicator{\bm{p}(\bm{x})=\bm{p}_c}
} \nonumber\\
 = & \hspace{-0.02in} \sum_{c\in[Q]} \sum_{\bm{x}\in\CX} \MBE[y|\bm{x}, \tilde{a}] \MBP[\bm{x}] \indicator{\bm{p}(\bm{x})=\bm{p}_c} \nonumber\\
 = & \sum_{\bm{x}\in\CX} \MBE[y|\bm{x}, \tilde{a}] \MBP[\bm{x}] 
= \MBE[y|\tilde{a}].
\end{align}
The last but one equation is from our assumption that all the propensity scores
are belong to a finite set $\{\bm{p}_1,\ldots,\bm{p}_Q\}$, and thus
$\sum_{c\in[Q]} \indicator{\bm{p}(\bm{x})=\bm{p}_c} = 1$ (namely, the propensity
score belongs to some value in the set).

Hence, our propensity score matching method unbiasedly estimate the 
$\MBE[y|\tilde{a}]$ for any action $\tilde{a}$.

With such unbiasedness property, the remaining is the same as Theorem~\ref{mthm:exact_matching},
except that the contexts $\bm{x}$ is replaced by the propensity score $\bm{p}$ verbatim.
\end{proof}

\subsection{Regret Bound for Contextual Algorithm $\CA_{\text{Fst}+\CE_\emptyset}$ (Theorem~\ref{thm:online_causal_forest})}
\label{sec:proof:online_forest}

\begin{comment}
\noindent{\bf Causal forest.}
Now, we analyze the causal forest algorithm. The analysis is from the confidence
bound from Wager and Athey's paper\cite{wager2018estimation}\cite{athey2019generalized} on the causal forest.

We first consider the asymptotic regret bound for the online version of the
causal forest algorithm. In this case, we do not know the exact structure of the
tree and the estimator is potentially biased. In this case, analyzing the regret
needs to deal with the bias-variance tradeoff. Recall that in the causal forest algorithm, we use the
$\epsilon$-decreasing exploration strategy. 
Note that our following analysis is for the {\em
  $\epsilon$-greedy-causal-forest} online oracle which only uses the online
feedbacks (not the algorithm that uses both data sources).
\end{comment}

{\NIPS
\begin{proof}[\bf Proof of Theorem~\ref{thm:online_causal_forest}]
The proof of Theorem~\ref{thm:online_causal_forest} consists of four parts.
First, Lemma~\ref{lemma:epsilon_lb} will show that if the exploration rate is
$\epsilon_t=t^{-1/2(1-\beta)}$, then in each data item of the dataset up till time $T$, any action $a\in[K]$ will be played with
a probability at least $\varepsilon_T=\frac{1}{K}T^{-1/2(1-\beta)}$, i.e.
$\MBP[A_t=a | X=\bm{x}] \ge \varepsilon_t$. Second,
Lemma~\ref{lemma:asymptotic_bias_var} will show that when each action was played
with probability at least $\varepsilon_t$ at time $t$, then the estimation
error at that time will be asymptotically bounded. Third, based on the previous
asymptotic results, our Lemma~\ref{lemma:finite_large_n} will show that when the
number of samples is large, the estimation error by our multi-action forest
estimator will be small with high probability. Fourth, we use
Lemma~\ref{lemma:asymptotic_bias_var} and Lemma~\ref{lemma:finite_large_n} to conclude that the
cummulative regret will be small.

\noindent\underline{\bf Step 1:}
Recall that in each time slot $t$, we have a
probability $\epsilon_t$ to draw a random action. 
Step 1 is to show that the $\epsilon$-decreasing strategy will
create an overlap condition for the dataset of online feedbacks. Moreover, we show that
compared to a constant exploration rate (instead of our $\epsilon$-decreasing
exploration), our strategy is not doing over-exploration up to a logarithmic factor.

\begin{lemma}
   \label{lemma:epsilon_lb}
   We have the following bound for the sum of power
   \begin{align}
     \label{eq:power_bound}
     T^{1-p} \le \sum_{t=1}^T t^{-p} \le T^{1-p}\log(T)^p, && \text{for some } p\in(0,1).
   \end{align}
   Applying to our case, we let $p=-\epsilon_0 = 1/2(1-\beta)$, and
   \[
   T^{1+\epsilon_0}\le
   \sum_{t=1}^T t^{\epsilon_0} \le T^{1+\epsilon_0}\log(T)^{-\epsilon_0}.
   \]
   Moreover, in the dataset collected till time $T$, for a randomly picked data
   point $(X, Y, A)$, we have $\MBP[A=a|X=x]\ge \frac{1}{K} T^{-1/2(1-\beta)}$. 
 \end{lemma}
 \begin{proof}
  The left inequality is easy to show. As $t^{-p}$ decreases in $t$,
  $T^{-p}\le t^{-p}$ for any $t\le T$, and thus
  $T^{1-p}=\sum_{t=1}^T T^{-p}\le \sum_{t=1}^T t^{-p}$. Now, we show the right
  inequality. According to Cauchy-Schwartz inequality (note that $1/p>1$),
  \begin{align*}
    &\frac{\sum_{t=1}^T t^{-p}}{T} \le \left( \frac{\sum_{t=1}^T (t^{-p})^{1/p}}{T}  \right)^{p}\\
    =& \left( \frac{\sum_{t=1}^T t^{-1}}{T}  \right)^{p} \le  \left( \frac{\log(T)}{T}  \right)^{p}.
  \end{align*}
  Then, we get the inequality $\sum_{t=1}^T t^{-p} \le T^{1-p}\log(T)^p$ that is
  (\ref{eq:power_bound}).
Then, we note that the expected total number of times to do the random
exploration is $\sum_{t=1}^T t^{\epsilon_0}$ till time $T$. Thus, the expected
number of times that we do the exploration in a randomly picked time slot is
${(\sum_{t=1}^T t^{\epsilon_0})}/{T}$. For a randomly picked data item, the
probability that an action is played $\MBP[A=a|X=\bm{x}]$ is greater than or
equal to $\frac{1}{K}$ times the probability that we do exploration in a randomly picked
time slot. Therefore, $\MBP[A=a|X=\bm{x}]\ge \frac{1}{K} (\sum_{t=1}^T t^{\epsilon_0})/T
= \frac{1}{K} T^{\epsilon_0}$.
 \end{proof}
In Lemma~\ref{lemma:epsilon_lb}, our main purpose is to give a lower bound on
the overlap (or ``exploration'') probability.
In particular, the lower bound $T^{1-p}$ corresponds to a fixed rate of
exploration $\epsilon_t= T^{-p}$ for $\forall t$. Then, for our
$\epsilon$-decreasing strategy we give an upper bound and a lower bound compaing
to two fixed-exploration-rate strategies.

\noindent\underline{\bf Step 2:}
In Lemma~\ref{lemma:epsilon_lb}, we have shown that our $\epsilon$-decreasing exploration
gives a {\em ``dynamic''} overlap condition, i.e. $\varepsilon_t$ changes in $t$. In contrast, the usual overlap
condition (e.g. \cite{imbens2015causal}) states a constant overlap probability.
Now, we will show that under this dynamic overlap condition, we have the
asymptotic convergence and normality properties for our multi-action forest estimator.

We first introduce the notation $\lesssim$. Here, $f(s)\lesssim g(s)$ means that $\lim_{s\rightarrow +\infty} \frac{f(s)}{g(s)} \le 1$.
 \begin{lemma}[Asymptotic bias and variance]
   \label{lemma:asymptotic_bias_var}
{\NIPS
     Suppose that we have $n$ i.i.d. training examples $(X_i, Y_i, A_i)\in [0,1]^d \times \mathbb{R} \times
  [k]$. Suppose the ignorability Assumption~\ref{assumption:ignorability} holds.
  Finally, suppose that all potential outcome distributions $(X_i,Y_i(a))$ for
  $\forall a\in[K]$ satisfy the same regularity assumptions as the pair
  $(X_i,Y_i)$ did in the statement of Theorem~3.1 in \cite{wager2018estimation}.
  Under this data-generating process, suppose the trained $\CF$ (in Line 11) is honest,
  $\alpha$-regular with $\alpha\le 0.2$ in the sense of
  Definition~\ref{def:honest} and \ref{def:alpha_regular}, and symmetric
  random-split (in the sense of Definition 3 and 5 in
  \cite{wager2018estimation}) multi-action forest.
}
 Denote
 $A{\triangleq}\frac{\pi}{d}\frac{\log((1-\alpha)^{-1})}{\log(\alpha^{-1})}$
  where $\pi\in[0,1]$ is a constant in Definition 3 of \cite{wager2018estimation}.
  {\color{black}Suppose in the fixed logged data of $n$ samples,
   \begin{align}
   \label{eq:overlap}
   \MBP[A=a|X=\bm{x}] > \varepsilon_n, \text{ for each } a\in[K], \text{ for any }\bm{x}.
   \end{align}
 }
 where $\varepsilon_n$ is a constant. Then for $s=n^\beta$ where $\beta=1-\frac{2A}{2+3A}$
 \begin{align}
   \label{eq:bias_bound}
   |\MBE[\hat{\mu}_n(\bm{x},a)] - \mu(\bm{x},a) | \lesssim Md\left( \frac{\varepsilon_n s}{2k-1} \right)^{-\frac{1}{2} \frac{\log\left( (1-\alpha)^{-1} \right)}{\log(\alpha^{-1})} \frac{\pi}{d}}.
 \end{align}
 In addition, there exists a sequence $\{\sigma_n\}_{n=1}^T$ where $\sigma_n=O(\frac{s}{n})$, $\frac{\MBE[\hat{\mu}_n(\bm{x},a)] - \hat{\mu}_n(\bm{x},
   a)}{\sigma_n(\bm{x})} \Rightarrow \CN(0,1)$ for $\forall a$, where ``$\Rightarrow$''
 means ``converges in distribution''.
 Here, $\hat{\mu}_n(\bm{x},a)\triangleq \frac{1}{B}\sum_{b\in[B]} \hat{L}_b(\bm{x},a)$
 is the prediction by the multi-action forest, with $n$ data samples.
   \end{lemma}
 \begin{proof}
   The proof mirrors the proof of Theorem~4.1 in \cite{wager2018estimation} (or
   Theorem~11 in its arXiv version\footnote{The paper's arXiv version is available at: https://arxiv.org/pdf/1510.04342.pdf}). The main steps involve bounding the bias
   of multi-action forests with an analogue to Theorem~3.2 in
   \cite{wager2018estimation} (or Theorem~3 in its arXiv version) and their
   incrementality using an analogue to Theorem~3.3 in \cite{wager2018estimation}
   (or Theorem~5 in its arXiv version). In general, the same arguments as used
   with regression forest in \cite{wager2018estimation} goes through, but the
   constants in the results get worse by a factor $\varepsilon_n$ that is the
   least probability that an action is played in the training data.
   Given these results, the subsampling-based argument from Section 3.3.2 in \cite{wager2018estimation} can be
   reproduced almost verbatim, and the final proof of this Theorem is identical
   to that of Theorem~3.1 in \cite{wager2018estimation} (or Theorem~1 in its
   arXiv version).

   As an ensemble method, the multi-action forest uses a subsample $s$ out of $n$
  data points to train a tree. 
  The subsample of data is denoted as $\CD_s=(Z_1,\ldots, Z_s) = ((X_{i_1},
  Y_{i_1}, A_{i_1}),\ldots, (X_{i_s}, Y_{i_s}, A_{i_s}))$.
  \cite{wager2018estimation} use the notation $X_i$ while we use the notation $\bm{x}_i$.

   \noindent{\bf Bias.} In this part, we want to show (we copy
   (\ref{eq:bias_bound}) below) :
 \begin{align*}
   |\MBE[\hat{\mu}_n(\bm{x},a)] - \mu(\bm{x},a) | \lesssim Md\left( \frac{\varepsilon_n s}{2k-1} \right)^{-\frac{1}{2} \frac{\log\left( (1-\alpha)^{-1} \right)}{\log(\alpha^{-1})} \frac{\pi}{d}}.
 \end{align*}
  To establish this claim, we first seek with an analogue to Lemma 2 in the
  arXiv version of \cite{wager2018estimation}, except now $s$ in (31) is
  replaced by $s_{\min}$, i.e., the minimum of the number of cases (i.e. the
  minimum number of observations for all the actions $a\in[K]$). Then, $s_{\min}/s
  \gtrsim \varepsilon_n$, because with probability at least $\varepsilon_n$ an
  action will be taken, so a variant of Equation (32) in
  \cite{wager2018estimation} where we replace $s$ with $\varepsilon_n s$ still
  holds for large $s$. 
  Notice that $\hat{\mu}(\bm{x},a)$ is a estimate of
  $\MBE[Y(a)|X=\bm{x}]$ (or $\mu(\bm{x},a)$)\footnote{Here, we actually do not need the
    ignorability Assumption~\ref{assumption:ignorability} (a.k.a.
    unconfoundedness) because the bandit algorithm does online intervention and
    we can directly get the feedback of $Y(a)$}. Then, we get
  (\ref{eq:bias_bound}) following the results of Theorem~3.2 in
  \cite{wager2018estimation} (or Theorem~3 in its arXiv version).

  We copy the definition of $\nu$-incrementality (Definition 6 of \cite{wager2018estimation}) here. 
  \begin{definition}
   The predictor $T$ is $\nu(s)$-incremental at $\bm{x}$ if 
   \begin{align*}
    \text{var}[\mathring{T}(\bm{x}; Z_1,\ldots, Z_s)]/\text{var}[\bm{x};Z_1, \ldots, Z_s] \gtrsim \nu(s),
   \end{align*}
   where $\mathring{T}$ is the H\'{a}jek projection
   \begin{align}
     \label{eq:hajek}
     \mathring{T} = \MBE[T] + \sum_{i=1}^n(\MBE[T|Z_i] - \MBE[T]).
   \end{align}
   In our notation, $f(s)\gtrsim g(s)$ means that ${\lim\inf}_{s\rightarrow
     \infty} f(s)/g(s)\ge 1$.
  \end{definition}
  
  \noindent{\bf Incrementality.} Suppose that the conditions of Lemma 3.2 of
  \cite{wager2018estimation} (or Lemma 4 in its arXiv version) hold and that $T$ is an honest
  $\alpha$-regular multi-action tree in the sense of Definition \ref{def:honest}
  and \ref{def:alpha_regular}. Suppose moreover that $\MBE[Y(a)|X=\bm{x}]$ and
  $\text{Var}[Y(a)|X=\bm{x}]$ for $\forall a\in[K]$ are all Lipschitz continuous
  at $\bm{x}$, and that $\text{Var}[Y|X=\bm{x}]>0$.
  {\color{black} Suppose, finally, that the
    overlap condition (\ref{eq:overlap}) holds with $\varepsilon_n>0$.}
  Then, $T$ is $\nu(s)$-incremental at $(\bm{x},a)$ with
  \begin{align*}
    \nu(s) = \varepsilon_n C_{f,d} / \log(s)^d,
  \end{align*}
  where $C_{f,d}$ is the constant from Lemma 3.2 of \cite{wager2018estimation} (or Lemma 4 in its arXiv version).

  To prove this claim, we follow the argument of the proof of Lemma 3.2 of
  \cite{wager2018estimation} (or Lemma 4 in its arXiv version).
  Like the proof in \cite{wager2018estimation}, we focus on the case where $f(x)=1$, in which case we use
  $C_{f,d}=2^{-(d+1)}(d-1)!$. We begin by setting up notation as in the proof of
  Lemma 3.2 of \cite{wager2018estimation} (or Lemma 4 in its arXiv version). We
  write the estimation for the action $a$ as $T^a(\bm{x}; \CD) = \sum_{i=1}^s S_i^a Y_i$, where
  \begin{align*}
    S_i^a =
    \begin{cases}
      |\{i:X_i\in L(\bm{x}; \CD_s), A_i = a\}|^{-1} & \text{if } X_i\in L(\bm{x}; \CD_s) \text{ and } A_i=a \\
      0 & else;
    \end{cases}
  \end{align*}
  where $L(\bm{x};\CD_s)$ denotes the leaf containing $\bm{x}$ in the tree trained
  with a subsample of data $\CD_s$.

      We also define the quantities
  \begin{align*}
   P_i^a = \bm{1}_{\{X_i \text{ is a }k\text{-PNN of }x \text{ among points with action }a\}}.
  \end{align*}
  where $k$-PNN ($k$-potential nearest neighbor) is defined in Definition 7 in Section 3.3.1 of \cite{wager2018estimation}.

  Because $T^a$ is a $k$-PNN predictor, $P_i^a=0$ implies that $S_i^a =0$. Moreover, by regularity of tree $T^a$ of the forest $\CF$, we know that the number of
  leaf samples $|\{i: X_i\in L(\bm{x}; \CD)\}| \ge k$. Thus, we can verify that
\begin{align}
  \label{eq:p_and_s}
  \MBE[S_1^a|Z_1]  \le \frac{1}{k} \MBE[P_1|Z_1]
\end{align}
We are now ready to use the same machinery as the Proof of Lemma 4 in the arXiv
version of \cite{wager2018estimation}. Similar to the Proof of Theorem 11 in the
arXiv version of \cite{wager2018estimation}, the random variable $P_1^a$ now
satisfy
\begin{align}
  \label{eq:p}
  \MBP\left[\MBE[P_1^a|Z_1]\ge \frac{1}{s^2\MBP[A_1=a]^2}\right] \lesssim k\times \frac{2^{d+1}\log(s)^d}{(d-1)!} \frac{1}{s\MBP[A_1=a]};
\end{align}
by the argument in (\ref{eq:p_and_s}) and $\varepsilon_n$-overlap
(\ref{eq:overlap}), (\ref{eq:p}) immediately implies that
\begin{align*}
  \MBP\left[\MBE[S_1^a | Z_1] \ge \frac{1}{k\varepsilon_n^2 s^2} \right] \lesssim k \frac{2^{d+1} \log(s)^d}{(d-1)!} \frac{1}{\varepsilon_n s}.
\end{align*}
By construction, we know that (because $\sum_{i=1}^s S_i^a=1$ by definition)
\begin{align*}
 \MBE[S_1^a | Z_1] = \MBE[S_1^a] = \frac{1}{s},
\end{align*}
which by the same argument as \cite{wager2018estimation} implies that
\begin{align}
  \label{eq:variance_bound}
  \MBE[ \MBE[S_1^a | Z_1]^2 ] \gtrsim \frac{(d-1)!}{2^{d+1}\log(s)^d} \frac{\varepsilon_n}{ks}.
\end{align}
The second part of the proof follows from a straight-forward adaptation of the
proof of Theorem~5 in the arXiv version of \cite{wager2018estimation}.

 So far, we have proved the tree estimator $T^a(\bm{x})$ is $\nu(s)$-incremental
 at $\bm{x}$ with $\nu(s) = \varepsilon_n C_{f,d}/\log(s)^d$.
 One can check that the proofs for Lemma~3.5 of
 \cite{wager2018estimation} (or Lemma~7 in its arXiv version) still goes through
 verbatim because the proof of Lemma~3.5 in \cite{wager2018estimation} uses the properties of the ensemble of forest, and
 our multi-action forest uses the same ensemble technique via subsampling.

 Now, we are going to show the result in Theorem~3.4 in
 \cite{wager2018estimation} (or Theorem 8 in its arXiv version), as follows:\\
 \noindent\underline{claim:} (in Theorem 3.4 of \cite{wager2018estimation}) ``Suppose, $\MBE[|Y-\MBE[Y|X=\bm{x}]|^{2+\delta} | X=\bm{x}] \le M$ for some
 constants $\delta,M>0$, uniformly over all $x\in[0,1]^d$. Then, there exists a
 sequence $\sigma_n(\bm{x},a)\rightarrow 0$ such that
 \begin{align*}
   \frac{\hat{\mu}_n(\bm{x},a) - \MBE[\hat{\mu}_n(\bm{x},a)]}{\sigma_n(\bm{x},a)} \Rightarrow \CN(0,1),
 \end{align*}
 where $\CN(0,1)$ is the standard normal distribution.
 ''
 Now we prove the above claim following the proof of Theorem 3.4 in
 \cite{wager2018estimation} (or Theorem~8 in its arXiv version).
 We focus on the trees w.r.t. the action $a$. Using the notation from Lemma~7 in
 the arXiv version of \cite{wager2018estimation}, let $\sigma_n(\bm{x},a)^2=s^2/n V_1$ be
 the variance of $\mathring{\hat{\mu}}$ {\color{black} ($\mathring{\hat{\mu}}$
   is the H\'ajek projection of $\hat{\mu}$ defined in (\ref{eq:hajek}))} where $V_1$ is defined in (41) in the
 arXiv version of \cite{wager2018estimation}. We know that
 \begin{align*}
  \sigma_n^2 = \frac{s}{n} s V_1 \le \frac{s}{n} \text{Var}[T^a].
 \end{align*}
{\color{black} Here, the variance of the base learner $\text{Var}[T^a]$ is finite by the
 Assumption in Lemma 3.3 in \cite{wager2018estimation}.}
So $\sigma_n\rightarrow 0$ as desired. Now, by our previous argument on the
incremental property, combined with Lemma 3.5 in \cite{wager2018estimation}, we
have ($\mathring{T^a}$ is the H\'ajek projection of $T^a$)
\begin{align}
  \frac{1}{\sigma_n^2} \MBE \left[ \left(  (\hat{\mu}_n(\bm{x},a)) - \mathring{\hat{\mu}}(\bm{x},a)\right)^2 \right] & \le \left( \frac{s}{n} \right)^2 \frac{\text{Var}[T^a]}{\sigma_n^2} \nonumber\\
                                                                                                                   & = \frac{s}{n} \text{Var}[T^a] / \text{Var}[\mathring{T^a}] \nonumber\\
                                                                                                                   & \lesssim \frac{s}{n} \frac{\log(s)^d}{\varepsilon_n C_{f,d}/4} \nonumber\\
                                                                                                                   & \rightarrow 0.                                                                                 \label{eq:mu_mathringmu}
\end{align}
Compared to the Proof of Theorem~8 in the arXiv version of
\cite{wager2018estimation}, the difference is that we add a term $\varepsilon_n$
for the incremental property. We have $\frac{s}{n} \frac{\log(s)^d}{\varepsilon_n
  C_{f,d}/4}\rightarrow 0$ by plugging in {\color{black} $s=n^\beta$ and
  $\varepsilon_n \ge n^{-\frac{1}{2}(1-\beta)}$.}
Then, following the proof of Theorem~8 in the arXiv version of
\cite{wager2018estimation}, all we need to check is that $\mathring{\hat{\mu}}$
is asymptotically normal. One way to do so is using the Lyapunov central limit
theorem (e.g. \cite{Bill08}). Writing
\begin{align}
  \label{eq:mathring_mu}
  \mathring{\hat{\mu}}(\bm{x}, a) = \frac{s}{n}\sum_{i=1}^n (\MBE[T^a|Z_i] - \MBE[T]),
\end{align}
it suffices to check the following Lyapunov's condition\footnote{From now on, the proof are the same as the proof of
Theorem~8 in the arXiv version of \cite{wager2018estimation} except that we
replace $S_i$ by $S_i^a$ and we replace $T$ by $T^a$ because we have multiple
actions}\footnote{Here, we use the notation $\tilde{\delta}$ instead of the
$\delta$ in usual Lyapunov condition}:
\begin{align}
  \label{eq:lyapunov}
  \lim_{n\rightarrow \infty} \sum_{i=1}^n \MBE\left[ |\MBE[T^a|Z_i] - \MBE[T^a]|^{2+\tilde{\delta}} \right]
  / \left( \sum_{i=1}^n \text{Var}[\MBE[T^a|Z_i]] \right)^{1+\tilde{\delta}/2} = 0
\end{align}
Using notation in the above discussion about incrementality, we write
$T^a=\sum_{i=1}^n S_i^a Y_i$. Thanks to honesty, we can verify that for any
index $i>1$, $Y_i$ is independent of $S_i^a$ conditionally on $X_i$ and $Z_1$,
and so (in the following, we slightly abuse the notation and $Y$ stands for
$Y(a)$ for some action $a$)
\begin{align*}
&  \MBE[T^a|Z_1] - \MBE[T^a]\\
=&\MBE[S_1^a(Y_1 - \MBE[Y_1|X_1]) | Z_1]
  + \left( \MBE\left[ \sum_{i=1}^n S_i^a\MBE[Y_i|X_i] | Z_1 \right] - \MBE[T^a] \right).
\end{align*}
Note that the two right-hand-side terms above are both mean-zero. By Jensen's
inequality, we also have that
\begin{align}
  \label{eq:decomposition}
  & 2^{-(1+\tilde{\delta})} \MBE\left[ |\MBE[T^a|Z_1] - \MBE[T^a]|^{2+\tilde{\delta}} \right]\nonumber\\
  \le & \MBE\left[ |\MBE[S_1(Y_1 - \MBE[Y_1|X_1]) | Z_1]|^{2+\tilde{\delta}} \right] \nonumber\\
  & + \MBE\left[ \left| \MBE\left[ \sum_{i=1}^n S_i^a \MBE[Y_i|X_i] | Z_1 \right] - \MBE[T^a] \right|^{2+\tilde{\delta}} \right].
\end{align}
Now, again by honesty (the sample used for estimation will not affect the
splitting of decision trees), $\MBE[S_1^a | Z_1] = \MBE[S_1^a | X_1]$, and so our
uniform $(2+\tilde{\delta})$-moment bounds on the distribution of $Y_i$ conditional on
$X_i$ implies that (recall that $M$ is the bounding constant in the Theorem's assumption)
\begin{align}
  \label{eq:bound1}
  &\MBE\left[ |\MBE[S_1^a(Y_1 - \MBE[Y_1|X_1])|Z_1] |^{2+\tilde{\delta}} \right] \nonumber\\
  =& \MBE\left[ \MBE[S_1^a|X_1]^{2+\tilde{\delta}} \left( |Y_1 - \MBE[Y_1|X_1]|\right)^{2+\tilde{\delta}} \right] \nonumber\\
   \le & M \MBE\left[ \MBE[S_1^a|X_1]^{2+\tilde{\delta}} \right] \le M \MBE\left[ \MBE[S_1^a|X_1]^2 \right],
\end{align}
because $S_1^a\le 1$. Meanwhile, because $\MBE[Y|X=\bm{x}]$ is Lipschitz, we can
define $u\triangleq \sup\{ |\MBE[Y|X=\bm{x}]|: \bm{x}\in[0,1]^d \}$, and see
that
\begin{align}
   \label{eq:bound2}
  \MBE [  | & \MBE \left[ \sum_{i=1}^n S_i^a \MBE[Y_i|X_i] | Z_1 \right] - \MBE[T^a] |^{2+\tilde{\delta}} ]\nonumber\\
 & \le (2u)^{\tilde{\delta}} \text{Var} \left[ \MBE\left[ \sum_{i=1}^n S_i^a \MBE[Y_i | X_i] | Z_1 \right] \right] \nonumber\\
 & \le  2^{1+\tilde{\delta}} u^{2+\tilde{\delta}} \left( \MBE\left[ \MBE[S_1^a|Z_1]^2 \right] + \text{Var}[(n-1)\MBE[S_2^a|Z_1]] \right) \nonumber\\
& \le (2u)^{2+\tilde{\delta}} \MBE\left[ \MBE[S_1^a|X_1]^2 \right].
\end{align}
Thus, the condition (\ref{eq:lyapunov}) that we need to check simplifies to
\begin{align}
  \label{eq:condition_to_check}
  \lim_{n\rightarrow \infty} n \MBE\left[ \MBE[S_1^a|X_1]^2 \right] / (n\text{Var}[\MBE[T^a|Z_1]])^{1+\tilde{\delta}/2} = 0.
\end{align}
Finally, as argued in the proofs of Theorem~5 and Corollary~6 in the arXiv
version of \cite{wager2018estimation},
\begin{align*}
 \text{Var}[\MBE[T^a | Z_1]] = \Omega \left( \MBE \left[ \MBE[S_1^a|X_1]^2 \right] \text{Var}[Y|X=\bm{x}]\right).
\end{align*}
Because the denominator in (\ref{eq:condition_to_check}) $\text{Var}[Y|X=\bm{x}]>0$ by assumption, we can use
(\ref{eq:variance_bound}) in our previous argument on the incrementality. Note
that the numerator in (\ref{eq:condition_to_check}) satisfies
\begin{align*}
 \left( n \MBE\left[ \MBE[S_1^a | X_1]^2 \right] \right)^{-\tilde{\delta}/2} \lesssim \left( \frac{C_{f,d}}{2k}\frac{\varepsilon_n n}{s\log(s)^d} \right)^{-\tilde{\delta}/2},
\end{align*}
which goes to $0$ when we plug in the values $s= O(n^\beta)$ and
$\varepsilon_n = n^{-1/2(1-\beta)}$. Compared to the formula in the proof of the
arXiv version of \cite{wager2018estimation}, we add a factor of $\varepsilon_n$
because of the overlap condition for a multi-action tree.
\end{proof}

{\NIPS
\noindent{\underline{\bf Step 3:}}
In this step, Lemma~\ref{lemma:finite_large_n} shows that for large sample size,
the estimation by our estimator is close to the true value with a high probability.

\begin{lemma}
  \label{lemma:finite_large_n}
For each $\omega^\prime>0$, there exists a $N_2>0$, such that for any $n>N_2$, we
 have for any $\delta > 0$
\begin{align} 
\MBP[ |\hat{\mu}_n(\bm{x}, a) - \MBE[\hat{\mu}_n(\bm{x}, a)] |  \le \sigma_n(\bm{x}, a)\delta ] 
\ge 1 - e^{-\delta^2/2} - \omega_n^\prime.
  \label{eq:thm8_2}
\end{align}
{\color{black}
  Here, $\omega_n^\prime =
  e^{-\delta^2/2}(4\delta\tilde{\varepsilon}+2\tilde{\varepsilon}^2) +
  \frac{C\psi\log n}{\sqrt{n}} + \left(\frac{s}{n}\frac{16
      \log(s)^d}{\varepsilon_n C_{f,d}}\right)^{1-2\omega/3}$ which is a
  function of $n$, where $\tilde{\varepsilon}\triangleq
  \left(\frac{s}{n}\frac{16 \log(s)^d}{\varepsilon_n
      C_{f,d}}\right)^{\omega/3}$. Recall that $\omega$ is the small constant in
  the theorem's statement.
}
\end{lemma}
\begin{proof}
  By Lemma~\ref{lemma:asymptotic_bias_var}, we know that
  $\frac{\hat{\mu}_n(\bm{x}, a) -
  \MBE[\hat{\mu}_n(\bm{x},a)])}{\sigma_n(\bm{x}, a)} \Rightarrow \CN(0,1)$, where
$\sigma(\bm{x}, a)\le \frac{s}{n}\text{Var}(T^a)$.

We will first show a property for a normal distributed random variable $X\sim
\CN(0,1)$, and then discuss the convergence rate towards the normal distribution. For every $\delta>0$,
\begin{align*}
 \MBP[|X| > \delta] = 2 \int_{0}^{+\infty} \frac{1}{\sqrt{2\pi}} e^{-(x+\delta)^2/2}dx,
\end{align*}
and, for every $x>0$,
\begin{align*}
 e^{-(x+\delta)^2} \le  e^{-t^2/2} e^{-x^2/2},
\end{align*}
hence
\begin{align}
 \label{eq:normal_property}
 & \MBP[|X|>\delta] \le 2 e^{-\delta^2/2} \int_{0}^{+\infty} \frac{1}{\sqrt{2\pi}}e^{-x^2/2}dx \nonumber\\
= & 2 e^{-\delta^2/2} \MBP[X> 0] = e^{-\delta^2/2}.
\end{align}

{\color{black}
  Now, we will further show the convergence rate towards the normal distribution.
  First of all, we will show the convergence of $\RHmu_n - \MBE[\RHmu_n]$
  ($\RHmu_n$ is the H\'ajek projection).
We now will show that $\RHmu_n - \MBE[\RHmu_n]$ has finite second absolute
moment and finite third absolute moment.
For the second absolute moment (variance), we have the following claim:
if $\MBE[|X|^{2+{\delta}}]\le M$ is bounded for some $\delta>0$, then
$\MBE[|X|^2]\le M+1$ is also bounded. To prove this claim, we only need to
discuss the cases when $|X|\le 1$ or $|X|>1$. In fact, $\MBE[|X|^2] =
\int_{|X|\le 1} |X|^2 f(X)dX + \int_{|X|>1}|X|^2 f(X)dX \le 1 +
\int_{|X|>1}|X|^{2+\delta} f(X)dX\le 1+M$ where
$f(\cdot)$ is the probability density function.

For the convergence rate, we have the following lemma:
\begin{lemma}[\cite{zahl1966bounds}]
 Letting $X_1,X_2,\ldots, X_n, \ldots$ be the sequence of independent random variables,
 $\MBE[X_i]=\mu_i$, $\MBE[(X_i-\mu_i)^2]=\sigma_i^2$,
 $\MBE[|X_i-\mu_i|^3]=\beta_i$. Let $F(x)$ be the CDF of $\sum_{i=1}^n
 (X_i-\mu_i)/(\sum_{i=1}^n \sigma_i^2)^{1/2}$, and $\Phi(x)$ be the CDF of the
 standard normal distribution. Then 
 \begin{align*}
  \sup_{x} |F(x)-\Phi(x)| < C \psi \log n / \sqrt{n},
 \end{align*}
 where $C$ is a constant, and $\psi$ is a function of the $\sigma_i$'s and $\beta_i$'s
 \footnote{When $X_1,X_2,\ldots, X_n$ are i.i.d. random variables, the $\log(n)$ term can
 be removed according to the Berry–Esseen theorem.}.
 \label{lemma:convergence_rate}
\end{lemma}

In our case, we let $X_i$ to be $\MBE[T^a|Z_i]$.
Next, we consider the $(2+\delta)$ absolute
moment $\MBE[|\MBE[T^a|Z_1]-\MBE[T^a]|^{2+\tilde{\delta}}]$.
  From the Inequality (\ref{eq:decomposition}) (\ref{eq:bound1}) and
(\ref{eq:bound2}), we know 
\begin{align*}
  & 2^{-(1+\tilde{\delta})}\MBE\left[ \left| \MBE[T^a|Z_1] - \MBE[T^a] \right|^{2+\tilde{\delta}} \right] \nonumber\\
  \le & M \MBE[ \MBE[S_1^a|X_1]^2 ] + (2u)^{2+\tilde{\delta}} \MBE[\MBE[S_1^a|X_1]^2].
\end{align*}
In addition, $\MBE[S_1^a|X_1]\le 1$ because $S_1^a\le 1$. Then,
\[
2^{-(1+\tilde{\delta})}\MBE\left[ \left| \MBE[T^a|Z_1] - \MBE[T^a]
  \right|^{2+\tilde{\delta}} \right]\le M+(2u)^{2+\tilde{\delta}}.
\]
Now, we have $\MBE\left[ \left| \MBE[T^a|Z_1] - \MBE[T^a]
  \right|^{2+\tilde{\delta}} \right]\le \left( M+(2u)^{2+\tilde{\delta}}
\right)\times 2^{(1+\tilde{\delta})}$.
When $\tilde{\delta}=0$, we have the second absolute moment
$\MBE[|\MBE[T^a|Z_i] - \MBE[T]|^2]$ is upper bounded by $4(M+4u^2)$.
Similarly, when $\tilde{\delta}=1$, we have the third absolute moment
$\MBE[|\MBE[T^a|Z_i] - \MBE[T]|^2]$ is upper bounded by $8(M+8u^3)$.

We notice that
\begin{align*}
\frac{s^2}{n^2}\sum_{i=1}^n \text{Var}[\MBE[T^a|Z_i]] =
  \text{Var}[\RHmu(\bm{x},a)] = \sigma_n^2.
\end{align*}
Now, based on the definition of $\RHmu_n(\bm{x},a)$ in (\ref{eq:mathring_mu}), we have
\begin{align*}
&  \frac{\sum_{i=1}^n (X_i-\mu_i)}{(\sum_{i=1}^n \text{Var}[\MBE[T^a|Z_i])^{1/2}}
  = \frac{\sum_{i=1}^n (\MBE[T^a|Z_i] - \MBE[T^a]}{(\sum_{i=1}^n \text{Var}[\MBE[T^a|Z_i] - \MBE[T^a]])^{1/2}} \nonumber\\
  = & \frac{\frac{n}{s}\RHmu_n(\bm{x},a) }{(\frac{n^2}{s^2}\text{Var}[\RHmu_n(\bm{x},a)])^{1/2}}
  = \frac{\RHmu_n(\bm{x},a) }{\sigma_n(\bm{x},a)}.
\end{align*}
  Thus, $F(x)$ is the CDF of the random variable $\frac{\RHmu_n(\bm{x},a) }{\sigma_n(\bm{x},a)}$.
  According to Lemma~\ref{lemma:convergence_rate}, we have
\[
  \sup_{x} |F(x)-\Phi(x)| < C \psi \log n / \sqrt{n}.
\]
Combined the property of normal CDF (\ref{eq:normal_property}), we have $\MBP[|\RHmu_n - \MBE[\RHmu_n]| \le \sigma_n\delta]\ge
1-e^{-\delta^2/2} - \frac{C\psi\log n}{\sqrt{n}}$.
Here, $\RHmu_n, \hat{\mu}_n, \sigma_n$ are short for $\RHmu_n(\bm{x},a)$,
$\hat{\mu}_n(\bm{x},a)$, $\sigma_n(\bm{x},a)$ respectively.

We now bound the large deviation probability for $\hat{\mu}_n$
\begin{align}
  \label{eq:mu_deviation}
  & \MBP[|\hat{\mu}_n - \MBE[\hat{\mu}_n]|\le \sigma_n\delta] \nonumber\\
  \ge & \MBP[|\RHmu_n - \hat{\mu}_n| + |\RHmu_n-\MBE[\RHmu_n]| +
        |\MBE[\hat{\mu}_n] - \MBE[\RHmu_n]]| \le \sigma_n\delta] \nonumber\\
  = & \MBP[|\RHmu_n - \hat{\mu}_n| + |\RHmu_n-\MBE[\RHmu_n]| \le \sigma_n\delta] \nonumber\\
  \ge & \MBP[|\RHmu_n-\MBE[{\RHmu}_n]| \le \sigma_n\delta - \tilde{\varepsilon}\sigma_n] - \MBP[|\RHmu_n - \hat{\mu}_n|> \tilde{\varepsilon}\sigma_n ] \\
        & (\text{the last inequality is because }\nonumber\\&~~~~ \MBP[|A|+|B|\le \delta] \ge \MBP[|A|\le \delta-\tilde{\varepsilon}] - \MBP[|B|>\tilde{\varepsilon}]) \nonumber
\end{align}

Before further development, we first show that the approximation argument in
(\ref{eq:mu_mathringmu}) can be turned into the bound in (\ref{eq:mu_mathringmu_new}) when $s\ge 4kde^{2d}$
where we recall $k$ is the constant for a regular tree.
The source for the approximation is from the proof of Lemma~4 in the arXiv
version of \cite{wager2018estimation}. In particular, we modify the approximation in
Equation (36) in the arXiv version of \cite{wager2018estimation}.
From Corollary 3.2 of \cite{borwein2009uniform}, we know for the upper
incomplete gamma function $\Gamma(d,c)$ we have $\Gamma(d,c)\le c^{d-1} e^{-c}
\times \left(1+\frac{1}{\frac{c}{d-1} - 1}\right)$ where $d,c$ are
real values. 
In the proof of Lemma~4 in the arXiv version of \cite{wager2018estimation}, $c=-\log(1-\exp[-2k\frac{\log(s)}{s-2k+1}])$.
One can verify that when $s\ge 4kd e^{2d}$, $\frac{c}{d-1}>2$, and thus
$\Gamma(d,c)\le 2 c^{d-1} e^{-c}$.

Moreover, we have $1-\exp\left[-2k\frac{\log(s)}{s-2k+1}\right] \le 4k
\frac{\log(s)}{s}$ when $s\ge 4k$.

Therefore, when $s\ge \max\{4k, 4kd e^{2d}\}=4kd e^{2d}$, the approximation
inequality (36) of the arXiv
version of \cite{wager2018estimation} is changed to $\MBP_{x=0}\left[
  \MBE[P_1|Z_1]\ge \frac{1}{s^2} \right]\le \frac{8k}{(d-1)!}\frac{\log(s)^d}{s}$.
Note that the upper bound becomes 4 times larger when we change ``$\lesssim$'' to ``$\le$''.
Thus, we can finally change the argument in Lemma~4 of arXiv version of
\cite{wager2018estimation} as $s\text{Var}[\MBE[S_1|Z_1]]\ge \frac{4}{k}
C_{f,d}/\log(s)^d$.
In Theorem~5, we will change the bound to
$\frac{\text{Var}[\mathring{T}(x;Z)]}{\text{Var}[T(x;Z)]}\ge \frac{\nu(s)}{4}$.
Next, we can change our (\ref{eq:mu_mathringmu}) to
\begin{align}
\frac{1}{\sigma_n^2}\MBE\left[ (\hat{\mu}_n(\bm{x},a) - \RHmu_n(\bm{x},a))^2
  \right]\le \frac{s}{n}\frac{16 \log(s)^d}{\varepsilon_n C_{f,d}}.
  \label{eq:mu_mathringmu_new}
\end{align}

For the second term of (\ref{eq:mu_deviation}), we have that
\begin{align}
  & \MBP[|\RHmu_n - \hat{\mu}_n| > \sigma_n\tilde{\varepsilon}]  = \MBP[|\RHmu_n - \hat{\mu}_n|^2 > \sigma_n^2\tilde{\varepsilon}^2] \nonumber\\
  \le & \frac{\MBE[|\RHmu_n - \hat{\mu}_n|^2]}{\sigma^2\tilde{\varepsilon}^2}
        \le 16\frac{s\sigma_n^2}{n}\frac{\log(s)^d}{\varepsilon_n C_{f,d}}/(\sigma_n^2\tilde{\varepsilon}^2)
        = 16\frac{s}{n}\frac{\log(s)^d}{\varepsilon_n C_{f,d}}/(\tilde{\varepsilon}^2)
   \label{eq:second_term}
\end{align}
Here, the last but one inequality is according to (\ref{eq:mu_mathringmu_new}) when
$s\ge 4kde^{2d}$.

Recall that we let $\tilde{\varepsilon}$ be
$\left(\frac{s}{n}\frac{16 \log(s)^d}{\varepsilon_n C_{f,d}}\right)^{\omega/3}$ which $\rightarrow
0$ as $n\rightarrow \infty$. Recall that $\omega>0$ is a small constant in our
theorem's statement.
There exists a $N_3$, such that when $n>N_3$, we have $\tilde{\varepsilon}<1$
and $4\tilde{\varepsilon} + 2\tilde{\varepsilon}^2 <1$.

Then, $\MBP[|\RHmu_n - \hat{\mu}_n| >
\sigma_n\tilde{\varepsilon}]\le \left(\frac{s}{n}\frac{16 \log(s)^d}{\varepsilon_n C_{f,d}}\right)^{1-2\omega/3}$.
Now, we let $N_2 = (4kde^{2d})^{1/\beta}$, so that when $n>N_2$ we have $s>4kde^{2d}$.
So far, we have bound
for the second term of the RHS of (\ref{eq:mu_deviation}).

For the first term of the RHS of (\ref{eq:mu_deviation}), we have
\begin{align}
  \MBP[|\RHmu_n - \MBE[\RHmu_n]| \le \sigma_n(\delta-\tilde{\varepsilon})]\ge 1-e^{-(\delta-\tilde{\varepsilon})^2/2} - \frac{C\psi\log n}{\sqrt{n}} \nonumber\\
  \ge 1-e^{-\delta^2/2}(1+4\delta\tilde{\varepsilon}+2\tilde{\varepsilon}^2) - \frac{C\psi\log n}{\sqrt{n}}
  ,
  \label{eq:first_term}
\end{align}
where the last inequality is because $e^{x}\le 1+2x$ for $x\in [0,1]$.

Combining inequations (\ref{eq:first_term}) and (\ref{eq:second_term}), we have
\begin{align*}
  & \MBP[|\hat{\mu}_n - \MBE[\hat{\mu}_n]|\le \sigma_n\delta] \\
    \le & 1-e^{-\delta^2/2}(1+4\delta\tilde{\varepsilon}+2\tilde{\varepsilon}^2) - \frac{C\psi\log n}{\sqrt{n}}
    - \left(\frac{s}{n}\frac{16 \log(s)^d}{\varepsilon_n C_{f,d}}\right)^{1-2\omega/3}.
\end{align*}
}

\begin{comment}
Actually, we now show that ``convergence in distribution'' (CDF $F_n$ converges
to $F$ point-wisely) implies
``uniform convergence in distribution'' (CDF $F_n$ uniformly converges
to $F$),  when the CDF $F(\cdot)$ of the convergent distribution
is continuous.
\footnote{The proof can also
  be found at https://math.stackexchange.com/questions/1670030/convergence-in-law-implies-uniform-convergence-of-cdfs}.
In fact,
for any $\varepsilon>0$, choose points $-\infty=x_0<x_1<\ldots<x_K=+\infty$, so
that $F(x_j)-F(x_{j-1})<\varepsilon$, for $j=1,2,\ldots, K$ (because $F$ is
continuous and non-decreasing, using the idea of $\varepsilon$-net). Then, we
can use the finite data points to approximate infinite data points. One can
verify that
\begin{align}
  \label{eq:epsilon_net}
  \sup_{x\in\mathbb{R}} |F_n(x)-F(x)|\le \max_{j=0,1,\ldots,K} |F_n(x_j)-F(x_j)|+\varepsilon
\end{align}
By the definition of {\em (point-wise) convergence in distribution}, for
$\forall \varepsilon$, there exist $N_j$, such that when $n>N_j$,
$|F_n(x_j)-F(x_j)|<\varepsilon$. Now, when $N>\max\{N_0,\ldots, N_K\}$, we have
$\max_{j=0,1,\ldots,K} |F_n(x_j)-F(x_j)| < \varepsilon$. Then, according to
(\ref{eq:epsilon_net}) and the definition of uniform convergence, $F_n$
uniformly converges to $F$.
\end{comment}

Note that $\sigma_n(\bm{x},a)>0$, then we get (\ref{eq:thm8_2}) in the statement
of Lemma~\ref{lemma:finite_large_n}.
\end{proof}
}

\noindent{\underline{\bf Step 4:}} With the results in
Lemma~\ref{lemma:asymptotic_bias_var} and 
Lemma~\ref{lemma:finite_large_n}, we now can prove
Theorem~\ref{thm:online_causal_forest} that gives an upper bound of the online regret of our
$\epsilon$-decreasing multi-action forest algorithm.

Now, let's go back to the proof of Theorem~\ref{thm:online_causal_forest}.
  First of all, we decompose the error into two parts
\begin{align}
  \label{eq:error_decompose}
|\hat{\mu}(\bm{x}, a) - \mu(\bm{x}, a)|
\le |\hat{\mu}(\bm{x}, a) - \MBE[\hat{\mu}(\bm{x}, a)]| + |\MBE[\hat{\mu}(\bm{x},a)] - \mu(\bm{x},a)|.
\end{align}
From (\ref{eq:bias_bound}) and the definition of ``$\lesssim$'', we know that there exists an integer $N_1>0$ and a constant $C_1>0$, such that for
any $n\ge N_1$ (and $s=n^\beta$ is a function of $n$), we have 
\begin{align}
  |\MBE[\hat{\mu}_n(\bm{x}, a)] - {\mu_n(\bm{x}, a)}| \le C_1 2Md\left( \frac{\varepsilon_n s}{2k-1}
  \right)^{-\frac{1}{2} \frac{\log\left( (1-\alpha)^{-1} \right)
    }{\log\left(\alpha^{-1} \right)} \frac{\pi}{d}}.
  \label{eq:thm8_1}
\end{align}

Now we combine (\ref{eq:thm8_1}) and (\ref{eq:thm8_2}). When $n>\max\{N_1,N_2\}$, with probability at least $1 - \frac{1}{2}
e^{-\delta^2/2} - \omega_n^\prime$, we have the following error bound
\begin{align}
|\hat{\mu}_n(\bm{x}, a) - \mu_n(\bm{x},a)| \le \sigma_n(\bm{x},a)\delta + 
2C_1Md\left( \frac{\varepsilon_n s}{2k-1}
  \right)^{-\frac{1}{2} \frac{\log\left( (1-\alpha)^{-1} \right)
    }{\log\left(\alpha^{-1} \right)} \frac{\pi}{d}}.
  \label{eq:error_bound}
\end{align} 

Now, we turn the error bound (\ref{eq:error_bound}) into the regret bound.
We note that at the beginning of time slot $t+1$, our online learning oracle
collects $t$ data points of feedbacks, where we can shuffle the data to be
i.i.d. samples satisfying Lemma~\ref{lemma:asymptotic_bias_var}.
Then, when
$t>N\triangleq \max\{N_1,N_2, N_3\}$, with a probability at least $1-e^{-\delta^2/2}-\omega_t^\prime$, the regret in
round $t+1$ for the online oracle (defined as $r_{t+1}$)
\begin{align*}
& r_{t+1} = \mu_t(\bm{x},a^\ast) - \mu_t(\bm{x},a)  \\
= & \left[ \mu_t(\bm{x},a^\ast) - \hat{\mu}_t(\bm{x}, a^\ast) \right] - \left[ \mu_t(\bm{x},a) -
  \hat{\mu}_t(\bm{x},a) \right] 
    + \left[\hat{\mu}_t(\bm{x}, a^\ast) - \hat{\mu}_t(\bm{x}, a) \right] \\
\le &  \left[ \mu_t(\bm{x},a^\ast) - \hat{\mu}_t(\bm{x}, a^\ast) \right] - \left[ \mu_t(\bm{x},a) -
  \hat{\mu}_t(\bm{x},a) \right] \\
\le & | \mu_t(\bm{x},a^\ast) - \hat{\mu}_t(\bm{x}, a^\ast) | + | \mu_t(\bm{x},a) - \hat{\mu}_t(\bm{x},a) | \\
\le & 2\sigma_t(\bm{x},a)\delta + 
4C_1Md\left( \frac{\varepsilon_t s}{2k-1}
  \right)^{-\frac{1}{2} \frac{\log\left( (1-\alpha)^{-1} \right)
      }{\log\left(\alpha^{-1} \right)} \frac{\pi}{d}}. \text{(recall that $s=t^\beta$)}
\end{align*}

We let $\delta_0=e^{-\delta^2/2}$, then $\delta=\sqrt{2\log(1/\delta_0)}$.
{\color{black}Recall that $\text{Var}[T^a(x)]$ is bounded by $V$
  \footnote{
{\color{black}
It is stated in Lemma 3.3 in \cite{wager2018estimation}.  Here, we use the proof of page 38 in the arXiv version of
  \cite{wager2018estimation} to justify a bound on $\text{Var}[T]$.
  In our regularity tree, each split has at least $k$ leafs. Thus,
  \begin{align*}
    k\text{Var}[T(x;Z)] \le |\{i:X_i\in L(x;Z)\}|\cdot \text{Var}[T(x;Z)] \rightarrow_p \text{Var}[Y|X=x].
  \end{align*}
  In addition, because of the regularity condition on the moment,
  $\text{Var}[Y|X=x]=\MBE[|Y-\MBE[Y|X=x]|^{2}|X=x]\le (M+1)$.
  Therefore, the variance $\text{Var}[T(x;Z)]$ is bounded. 
 }
}
then with probability at least $1-\delta_0-\omega_t^\prime$, for $t>N$ we have
\begin{align}
r_{t+1} \le & 2 \sqrt{t^{\beta-1} V} \sqrt{2\log(\frac{1}{\delta_0})} \nonumber\\
& + 4C_1Md\left( \frac{\varepsilon_t s}{2k-1}
  \right)^{-\frac{1}{2} \frac{\log\left( (1-\alpha)^{-1} \right)
    }{\log\left(\alpha^{-1} \right)} \frac{\pi}{d}} + \varepsilon_t \Delta_{\max},
 \label{eq:thm8_4}
\end{align}
where $\Delta_{\max}$ denotes the maximum regret
for choosing a sub-optimal action as defined in \cite{abbasi2011improved}\footnote{For $\Delta_{\max}$ to exist, we have
a mild assumption that the average rewards are bounded for each actions.}.
Recall that we denote $A{=}\frac{\log((1-\alpha)^{-1})\pi}{\log(\alpha^{-1})d}$.
Now we denote $\epsilon_0=-\frac{A}{2+3A}$, and $\varepsilon_t = t^{\epsilon_0}$.
One can check that $\beta =1-\frac{2A}{2+3A} = \frac{1-A\epsilon_0}{1+A}$. 

Here, we notice $(\varepsilon_t
s)^{-\frac{1}{2}A}=t^{-\frac{1}{2}A(\beta+\epsilon_0)}$.
One can check that by the above 
parameters setting, each terms in (\ref{eq:error_bound}) have the same exponent
w.r.t. $t$, i.e.
\begin{align}
  \label{eq:assignment}
\frac{1}{2}(\beta-1) =
  -\frac{1}{2}A(\beta+\epsilon_0)=\epsilon_0=-\frac{A}{2+3A}
\end{align}

Then (\ref{eq:thm8_4}) can be rewritten as (with probability at least $1-\delta_0-\omega_t^\prime$)
\begin{align}
r_{t+1} {\le}\hspace{-0.00in} \left(\hspace{-0.00in}  2\sqrt{V}\sqrt{2 \log(\frac{1}{\delta_0})} {+} 4C_1Md(2k{-}1)^{\frac{1}{2}A} {+}
\Delta_{\max} \hspace{-0.00in}\right) \hspace{-0.00in} t^{\beta-1}.
  \label{eq:thm8_5}
\end{align}

Consider the probability $\delta_0+\omega_t^\prime$, from (\ref{eq:thm8_5}) we have
\begin{align*}
r_{t+1} {\le}&\hspace{-0.00in} \left(\hspace{-0.00in}  2\sqrt{V}\sqrt{2 \log(\frac{1}{\delta_0})} {+} 4C_1Md(2k{-}1)^{\frac{1}{2}A} 
{+}
\Delta_{\max} \hspace{-0.00in}\right) \hspace{-0.00in} t^{\beta-1} \\
& + (\delta_0+\omega_t^\prime)\Delta_{\max}.
\end{align*}

Let $C_3 \triangleq
\left(2\sqrt{V}\sqrt{2\log(\frac{1}{\delta_0})} + 4C_1Md(2k-1)^{\frac{1}{2}A} +
  \Delta_{\max} \right)$ be a constant. Then, we further denote $p{\triangleq}
\frac{2+3A}{A}>1$ (where $p=\frac{2}{1-\beta}$) and by H\"older's inequality, when $T>N$ we have 
\begin{align*}
  & R(T, \CA_{\text{Fst}+\CE_\emptyset})
  = \sum_{t=1}^{N} r_t + \sum_{t=N+1}^T r_t \\
\le& \sum_{t=1}^{N} r_t + \left( (T{-}N)\delta_0{+}\sum_{t=N+1}^T\omega_t^\prime \right) \Delta_{\max} + T^{1-1/p} C_3 \left(\sum_{t=1}^T (\frac{r_t}{C_3} )^p\right)^{1/p} \\
 = & \sum_{t=1}^{N} r_t +\left( (T{-}N)\delta_0{+}\sum_{t=N+1}^T\omega_t^\prime \right)\Delta_{\max}+ C_3 T^{1-\frac{1}{p}} (\sum_{t=1}^T\frac{1}{t})^{\frac{1}{p}} \\
\le & \sum_{t=1}^{N} r_t + \left (\left( (T{-}N)\delta_0{+}\sum_{t=N+1}^T\omega_t^\prime \right)+N \right)\Delta_{\max} +  C_3 T^{1-\frac{1}{p}} (\log T)^{\frac{1}{p}},
\end{align*}
where the last inequality holds because $\sum_{t=1}^T\frac{1}{t}\le \log T$.

Now, we let $\delta_0 = T^{-\frac{A}{2+3A}}$.

Here, 
\begin{align*}
& \sum_{t=N+1}^T\omega_t^\prime= \\
& \sum_{t=N+1}^T \left( \delta_0
  (4\sqrt{2\log(\frac{1}{\delta_0})} \tilde{\varepsilon} {+} 2\tilde{\varepsilon}^2) {+}
  \frac{C\psi \log(t)}{\sqrt{t}} {+} \left(\frac{s}{t}\frac{16 \log(s)^d}{\varepsilon_t C_{f,d}}\right)^{1-\frac{2\omega}{3}} \right).
\end{align*}
Recall that when $n>N_3$, $\tilde{\varepsilon}<1$, and in our parameter setting
$\frac{s}{t \varepsilon_t} = t^{-1/2(1-\beta)}= t^{-1/p}$.
Hence,
\begin{align*}
  \sum_{t=N+1}^T\omega_t^\prime 
  \le &\sum_{t=N+1}^T \left( \delta_0 (4\sqrt{2\log(1/\delta_0)} + 2)  \right.\\
  & \left. + \frac{C\psi \log(t)}{\sqrt{t}} + \left(t^{-\frac{1}{p}}\frac{16 \log(s)^d}{C_{f,d}}\right)^{1-2\omega/3}  \right) \\
  \le & (T-N) (T^{-\frac{1}{p}} (4\sqrt{2\frac{1}{p}\log(T)}+2)) 
    + (T-N) \frac{C\psi\log(T)}{\sqrt{T}}\\
  & + (T-N) \left(T^{-\frac{1}{p}}\frac{16 \log(T)^d}{C_{f,d}}\right)^{1-2\omega/3}  \\
  \le & T^{1-\frac{1}{p}} (4\sqrt{2\frac{1}{p}\log(T)}+2) + \sqrt{T} C\psi\log(T)\\
 &  + T^{1-\frac{1}{p}+\frac{1}{2}\frac{2\omega}{3}}\left(\frac{16 \log(T)^d}{C_{f,d}}\right)^{1-2\omega/3}.
\end{align*}

We notice that $1-\frac{1}{p}>\frac{1}{2}$, so the exponent $T^{1-\frac{1}{p} +
  \frac{1}{3}\omega}$ dominates, and we use another $T^{\frac{1}{6}\omega}$ to
hide the $\log(T)$ terms. 
Then we have $R(T,\CA_{\text{Fst}+\CE_\emptyset})=O(T^{1-\frac{1}{p} +
  \frac{1}{2}\omega})$. Note that $1/p=\frac{1}{2}(1-\beta)$, then
\begin{align*}
\lim_{T\rightarrow +\infty} \frac{R(T,\CA_{\text{Fst}+\CE_\emptyset})}{T^{1- \frac{1}{2}(1-\beta)
  + \frac{\omega}{2} }} =
  \lim_{T\rightarrow +\infty} \frac{R(T,\CA_{\text{Fst}+\CE_\emptyset})}{T^{\frac{1+\beta+\omega}{2} }}
  = 0 && \text{ for any small } \omega>0.
\end{align*}
  Thus, using the big-$O$ notation, $\lim_{T\rightarrow +\infty} \frac{R(T,\CA_{\text{Fst}+\CE_\emptyset})}{T} = O(T^{-\frac{A}{2+3A}
  + \frac{\omega}{2}})$ for any small $\omega$. 

Finally, one can verify $1-\frac{A}{2+3A}=\frac{1+\beta}{2}$ which is less than $1$. Then, we reach our claim in the theorem that 
$\lim\nolimits_{T\rightarrow +\infty} \frac{{R}(T,\CA_{\text{Fst}+\CE_\emptyset})}
{T^{(1+\beta+\omega)/2}} = 0$,
and $\lim\limits_{T\rightarrow +\infty} \frac{R(T,\CA_{\text{Fst}+\CE_\emptyset})}{T}=0$ for
any $\omega$ that is smaller than $\frac{1-\beta}{2}$.
Namely, we have shown that the asymptotic regret is sub-linear w.r.t. $T$. 
}
\end{proof}

        % Now, we complete the proof for the Theorem.
 
\subsection{Regret Bound for Contextual Independent Algorithm $\CA_{\text{UCB+IPSW}}$ (Theorem~\ref{mthm:ipsw})}

\begin{proof}[\bf Proof of Theorem~\ref{mthm:ipsw}]
 The proof follows the same idea as previous ones. 
 We will first show that the estimation relying on the offline data is unbiased.
Second, we use a weighted Chernoff bound to show the effective number of logged
samples (a.k.a. Effective Sample Size) in terms of the confidence bound.

Many previous works have shown the inverse propensity weighting method provides
an unbiased estimator\cite{swaminathan2015counterfactual}. In fact, for $\tilde{a}\in[K]$
\begin{align*}
\MBE[\bar{y}_{\tilde{a}}] & = 
\frac{
  \MBE[ \sum_{i\in [-I]} \MBE[y|\bm{x}_i,\tilde{a})] \MBE[\indicator{a_i=\tilde{a}}] / p(\bm{x}_i,\tilde{a}]) 
}{
  \sum_{i\in [-I]} \MBE[\indicator{a_i=\tilde{a}}] / p(\bm{x}_i,\tilde{a}) ]
} \\
& =  \frac{
 \MBE[\sum_{i\in [-I]} \MBE[y |\bm{x}_i, \tilde{a}]]
}{
 I 
} \\
& = \sum_{\bm{x}\in \CX} \MBP[\bm{x}] \MBE[y|\bm{x},\tilde{a}] = \MBE[\bar{y}_{\tilde{a}}].
\end{align*}
The second equation holds because the probability that we observe the action
$\tilde{a}$ is $\MBE[\indicator{a_i=\tilde{a}}]$ which is the propensity score $p(\bm{x}_i,\tilde{a})$.
The last equation is because the expectation for data item $i$ is taken over the
contexts $\bm{x}$.

According to Chernoff-Hoeffding bound~\cite{hoeffding1994probability}, we have
the following Lemma.
\begin{lemma}
 If $X_1,X_2, \ldots, X_n$ are independent random variables and $A_i\le X_i \le B_i(i=1,2,\ldots,n)$, we have
 the following bounds for the sum $X=\sum_{i=1}^n X_i$:
\[
  \MBP[X \le \MBE[X] - \delta] \le e^{-\frac{2\delta^2}{\sum_{i=1}^n (B_i-A_i)^2 }}.
\]
\[
  \MBP[X \ge \MBE[X] + \delta] \le e^{-\frac{2\delta^2}{\sum_{i=1}^n (B_i-A_i)^2 }}.
\]
\end{lemma}
In our case to estimate the outcome for an action $a$, we have 
$X_i = y_i \frac{\indicator{a_i={a}} /
  p(\bm{x}_i,a_i)}{\sum_{i\in[-I]} \indicator{a_i=a}/p(\bm{x}_i, a_i)} $, and 
$X = \sum_{i\in [-I]} X_i = \bar{y}_a$. 
Hence the constants $A_i = 0$, $B_i = \frac{\indicator{a_i=a} /
  p(\bm{x}_i,a_i)}{\sum_{i\in[-I]} \indicator{a_i=a}/p(\bm{x}_i, a_i)}$.
Therefore, we have 
\begin{align*}
  & \MBP[ |\bar{y}_a - \MBE[y|a]| \ge \delta ] \\
  \le &
  2e^{
    -\frac{
      2\delta^2
    }{
       \sum_{i\in [-I]} \left( 
        \frac{\indicator{a_i=a} /
         p(\bm{x}_i,a_i)}{\sum_{i\in[-I]} \indicator{a_i=a}/p(\bm{x}_i, a_i)}
      \right)^2  
    }
  } \\
 = & 2e^{
    -\frac{
      2\delta^2
    }{
       \frac{
       \sum_{i\in [-I]} \left( \indicator{a_i=a} /
         p(\bm{x}_i,a_i) \right)^2 
       }{\left(\sum_{i\in [-I]} \indicator{a_i=a}/p(\bm{x}_i, a_i)\right)^2}
    }
  } \\
 = & 2e^{
    -
      2 \delta^2
      \frac{
         \left( \sum_{i\in [-I]} \indicator{a_i=a}/p(\bm{x}_i, a_i) \right)^2
      }{
          \sum_{i\in [-I]}  \left(\indicator{a_i=a} /
         p(\bm{x}_i,a_i) \right)^2 
      }
  }
\end{align*}
We compare it with the Chernoff-Hoeffding bound used in the UCB
algorithm\cite{auer2002finite}. When we have $n_a$ online samples of arm $a$,
\[
  \MBP[ |\bar{y}_a - \MBE[y|a]| \ge \delta] \le 2e^{-2n_a\delta^2}.
\]
By this comparison, we let $n=\widehat{N}_a$ and we will get the same bound.

Now, we show that by using these $\lfloor \widehat{N}_a \rfloor$ samples from logged data, the online bandit UCB
oracle will always have a tighter bound than that for $\lfloor \widehat{N}_a
\rfloor$ i.i.d. samples from the online environment.

In the online phase, let the number of times to play the action $a$ to be $T_a$. For the offline samples, let $X_i {=} y_i\frac{\indicator{a_i=a}/p(\bm{x}_i,a_i)
}{\sum_{i\in[-I]}\indicator{a_i=a}/p(\bm{x}_i,a_i)
}\frac{\widehat{N}_a}{\widehat{N}_a + T_a}$.
For the online samples, let $X^t {=} y_t \frac{1}{\widehat{N}_a + T_a}$.
Let us consider the sequence $\{X_1, \ldots, X_I, X^1, \ldots, X^{T_a}\}$.
Now, $X = \sum_{i\in[-I]} X_i + \sum_{t\in [T_a]}X^t$.
Then, we have $\MBE[X] = \MBE[y|a]$, and $0\le X_i \le
\frac{\widehat{N}_a}{\widehat{N}_a+T_a}B_i (\forall i{\in}[-I])$, $0{\le} X^t {\le}
\frac{1}{\widehat{N}_a+T_a}$. In addition, we have
\begin{small}
\begin{align*}
 & \left(\frac{\widehat{N}_a}{\widehat{N}_a {+} T_a}\right)^2  \frac{\sum_{i\in [-I]} \left(\indicator{a_i=a}/p(\bm{x}_i,a_i)\right)^2
}{\sum_{i\in[-I]}\indicator{a_i=a}/p(\bm{x}_i,a_i) }
  {+} \hspace{-0.05in} \sum_{t\in[T_a]} \hspace{-0.07in}\left(\frac{1}{\widehat{N}_a{+}T_a}\right)^2 \\
& =  \left(\frac{\widehat{N}_a}{\widehat{N}_a + T_a}\right)^2 
\left( \frac{1}{\widehat{N}_a}  \right)
 + \frac{T_a}{(\widehat{N}_a+T_a)^2} 
= \frac{1}{\widehat{N}_a+T_a}.
\end{align*}
\end{small}
Therefore, 
\[
  \MBP[\bar{y}_a \le \MBE[y|a] -\delta ] \le e^{-2\delta^2 (\widehat{N}_a+T_a)},
\]
\[
  \MBP[\bar{y}_a \ge \MBE[y|a] +\delta ] \le e^{-2\delta^2 (\widehat{N}_a+T_a)}.
\]
In other words, when we have $T_a$ online samples of an action $a$, the confidence
interval is as if we have $T_a+\widehat{N}_a$ total samples for the bandit oracle.
Then, the regret bound reduces to the case where we have $\widehat{N}_a$ offline
samples for arm $a$ that do not have contexts.
\end{proof}

\subsection{Regret Bound for Contextual Algorithm $\CA_{\text{LinUCB+LR}}$
  (problem dependent Theorem~\ref{mthm:ContextDepend:LinPI} and problem
  independent Theorem~\ref{mthm:linear_problem_dependent})}

\begin{proof}[\bf Proof of Theorem~\ref{mthm:ContextDepend:LinPI}]
 The proof follows the analytical framework of the
 paper\cite{abbasi2011improved}. Especially, this Theorem corresponds to the
 Theorem 3 in the paper\cite{abbasi2011improved}.
 The
 proofs in papers\cite{auer2002using}\cite{chu2011contextual} have similar ideas.
  
In particular, we consider that the offline
samples have features $\bm{x}_{-1}, \bm{x}_{-2}, \ldots, \bm{x}_{-N}$, and the online samples have
features $\bm{x}^1, \bm{x}^2, \ldots, \bm{x}^T$. To have a unified index system,
we let $\bm{x}_{N+t}\triangleq \bm{x}^t$ for $t\ge 1$.

Because we choose the ``optimal'' action in the online phase, we have 
the pseudo-regret in time slot $t$ is 
\[
r_t \le 2 \sqrt{\beta_{t-1}(\delta)} \min\{ ||\bm{x}_{N+t}||_{V_{N+t-1}^{-1}}, 1 \}.
\]
Then, we have (recall that in this paper, we set $V_0$ as a $d\times d$ identity
matrix $\bm{I}_d$)
\begin{align*}
  & \sqrt{8\beta_n(\delta)} \sum_{n=1}^N  \min\{1,||\bm{x}_{n}||_{V_{n-1}^{-1}}\} + \sum_{t=1}^T r_t  \\
\le & \sqrt{8 (N+T) \beta_n(\delta) \log \frac{ \texttt{trace}(V_0)+(N+T)L^2 }{\texttt{det} V_0 } }.
\end{align*}
Here, we observe that 
\begin{align*}
\sum_{t=1}^T r_t & \le 
\sqrt{8 (N+T) \beta_n(\delta) \log \frac{ \texttt{trace}(V_0)+(N+T)L^2
  }{\texttt{det} V_0 } } \\
& -
 \sqrt{8\beta_n(\delta)} \sum_{n=1}^N  \min\{1,||\bm{x}_{n}||_{V_{n-1}^{-1}}\} 
.
\end{align*}
Now, we give a lower bound of the last term 
\[
\sqrt{8\beta_n(\delta)} \sum_{n=1}^N  \min\{1,||\bm{x}_{n}||_{V_{n-1}^{-1}}\}.
\]

Here, $||\bm{x}||_{A} = \sqrt{\bm{x}^T A \bm{x}} \ge \sqrt{\lambda_{\min} (A)} ||\bm{x}||_2$.
We have the following claim that $\lambda_{\min}(V_n^{-1}) \ge \frac{1}{1+(n-1)L^2}$. This is
because $\lambda_{\min}(V_n^{-1}) = 1/\lambda_{\max}(V_n)$. 
In fact, for the symmetric matrices, we have 
\[
\lambda_{\max}(A+B) \le \lambda_{\max}(A) + \lambda_{\max}(B).
\]
We have $\lambda_{\max}(I) = 1$, and $\lambda_{\max} (\bm{x} \bm{x}^T) =||\bm{x}||_2^2$. 
Therefore, 
\[
\lambda_{\max}(V_{n-1}) \le 1 + ||\bm{x}_1||_2^2 +
\ldots + ||\bm{x}_{n-1}||_2^2 \le 1 + (n-1) ||\bm{x}||_{\max}^2,
\] 
where we
consider $||\bm{x}_i||_2^2\le ||\bm{x}||_{\max}^2$ for $i\in[n]$.
Also, we consider $||\bm{x}_i||_2^2\ge ||\bm{x}||_{\min}^2$ for $i\in[n]$. 

Let $L=||\bm{x}||_{\max}$. Then, 
\begin{align*}
&\sum_{n=1}^N  \min\{1,||\bm{x}_{n}||_{V_{n-1}^{-1}}\}\\
\ge & \sum_{n=1}^N \min\{1, ||\bm{x}||_{\min} \sqrt{\frac{1}{1+(n-1)L^2}}  \}\\
\ge & \min\{1, ||\bm{x}||_{\min}\} \sum_{n=1}^N \sqrt{\frac{1}{1+(n-1)L^2}} \\
\ge & \min\{1, ||\bm{x}||_{\min}\} \sum_{n=1}^N \frac{2}{L^2} \left( \sqrt{1{+}nL^2} - \sqrt{1{+}(n{-}1)L^2} \right) \\
 = & \min\{1, ||\bm{x}||_{\min}\} \frac{2}{L^2} \left( \sqrt{1+NL^2} - 1\right)
.
\end{align*}
Hence, we have the final bound of regret 
\begin{align*}
& \sum_{t=1}^T r_t \le 
\sqrt{8 (N{+}T) \beta_n(\delta) \log \frac{ \texttt{trace}(V_0){+}(N{+}T)L^2
  }{\texttt{det} V_0 } }  \\
& -
 \sqrt{8\beta_n(\delta)} \min\{1, ||\bm{x}||_{\min}\} \frac{2}{L^2} \left( \sqrt{1{+}NL^2} {-} 1\right).
\end{align*}
\end{proof}
Compared with the previous regret bound without offline data, the regret bound
changes from $O(\sqrt{T})$ to $O(\sqrt{N+T}) - \Omega(\sqrt{N})$.
From the view of regret-bound, using offline data does not bring us a large
amount of regret-reduction.

We now show a better bound for the problem-dependent case. This corresponds to
section 5.2 of the paper\cite{abbasi2011improved}.
Let $\Delta_t$ be the ``gap'' at step $t$ as defined in the paper of Dani et
al.\cite{dani2008stochastic}. Intuitively, $\Delta_t$ is the difference between
the rewards of the best and the ``second best'' action in the decision set
$D_t$. We consider the samllest gap $\bar{\Delta}_n = \min_{1\le t \le n}\Delta_t$.

\begin{proof}[\bf Proof of Theorem~\ref{mthm:linear_problem_dependent}]
  We will first show a high-probability bound, i.e. with probability at least
  $1-\delta$, the cummulative regret has the bound
  \begin{align*}
  R(T, \CA_{\text{LinUCB+LR}}) \le \frac{4\beta_{N+T}(\delta)}{{\Delta}_{\min}} d\log(1+\kappa) 
  \end{align*}
 when the parameters $\{\beta_t\}_{t=1}^T$ ensure the confidence bound in each time slot.
  
Recall that the contexts of samples returned by the offline evaluator are
$\bm{x}_{-1}, \bm{x}_{-2},\ldots, \bm{x}_{-N}$. We denote $r_t\triangleq \max_{a\in [K]}
  \MBE[y_t|\bm{x}_t,a]-\MBE[y_t|\bm{x}_t, a_t]$ as the pseudo-regret in time
  slot $t$. Recall that $\beta_t(\delta)$ is the parameter $\beta_t$ in the
  $t^{th}$ time slot, and the $\delta$ is to emphasize that it is a function of $\delta$.
  From the proof for the problem-independent bound in paper\cite{abbasi2011improved}, we know  $\sum_{t=1}^T r_t \le \frac{4\beta_{N+T}(\delta)}{{\Delta}_{\min}} \log \frac{\texttt{det} V_{T}}{\texttt{det}V_N} $.
 The following is to bound $\log \frac{\texttt{det} V_{N+T}}{\texttt{det}V_N}$.
 We have the following lemma.
 \begin{lemma}
   \label{lemma:condition_number}
 Let $\kappa = \frac{T L^2}{\lambda_{\min}(V_N)}$, then $(1+\kappa) V_N \succcurlyeq V_{T+N}$.
 \end{lemma}
 \begin{proof}[\bf Proof of Lemma~\ref{lemma:condition_number}]
We first consider the case where all the data samples are returned before the
first online phase start.
   Denote the $V$ matrix in the online time slot $t$ after using the logged data
 as $V_{N+t}$.
 Note that $V_{T+N}=V_N+\sum_{t=1}^T \bm{x}_{t}\bm{x}_{t}^\prime$. Thus the above lemma is equivalent to $\sum_{t=1}^T \bm{x}_{N+t}\bm{x}_{N+t}^\prime
  {\preccurlyeq} \kappa V_N$. Here, we use $\bm{x}^\prime$ to denote the
  transpose of $\bm{x}$ (to avoid using ``$\bm{x}^T$'' with the confusing $T$).
 The positive semi-definiteness means that for any $\bm{x}$ where $||\bm{x}||_2{=}1$, we want to have
  \begin{align}
  \bm{x}^\prime \left(\sum_{t=1}^T \bm{x}_{t}\bm{x}_{t}^\prime \right) \bm{x}
  {\le} \kappa \bm{x}^\prime V_N \bm{x}.
    \label{eq:compare}
  \end{align}
  In fact $\bm{x}^\prime \left(\sum_{t=1}^T \bm{x}_{t}\bm{x}_{t}^\prime \right) \bm{x}
  \le T L^2$, because $L$ is the maximum 2-norm of $\bm{x}_t$. In addition, $\bm{x}^\prime V_N \bm{x}\ge \lambda_{\min}(V_N)$. Hence, we always
  have (\ref{eq:compare}) for $\forall \bm{x}$. Hence we proved the above lemma.
 \end{proof}
 We have $\texttt{det} A \le \texttt{det} B$ if $A \preccurlyeq B$. 
  Hence, 
\[ 
\texttt{det} V_{T+N} \le \texttt{det} (1+\kappa)V_N = (1+\kappa)^d\texttt{det} V_N.
\]
Then, $\log \frac{\texttt{det} V_{N+T}}{\texttt{det}V_N}\le d\log(1+\kappa)$,
which leads to our Theorem.

Now, we set $\beta_t(\delta)=2d(1+2\ln(1/\delta))$, and the parameter is in the
confidence ball with probability at least $1-\delta$. Moreover, we set
$\delta=1/T$. Then, the regret in each time slot can be divided into two parts:
(1) the $\delta$ probability part (summing up to at most 1, because the outcome
is bounded); and (2) the
$1-\delta$ probability part (summing up to at most
$\frac{8d(1+2\ln (T))}{\Delta_{\min}} d\log (1+\kappa)$).
Therefore, the expected cumulative reward has an upper bound $\frac{8d(1+2\ln
  (T))}{\Delta_{\min}} d\log (1+\kappa){+}1$.

Now, plugging in the definition of $\kappa$, we have proved
\begin{align*}
{R}(T,\CA_{\text{LinUCB+LR}}) \le
 \frac{8d^2(1+2\ln(T))}{\Delta_{\min}} \log\left(1+\frac{TL^2}{\lambda_{\min}(\bm{V}_N)}\right) + 1.
\end{align*}
\end{proof}

\subsection{Relaxations of The Assumptions on The Logged Data (Theorem~\ref{thm:no_ignorability})}

\begin{proof}[\bf Proof of Theorem~\ref{thm:no_ignorability}]
  Let us consider the number of times that a sub-optimal action is played, using
  the UCB online bandit oracle.
  Let us denote the expected reward (or outcome) $\MBE[y|a]$ for an action $a$
  as $\mu_a$. 
  In the $t_{th}$ online round, we make the wrong decision to play an action $a$
  only if $(\mu_{a^\ast} - \mu_a) + \left( \frac{\delta_{a^\ast} N_a}{N_a+t} -
    \frac{\delta_a N_a}{N_a+t} \right) < I_a - I_{a^\ast}$, where $I_a$ is half of the
  width of the confidence interval {\color{black}$\beta \sqrt{\frac{2\ln
        (n)}{n_a}}$} for action $a$, where $n_a$ is the number of times that the
  online bandit oracle plays action $a$ and $n=\sum_{a\in[K]}n_a$. 
Now, we only need to consider the case where
$\delta_a-\delta_{a^\ast}\ge 0$. Otherwise, the offline data lets us to have
less probability to select the sub-optimal actions, and thus leads to a lower regret.
  
  According to Chernoff bound, when we have 
  \begin{align}
    \label{eq:chernoff_condition}
  (N_a+t)[\Delta_a + \frac{N_a}{N_a+t}(\delta_{a^\ast}-\delta_a)]^2 \ge 8\ln(N_a+T),
  \end{align}
the violation probability will be very low. In fact, under (\ref{eq:chernoff_condition})
\begin{align*}
  \MBP\left[(\mu_{a^\ast} - \mu_a) + \left( \frac{\delta_{a^\ast} N_a}{N_a+t} -
    \frac{\delta_a N_a}{N_a+t} \right) < I_a - I_{a^\ast} \right] \le t^{-4}.
\end{align*}
Then we can let $l_a$ to be a number such that when $t>l_a$, the inequality
(\ref{eq:chernoff_condition}) is satisfied.

In fact, when $l_a {=} \lceil 16\frac{\ln(N_a+T)}{\Delta_a^2} {+}
  [N_a(\frac{2(\delta_a{-}\delta_{a^\ast})}{\Delta_a}{-}1)] - N_a \rceil$,
(\ref{eq:chernoff_condition}) is satisfied.
Therefore, the expected number of times that we play an action $a$ is less than
\begin{align*}
&l_a+\sum_{t=1}^T t^{-4} \\
\le &\left( 16\frac{\ln(N_a{+}T)}{\Delta_a^2} {-} 2N_a ( 1{-}\frac{\max\{0,\delta_a {-} \delta_{a^\ast}\} }{\Delta_a}  ) {+} (1{+}\frac{\pi^2}{3}) \right).
\end{align*}
When we sum up over all actions $a\ne a^\ast$, we get 
 $
 R(T,\CA) \le
\sum_{a\ne a^\ast} \hspace{-0.0in} \Delta_a 
\hspace{-0.0in} \left( 16\frac{\ln(N_a{+}T)}{\Delta_a^2} 
{-} 2N_a ( 1{-}\frac{\max\{0,\delta_a {-} \delta_{a^\ast}\} }{\Delta_a}  ) {+} (1{+}\frac{\pi^2}{3}) \right).
$
\end{proof}

\bibliographystyle{ACM-Reference-Format}
{\small
\bibliography{bib}
}

\bibliographystyle{ACM-Reference-Format}
{\small
\bibliography{bib}
}

\end{document}